\documentclass[a4paper]{article}

\usepackage{graphicx}

\usepackage{amsmath}
\usepackage{amsfonts}
\usepackage{amssymb}         
\usepackage{amsopn}
\usepackage{theorem}          
\usepackage{latexsym}
\usepackage{graphicx}        
\usepackage{mathrsfs}		
\usepackage{psfrag}
\usepackage{algorithmic}
\usepackage{algorithm}
\usepackage{color}
\usepackage{enumerate}
\usepackage{verbatim}
\usepackage{url}
\usepackage{tikz}
\usepackage{multirow}
\usepackage{fancyhdr}

\definecolor{offwhite}{gray}{0.92}
\definecolor{lightgrey}{gray}{0.85}
\definecolor{lightmidgrey}{gray}{0.75}
\definecolor{midgrey}{gray}{0.6}
\definecolor{darkmidgrey}{gray}{0.5}
\definecolor{darkgrey}{gray}{0.4}
\definecolor{verydarkgrey}{gray}{0.2}
\definecolor{grey9}{gray}{0.9}
\definecolor{grey8}{gray}{0.8}
\definecolor{grey7}{gray}{0.7}
\definecolor{grey6}{gray}{0.6}
\definecolor{grey5}{gray}{0.5}
\definecolor{grey4}{gray}{0.4}
\definecolor{grey3}{gray}{0.3}
\definecolor{grey2}{gray}{0.2}
\definecolor{grey1}{gray}{0.1}
\definecolor{pink}{rgb}{1,0.5,0.5}
\definecolor{lightpink}{rgb}{1,0.8,0.8}
\definecolor{orange}{cmyk}{0,0.61,0.87,0}
\definecolor{lightorange}{cmyk}{0,0.58,0.95,0}
\definecolor{darkgreen}{rgb}{0,0.3,0}
\definecolor{darkred}{rgb}{0.5,0,0}
\definecolor{darkblue}{rgb}{0,0,0.5}
\definecolor{lightgreen}{rgb}{0.6,1,0.4}
\definecolor{lightcyan}{rgb}{0.8,1,0.9}
\definecolor{purple}{rgb}{0.7,0.1,0.7}
\definecolor{darkpurple}{rgb}{0.55,0.05,0.55}
\definecolor{reasonablegreen}{rgb}{0,0.5,0}

\definecolor{header}{gray}{0.75}
\definecolor{subheader}{gray}{0.85}
\definecolor{evenrow}{gray}{0.90}
\definecolor{oddrow}{gray}{1.0}

\definecolor{er}{rgb}{1,1,0.5}
\definecolor{hdm}{rgb}{0.5,1,0.5}
\definecolor{rel}{rgb}{0.5,1,1}
\definecolor{orm}{rgb}{0.8,0.8,0.5}
\definecolor{uml}{rgb}{0.6,0.6,1}
\definecolor{www}{rgb}{1,0.5,1}
\definecolor{xml}{rgb}{0.8,0.8,1}
\definecolor{comment}{rgb}{0.75,0.75,0.75}

\definecolor{GreenYellow}{cmyk}{0.15,0,0.69,0}
\definecolor{Yellow}{cmyk}{0,0,1,0}
\definecolor{Goldenrod}{cmyk}{0,0.10,0.84,0}
\definecolor{Dandelion}{cmyk}{0,0.29,0.84,0}
\definecolor{Apricot}{cmyk}{0,0.32,0.52,0}
\definecolor{Peach}{cmyk}{0,0.50,0.70,0}
\definecolor{Melon}{cmyk}{0,0.46,0.50,0}
\definecolor{YellowOrange}{cmyk}{0,0.42,1,0}
\definecolor{Orange}{cmyk}{0,0.61,0.87,0}
\definecolor{BurntOrange}{cmyk}{0,0.51,1,0}
\definecolor{Bittersweet}{cmyk}{0,0.75,1,0.24}
\definecolor{RedOrange}{cmyk}{0,0.77,0.87,0}
\definecolor{Mahogany}{cmyk}{0,0.85,0.87,0.35}
\definecolor{Maroon}{cmyk}{0,0.87,0.68,0.32}
\definecolor{BrickRed}{cmyk}{0,0.89,0.94,0.28}
\definecolor{Red}{cmyk}{0,1,1,0}
\definecolor{OrangeRed}{cmyk}{0,1,0.50,0}
\definecolor{RubineRed}{cmyk}{0,1,0.13,0}
\definecolor{WildStrawberry}{cmyk}{0,0.96,0.39,0}
\definecolor{Salmon}{cmyk}{0,0.53,0.38,0}
\definecolor{CarnationPink}{cmyk}{0,0.63,0,0}
\definecolor{Magenta}{cmyk}{0,1,0,0}
\definecolor{VioletRed}{cmyk}{0,0.81,0,0}
\definecolor{Rhodamine}{cmyk}{0,0.82,0,0}
\definecolor{Mulberry}{cmyk}{0.34,0.90,0,0.02}
\definecolor{RedViolet}{cmyk}{0.07,0.90,0,0.34}
\definecolor{Fuchsia}{cmyk}{0.47,0.91,0,0.08}
\definecolor{Lavender}{cmyk}{0,0.48,0,0}
\definecolor{Thistle}{cmyk}{0.12,0.59,0,0}
\definecolor{Orchid}{cmyk}{0.32,0.64,0,0}
\definecolor{DarkOrchid}{cmyk}{0.40,0.80,0.20,0}
\definecolor{Purple}{cmyk}{0.45,0.86,0,0}
\definecolor{Plum}{cmyk}{0.50,1,0,0}
\definecolor{Violet}{cmyk}{0.79,0.88,0,0}
\definecolor{RoyalPurple}{cmyk}{0.75,0.90,0,0}
\definecolor{BlueViolet}{cmyk}{0.86,0.91,0,0.04}
\definecolor{Periwinkle}{cmyk}{0.57,0.55,0,0}
\definecolor{CadetBlue}{cmyk}{0.62,0.57,0.23,0}
\definecolor{CornflowerBlue}{cmyk}{0.65,0.13,0,0}
\definecolor{MidnightBlue}{cmyk}{0.98,0.13,0,0.43}
\definecolor{NavyBlue}{cmyk}{0.94,0.54,0,0}
\definecolor{RoyalBlue}{cmyk}{1,0.50,0,0}
\definecolor{Blue}{cmyk}{1,1,0,0}
\definecolor{Cerulean}{cmyk}{0.94,0.11,0,0}
\definecolor{Cyan}{cmyk}{1,0,0,0}
\definecolor{ProcessBlue}{cmyk}{0.96,0,0,0}
\definecolor{SkyBlue}{cmyk}{0.62,0,0.12,0}
\definecolor{Turquoise}{cmyk}{0.85,0,0.20,0}
\definecolor{TealBlue}{cmyk}{0.86,0,0.34,0.02}
\definecolor{Aquamarine}{cmyk}{0.82,0,0.30,0}
\definecolor{BlueGreen}{cmyk}{0.85,0,0.33,0}
\definecolor{Emerald}{cmyk}{1,0,0.50,0}
\definecolor{JungleGreen}{cmyk}{0.99,0,0.52,0}
\definecolor{SeaGreen}{cmyk}{0.69,0,0.50,0}
\definecolor{Green}{cmyk}{1,0,1,0}
\definecolor{ForestGreen}{cmyk}{0.91,0,0.88,0.12}
\definecolor{PineGreen}{cmyk}{0.92,0,0.59,0.25}
\definecolor{LimeGreen}{cmyk}{0.50,0,1,0}
\definecolor{YellowGreen}{cmyk}{0.44,0,0.74,0}
\definecolor{SpringGreen}{cmyk}{0.26,0,0.76,0}
\definecolor{OliveGreen}{cmyk}{0.64,0,0.95,0.40}
\definecolor{RawSienna}{cmyk}{0,0.72,1,0.45}
\definecolor{Sepia}{cmyk}{0,0.83,1,0.70}
\definecolor{Brown}{cmyk}{0,0.81,1,0.60}
\definecolor{Tan}{cmyk}{0.14,0.42,0.56,0}
\definecolor{Gray}{cmyk}{0,0,0,0.50}
\definecolor{Black}{cmyk}{0,0,0,1}
\definecolor{White}{cmyk}{0,0,0,0}

\definecolor{azure}{rgb}{0.94, 1.0, 1.0}

\newtheorem{result}{\ }[section]
\theoremstyle{changebreak}                
\newtheorem{theorem}[result]{Theorem}

\newtheorem{lemma}[result]{Lemma}
\newtheorem{corollary}[result]{Corollary}

\newenvironment{proof}
 {{\sl Proof.}\hspace*{1 ex}}%
 {{\nopagebreak\hspace*{\fill}$\Box$\par\vspace{12pt}}}

\newcommand{\transpose}[1]{{#1}^\top}
\newcommand{\centroid}[1]{\mathsf{centroid}(#1)}

\newcommand{\rank}[1]{\mathsf{rk}(#1)}
\newcommand{\diag}[1]{\mathsf{diag}(#1)}
\newcommand{\trace}[1]{\mathsf{tr}(#1)}
\newcommand{\dist}[1]{\mathsf{dist}(#1)}

\DeclareMathOperator*{\opt}{opt}

\begin{document}

\begin{flushleft}\fbox{\color{blue}TOP (invited survey, to appear in 2020, Issue 2)}\end{flushleft}

\thispagestyle{empty}
\begin{center} 

{\LARGE Distance Geometry and Data Science}\footnote{This research was partly funded by the European Union's Horizon 2020 research and innovation programme under the Marie Sklodowska-Curie grant agreement n. 764759 ETN ``MINOA".}
\par \bigskip
{\sc Leo Liberti${}^{1}$} 
\par \bigskip
\begin{minipage}{15cm}
\begin{flushleft}
{\small
\begin{itemize}
\item[${}^1$] {\it LIX CNRS, \'Ecole Polytechnique, Institut Polytechnique de Paris, 91128 Palaiseau, France} \\ Email:\url{liberti@lix.polytechnique.fr}
\end{itemize}
}
\end{flushleft}
\end{minipage}
\par \medskip \today
\end{center}
\par \bigskip

\begin{flushright}
{\it Dedicated to the memory of Mariano Bellasio (1943-2019).}
\end{flushright}

\begin{abstract}
  Data are often represented as graphs. Many common tasks in data science are based on distances between entities. While some data science methodologies natively take graphs as their input, there are many more that take their input in vectorial form. In this survey we discuss the fundamental problem of mapping graphs to vectors, and its relation with mathematical programming. We discuss applications, solution methods, dimensional reduction techniques and some of their limits. We then present an application of some of these ideas to neural networks, showing that distance geometry techniques can give competitive performance with respect to more traditional graph-to-vector mappings. 
\end{abstract}

{\renewcommand{\baselinestretch}{-0.2}\footnotesize 
  \tableofcontents\par
\renewcommand{\baselinestretch}{1.00}\normalsize }

\section{Introduction}
\label{s:introduction}
This survey is about the application of Distance Geometry (DG) techniques to problems in Data Science (DS). More specifically, data are often represented as graphs, and many methodologies in data science require vectors as input. We look at the fundamental problem in DG, namely that of reconstructing vertex positions from given edge lengths, in view of using its solution methods in order to produce vector input for further data processing.

The organization of this survey is based on a ``storyline''. In summary, we want to exhibit alternative competitive methods for mapping graphs to vectors in order to analyse graphs using Machine Learning (ML) methodologies requiring vectorial input. This storyline will take us through fairly different subfields of mathematics, Operations Research (OR) and computer science. This survey does not provide exhaustive literature reviews in all these fields. Its purpose (and usefulness) rests in communicating the main idea sketched above, rather than serving as a reference for a field of knowledge. It is nonetheless a survey because, limited to the scope of its purpose, it aims at being informative and also partly educational, rather than just giving the minimal notions required to support its goal. 

Here is a more detailed account of our storyline. We first introduce DG, some of its history, its fundamental problem and its applications. Then we motivate the use of graph representations for several types of data. Next, we discuss some of the most common tasks in data science (e.g.~classification, clustering) and the related methodologies (unsupervised and supervised learning). We introduce some robust and efficient algorithms used for embedding general graphs in vector spaces. We present some dimensional reduction operations, which are techniques for replacing sets $X$ of high-dimensional vectors by lower-dimensional ones $X'$, so that some of the properties of $X$ are preserved at least approximately in $X'$. We discuss the instability of distances on randomly generated vectors and its impact on distance-based algorithms. Finally, we present an application of much of the foregoing theory: we train an Artificial Neural Network (ANN) on many training sets, so as to learn several given clusterings on sentences in natural language. Some training sets are generated using traditional methods, namely incidence vectors of short sequences of consecutive words in the corpus dictionary. Other training sets are generated by representing sentences by graphs and then using a DG method to encode these graphs into vectors. It turns out that some of the DG-generated training sets have competitive performances with the traditional methods. While the empirical evidence is too limited to support any general conclusion, it might invite more research on this topic.

The survey is interspersed with eight theorems with proofs. Aside from Thm.~\ref{thm:distres} about distance instability, the proof of which is taken almost verbatim from the original source \cite{beyer}, the proofs from the other theorems are not taken from any source. This does not mean that the theorems and their proofs are actually original. The theorems are usually quite easy to prove. Their proofs are reasonably short, and, we hope, easy to follow. There are several reasons for the presence of these theorems in this survey: (a) we have not found them stated and proved clearly anywhere else, and we wish we had during our research work (Thm.~\ref{thm:dgpnumsol}-\ref{thm:uie}); (b) their proofs showcase some point we deem important about the underlying theory (Thm.~\ref{thm:PPI}-\ref{thm:distres}); (c) they give some indication of the proof techniques involved in the overarching field (Thm.~\ref{thm:rpnorm}-\ref{thm:PPI}); (d) they justify a precise mathematical statement for which we found no citation (Thm.~\ref{thm:sdpedmcp}). While there may be some original mathematical results in this survey, e.g.~Eq.~\eqref{eq:sdpedmcp} and the corresponding Thm.~\ref{thm:sdpedmcp} (though something similar might be found in Henry Wolkowicz' work) as well as the computational comparison in Sect.~\ref{s:rpidapprox}, we believe that the only truly original part is the application of DG techniques to constructing training sets of ANNs in Sect.~\ref{s:anndg}. Sect.~\ref{s:datgph}, about representing data by graphs, may also contain some new ideas to Mathematical Programming (MP) readers, although everything we wrote can be easily reconstructed from existing literature, though some of which might perhaps be rather exotic to MP readership.


The rest of this paper is organized as follows. In Sect.~\ref{s:mp} we give a brief introduction to the field of MP, considered as a formal language for optimization. In Sect.~\ref{s:dg} we introduce the field of DG. In Sect.~\ref{s:datgph} we give details on how to represent four types of data as graphs. In Sect.~\ref{s:clustering} we introduce methods for clustering on vectors as well as directly on graphs. In Sect.~\ref{s:dgpsol} we present many methods for realizing graphs in Euclidean spaces, most of which are based on MP. In Sect.~\ref{s:dimred} we present some dimensional reduction techniques. In Sect.~\ref{s:distres} we discuss the distance instability phenomenon, which may have a serious negative inpact on distance-based algorithms. In Sect.~\ref{s:anndg} we present a case-in-point application of natural language clustering by means of an ANN, and discuss how the DG techniques can help construct the input part of the training set. 

\section{Mathematical Programming}
\label{s:mp}
Many of the methods discussed in this survey are optimization methods. Specifically, they belong to MP, which is a field of optimization sciences and OR. While most of the readers of this paper should be familiar with MP, the interpretation we give to this term is more formal than most other treatments, and we therefore discuss it in this section.

\subsection{Syntax}
MP is a formal language for describing optimization problems. The valid sentences of this language are the MP {\it formulations}. Each formulation consist of an array $p$ of {\it parameter} symbols (which encode the problem input), an array $x$ of $n$ {\it decision variable} symbols (which will contain the solution), an {\it objective function} $f(p,x)$ with an optimization direction (either $\min$ or $\max$), a set of {\it explicit constraints} $g_i(p,x)\le 0$ for all $i\le m$, and some {\it implicit constraints}, which impose that $x$ should belong to some implicitly described set $X$. For example, some of the variables should take integer values, or should belong to the non-negative orthant, or to a positive semidefinite (psd) cone. The typical MP formulation is as follows:
\begin{equation}
  \left.\begin{array}{rrcl}
    \opt_x \ & f(p,x) && \\
    \forall i\le m \ & g_i(p,x) &\le& 0 \\
    & x &\in& X.
  \end{array}\right\}
  \label{eq:mp}
\end{equation}

It is customary to define MP formulations over explicitly closed feasible sets, in order to prevent issues with feasible formulations which have infima or suprema but no optima. This prevents the use of strict inequality symbols in the MP language. 

\subsection{Taxonomy}
MP formulations are classified according to syntactical properties. We list the most important classes:
\begin{itemize}
\item if $f,g_i$ are linear in $x$ and $X$ is the whole space, Eq.~\eqref{eq:mp} is a Linear Program (LP);
\item if $f,g_i$ are linear in $x$ and $X=\{0,1\}^n$, Eq.~\eqref{eq:mp} is a Binary Linear Program (BLP);
\item if $f,g_i$ are linear in $x$ and $X$ is the whole space intersected with an integer lattice, Eq.~\eqref{eq:mp} is a Mixed-Integer Linear Program (MILP);
\item if $f$ is a quadratic form in $x$, $g_i$ are linear in $x$, and $X$ is the whole space, Eq.~\eqref{eq:mp} is a Quadratic Program (QP); if $f$ is convex, then it is a convex QP (cQP);
\item if $f$ is linear in $x$ and $g_i$ are quadratic forms in $x$, and $X$ is the whole space or a polyhedron, Eq.~\eqref{eq:mp} is a Quadratically Constrained Program (QCP); if $g_i$ are convex, it is a convex QCP (cQCP);
\item if $f$ and $g_i$ are quadratic forms in $x$, and $X$ is the whole space or a polyhedron, Eq.~\eqref{eq:mp} is a Quadratically Constrained Quadratic Program (QCQP); if $f,g_i$ are convex, it is a convex QCQP (cQCQP);
\item if $f,g_i$ are (possibly) nonlinear functions in $x$, and $X$ is the whole space or a polyhedron, Eq.~\eqref{eq:mp} is a Nonlinear Program (NLP); if $f,g_i$ are convex, it is a convex NLP (cNLP);
\item if $x$ is a symmetric matrix of decision variables, $f,g_i$ are linear, and $X$ is the set of all psd matrices, Eq.~\eqref{eq:mp} is a Semidefinite Program (SDP);
\item if we impose some integrality constraints on any decision variable on formulations from the classes QP, QCQP, NLP, SDP, we obtain their respective mixed-integer variants MIQP, MIQCQP, MINLP, MISDP.
\end{itemize}
This taxonomy is by no means complete (see \cite[\S 3.2]{refmathprog} and \cite{williams}).

\subsection{Semantics}
As in all formal languages, sentences are given a meaning by replacing variable symbols with other mathematical entities. In the case of MP, its semantics is assigned by an algorithm, called {\it solver}, which looks for a numerical solution $x^\ast\in\mathbb{R}^n$ having some optimality properties and satisfying the constraints. For example, BLPs such as Eq.~\eqref{eq:modularity2} can be solved by the CPLEX solver \cite{cplex128}. This allows users to solve optimization problems just by ``modelling'' them (i.e.~describing them as a MP formulation) instead of having to invent a specific solution algorithm. As a formal descriptive language, MP was shown to be Turing-complete \cite{undecminlp,universal_mp}.

\subsection{Reformulations}
\label{s:reformulations}
It is always the case that infinitely many formulations have the same semantics: this can be seen in a number of trivial ways, such as e.g.~multiplying some constraint $g_i\le 0$ by any positive scalar in Eq.~\eqref{eq:mp}. This will produce an uncountable number of different formulations with the same feasible and optimal set.

Less trivially, this property is precious insofar as solvers perform more or less efficiently on different (but semantically equivalent) formulations. More generally, a symbolic transformation on an MP formulation for which one can provide some guarantees on the consequent changes on the feasible or optimal set is called a {\it reformulation} \cite{refmathprog,arschapter,rose}.

Three types of reformulation guarantees will appear in this survey:
\begin{itemize}
\item the {\it exact} reformulation: the optima of the reformulated problem can be mapped easily back to those of the original problem; 
\item the {\it relaxation}: the optimal objective function value of the reformulated problem provides a bound (in the optimization direction) on the optimal objective function value of the original problem;
\item the {\it approximating} reformulation: a sequence of formulations based on a parameter which also appears in a ``guarantee statement'' (e.g.~an inequality providing a bound on the optimal objective function value of the original problem); when the parameter tends to infinity, the guarantee proves that formulations in the sequence get closer and closer to an exact reformulation or to a relaxation.
\end{itemize}

Reformulations are only useful when they can be solved more efficiently than the original problem. Exact reformulations are important because the optima of the original formulation can be retrieved easily. Relaxations are important in order to evaluate the quality of solutions of heuristic methods which provide solutions without any optimality guarantee; moreover, they are crucial in Branch-and-Bound (BB) type solvers (such as e.g.~CPLEX). Approximating reformulations are important to devise approximate solution methods for MP problems. 

There are some trivial exact reformulations which guarantee that Eq.~\eqref{eq:mp} is much more general than it would appear at first sight: for example, inequality constraints can be turned into equality constraints by the addition of slack or surplus variables; equality constraints can be turned to inequality constraints by listing the constraint twice, once with $\le$ sense and once with $\ge$ sense; minimization can be turned to maximization by the equation $\min f=-\max -f$ \cite[\S 3.2]{arschapter}. 

\subsubsection{Linearization}
\label{s:linearization}
We note two easy, but very important types of reformulations.
\begin{itemize}
\item The {\it linearization} consists in identifying a nonlinear term $t(x)$ appearing in $f$ or $g_i$, replacing it with an added variable $y_t$, and then adjoining the {\it defining constraint} $y_t=t(x)$ to the formulation.
\item The {\it constraint relaxation} consists in removing a constraint: since this means that the feasible region becomes larger, the optima may only improve. Thus, relaxing constraints yields a relaxation.
\end{itemize}
These two reformulation techniques are often used in sequence: one identifies problematic nonlinear terms, linearizes them, and then relaxes the defining constraints. Carrying this out recursively for every term in an NLP \cite{mccormick}, and only relaxing the nonlinear defining constraints yields an LP relaxation of an NLP \cite{s_and_costas,tawarmalani1,couenne}.

\section{Distance Geometry}
\label{s:dg}
DG refers to a foundation of geometry based on the concept of distances instead of those of points and lines (Euclid) or point coordinates (Descartes). The axiomatic foundations of DG were first laid out in full generality by Menger \cite{menger28}, and later organized and systematized by Blumenthal \cite{blumenthal}. A {\it metric space} is a triplet $(\mathbb{X},\mathbb{D},d)$, where $\mathbb{X}$ is an abstract set, $\mathbb{D}\subseteq\mathbb{R}_+$, and $d$ is a binary relation $d:\mathbb{X}\times\mathbb{X}\to\mathbb{D}$ obeying the metric axioms:
\begin{enumerate}
\item $\forall x,y\in\mathbb{X} \quad d(x,y)=0 \leftrightarrow x=y$ (identity);
\item $\forall x,y\in\mathbb{X} \quad d(x,y)=d(y,x)$ (symmetry);
\item $\forall x,y,z\in\mathbb{X} \quad d(x,y)+d(y,z)\ge d(x,z)$ (triangle inequality).
\end{enumerate}

Based on these notions, one can define sequences and limits (calculus), as well as open and closed sets (topology). For any triplet $x,y,z$ of distinct elements in $\mathbb{X}$, $y$ is {\it between} $x$ and $z$ if $d(x,y)+d(y,z)=d(x,z)$. This notion of {\it metric betweenness} can be used to characterize convexity: a subset $\mathbb{Y}\subseteq\mathbb{X}$ is {\it metrically convex} if, for any two points $x,z\in\mathbb{Y}$, there is at least one point $y\in\mathbb{Y}$ between $x$ and $z$. The fundamental notion of invariance in metric spaces is that of congruence: two metric spaces $\mathbb{X},\mathbb{Y}$ are {\it congruent} if there is a mapping $\mu:\mathbb{X}\to\mathbb{Y}$ such that for all $x,y\in\mathbb{X}$ we have $d(x,y)=d(\mu(x),\mu(y))$.

The word ``isometric'' is often used as a synonym of ``congruent'' in many contexts, e.g.~with {\it isometric embeddings} (Sect.~\ref{s:uie}). In this survey, we mostly use ``isometric'' in relation to mappings from graphs to sets of vectors such that the weights of the edges are the same as the length of the segments between the vectors corresponding to the adjacent vertices. In other words, ``isometric'' is mostly used for partially defined metric spaces --- only the distances corresponding to the graph edges are considered.

While a systematization of the axioms of DG were formulated in the twentieth century, DG is pervasive throughout the history of mathematics, starting with Heron's theorem (computing the area of a triangle given the side lengths) \cite{euler1766}, going on to Euler's conjecture on the rigidity of (combinatorial) polyhedra \cite{heron}, Cauchy's creative proof of Euler's conjecture for strictly convex polyhedra \cite{cauchyrigid}, Cayley's theorem for inferring point positions from determinants of distance matrices \cite{cayley1841}, Maxwell's analysis of the stiffness of frames \cite{maxwell1864}, Henneberg's investigations on rigidity of structures \cite{henneberg1911}, G\"odel's fixed point theorem for showing that a tetrahedron with nonzero volume can be embedded isometrically (with geodetic distances) on the surface of a sphere \cite{goedelDG1}, Menger's systematization of DG \cite{menger31}, yielding, in particular, the concept of the Cayley-Menger determinant (an extension of Heron's theorem to any dimension, which was used in many proofs of DG theorems), up to Connelly's disproof of Euler's conjecture \cite{connelly-countereg}. A fuller account of many of these achievements is given in \cite{six}. An extension of G\"odel's theorem on the sphere embedding in any finite dimension appears in \cite{dgpsphere}.

\subsection{The distance geometry problem}
Before the widespread use of computers, the main applied problem of DG was to congruently embed finite metric spaces in some vector space. The first mention of the need for isometric embeddings using only a partial set of distances probably appeared in \cite{yemini78}. This need arose from wireless sensor networks: by estimating a set of distances for pairs of sensors which are close enough to establish peer-to-peer communication, is it possible to recover the position for all sensors in the network? Note that (a) distances can be recovered from peer-to-peer communicating pairs by monitoring the amount of battery required to exchange data; and (b) the positions for the sensors are in $\mathbb{R}^K$, with $K=2$ (usually) or $K=3$ (sometimes).

Thus we can formulate the main problem in DG. 
\begin{quote}
  {\sc Distance Geometry Problem} (DGP): given an integer $K>0$ and a simple undirected graph $G=(V,E)$ with an edge weight function $d:E\to\mathbb{R}_+$, determine whether there exists a {\it realization} $x:V\to\mathbb{R}^K$ such that:
  \begin{equation}
    \forall \{u,v\}\in E \quad \|x(u)-x(v)\| = d(u,v). \label{eq:dgp}
  \end{equation}
\end{quote}
We let $n=|V|$ and $m=|E|$ in the following.

We can re-state the DGP as follows: given a weighted graph $G$ and the dimension $K$ of a vector space, draw $G$ in $\mathbb{R}^K$ so that each edge is drawn as a straight segment of length equal to its weight. We remark that the realization $x$, defined as a function, is usually represented as an $n\times K$ matrix $x=(x_{uk}\;|\;u\in V\land k\le K)$, which may also be seen as an element of $\mathbb{R}^{nK}$. 

Notationally, we usually write $x_u,x_v$ and $d_{uv}$. If the norm used in Eq.~\eqref{eq:dgp} is $\ell_2$, then the above equation is usually squared, so it becomes a multivariate polynomial of degree two:
\begin{equation}
  \forall \{u,v\}\in E \quad \|x_u-x_v\|_2^2 = d_{uv}^2. \label{eq:dgp2}
\end{equation}
While most of the distances in this paper will be Euclidean, we shall also mention the so-called {\it linearizable norms} \cite{oneinfnorm-lncs}, i.e.~$\ell_1$ and $\ell_\infty$, because they can be described using linear forms. We also remark that the input of the DGP can also be represented by a {\it partial $n\times n$ distance matrix} $D$ where only the entries $d_{uv}$ corresponding to $\{u,v\}\in E$ are specified.

Many more notions about the DGP can be found in \cite{dgp-sirev,dgbook}.

\subsection{Number of solutions}
A DGP instance may have no solutions if the given distances do not define a metric, a finite number of solutions if the graph is rigid, or uncountably many solutions if the graph is flexible.

Restricted to the $\ell_2$ norm, there are several different notions of rigidity. We only define the simplest, which is easiest to explain intuitively: if we consider the graph as a representation of a joint-and-bar framework, a graph is flexible if the framework can move (excluding translations and rotations) and rigid otherwise. The formal definition of rigidity of a graph $G=(V,E)$ involves: (a) a mapping $\mathsf{D}$ from a realization $x\in\mathbb{R}^{nK}$ to the partial distance matrix
\[\mathsf{D}(x)=(\|x_u-x_v\|\;|\;\{u,v\}\in E);\]
and (b) the completion $\mathsf{K}(G)$ of $G$, defined as the complete graph on $V$. We want to say that $G$ is rigid if, were we to move $x$ ever so slightly (excluding translations and rotations), $\mathsf{D}(x)$ would also vary accordingly. We formalize this idea indirectly: a graph is {\it rigid} if the realizations in a neighbourhood $\chi$ of $x$ corresponding to changes in $\mathsf{D}(x)$ are equal to those in the neighbourhood $\bar{\chi}$ of a realization $\bar{x}$ of $\mathsf{K}(G)$ \cite[Ch.~7]{dgbook}. We note that realizations $\bar{x}\in\bar{\chi}$ correspond to small variations in $\mathsf{D}(\mathsf{K}(G))$: this definition makes sense because $\mathsf{K}(G)$ is a complete graph, which implies that its distance matrix has no variable components that can change, and hence $\bar{\chi}$ may only contain congruences.

We obtain the following formal characterization of rigidity \cite{asimow1}:
\begin{equation}
  \mathsf{D}^{-1}(\mathsf{D}(x))\cap\chi = \mathsf{D}^{-1}(\mathsf{D}(\bar{x}))\cap\bar{\chi}. \label{eq:rigid}
\end{equation}

Uniqueness of solution (modulo congruences) is sometimes a necessary feature in applications. Many different sufficient conditions to uniqueness have been found \cite[\S 4.1.1]{dgp-sirev}. By way of example as concerns the number of DGP solutions in graphs, a complete graph has at most one solution modulo congruences, as remarked above. It was proved in \cite{liberti-gsi13} that protein backbone graphs have a realization set having power of two cardinality with probability 1. As shown in Fig.~\ref{f:card} (bottom row), a cycle graph on four vertices has uncountably many solutions.
\begin{figure}[!ht]
  \begin{center}
    \fbox{\includegraphics[width=10cm]{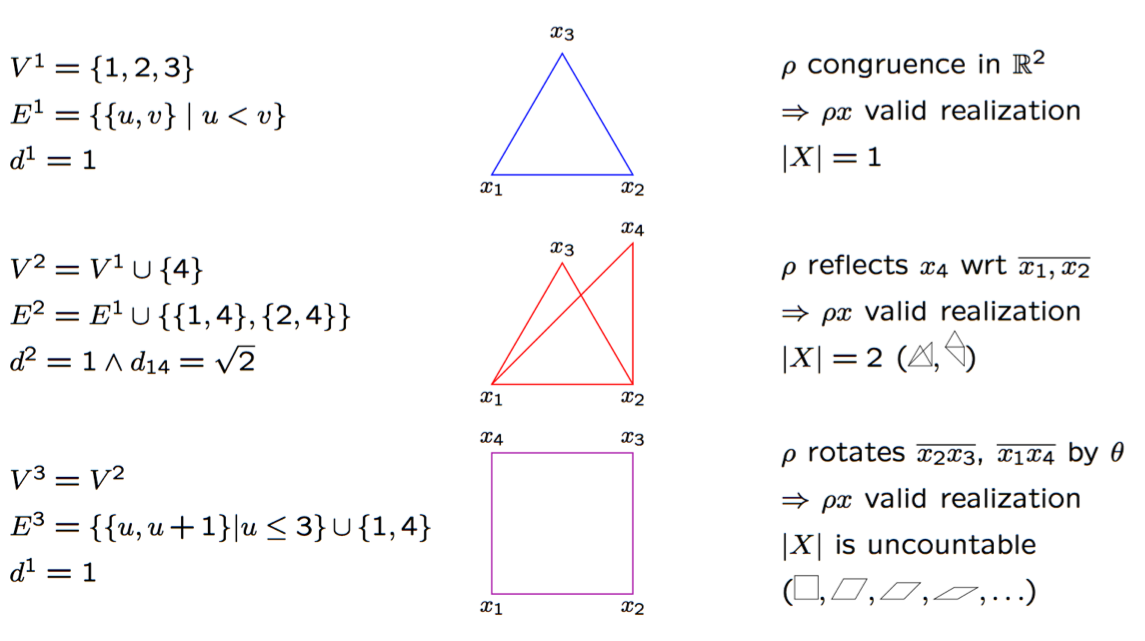}}
  \end{center}
  \caption{Instances with one, two, and uncountably many realizations.}
  \label{f:card}
\end{figure}

On the other hand, the remaining possibility of an infinite but countably many realizations of a DGP instance cannot happen, as shown in Thm.~\ref{thm:dgpnumsol}. This result is a corollary of a well-known theorem of Milnor's. It was noted informally in \cite[p.~27]{dgp-sirev} without details; we provide a proof here.
\begin{theorem}
  No DGP instance may have an infinite but countable number of solutions. \label{thm:dgpnumsol}
\end{theorem}
\begin{proof}
  Eq.~\eqref{eq:dgp2} is a system of $m$ quadratic equations associated with the instance graph $G$. Let $X\subseteq\mathbb{R}^{nK}$ be the variety associated to Eq.~\eqref{eq:dgp2}. Now suppose $X$ is countable: then no connected component of $X$ may contain uncountably many elements. By the notion of connectedness, this implies that every connected component is an isolated point in $X$. If $X$ is countable, it must contain a countable numbers of connected components. By \cite{milnor64}, the number of connected components of $X$ is finite;  in particular, it is bounded by $O(3^{nK})$. Hence the number of connected components of $X$ is finite. Since each is an isolated point, i.e.~a single realization of $G$, $|X|$ is finite. 
\end{proof}

\subsection{Applications}
The DGP is an inverse problem with many applications to science and engineering.

\subsubsection{Engineering}
When $K=1$ a typical application is that of clock synchronization \cite{singer4}. Network protocols for wireless sensor networks are designed so as to save power in communication. When synchronization and battery usage are key, the peer-to-peer communications needed to exchange the timestamp can be limited to the exchange of a single scalar, i.e.~the time (or phase) difference. The problem is then to retrieve the absolute times of all of the clocks, given some of the phase differences. This is equivalent to a DGP on the time line, i.e.~in a single dimension. We already sketched above the problem of Sensor Network Localization (SNL) in $K\in\{2,3\}$ dimensions. In $K=3$ we also have the problem of controlling fleets of Underwater Autonomous Vehicles (UAV), which requires the localization of each UAV in real-time \cite{bahr,dokmanic2}.

\subsubsection{Science}
\label{s:science}
An altogether different application in $K=3$ is the determination of protein structure from Nuclear Magnetic Resonance (NMR) experiments \cite{wuthrich_89}: proteins are composed of a linear backbone and some side-chains. The backbone determines a total order on the backbone atoms, by which follow some properties of the protein backbone graph. Namely, the distances from vertex $i$ to vertices $i-1$ and $i-2$ in the order are known almost exactly because of chemical information, and the distance between vertex $i$ and vertex $i-3$ is known approximately because of NMR output. Moreover, some other distances (with longer index difference) may also be known because of NMR --- typically, when the protein folds and two atoms from different folds happen to be close to each other. If we suppose all of these distances are known exactly, we obtain a subclass of DGP which is called {\sc Discretizable Molecular DGP} (DMDGP). The structure of the graph of a DMDGP instance is such that vertex $i$ is adjacent to its three immediate predecessors in the order: this yields a graph which consists of a sequence of embedded cliques on 4 vertices, the edges of which are called {\it discretization edges}, with possibly some extra edges called {\it pruning edges}.

If we had to realize this graph with $K=2$, we could use {\it trilateration} \cite{eren04}: given three points in the plane, compute the position of a fourth point at known distance from the three given points. Trilateration gives rise to a system of equations which has either no solution (if the distance values are not a metric) or a unique solution, since three distances in two dimensions are enough to disambiguate translations, rotations and reflections. Due to the specific nature of the DMDGP graph structure, it would suffice to know the positions of the first three vertices in the order to be able to recursively compute the positions of all other vertices. With $K=3$, however, there remains one degree of freedom which yields an uncertainty: the reflection.

We can still devise a combinatorial algorithm which, instead of finding a unique solution in $n-K$ trilateration steps, is endowed with back-tracking over reflections. Thus, the DMDGP can be solved completely (meaning that all incongruent solutions can be found) in worst-case exponential time by using the Branch-and-Prune (BP) algorithm \cite{lln5}. The DMDGP has other very interesting symmetry properties \cite{powerof2}, which allow for an {\it a priori} computation of its number of solutions \cite{liberti-gsi13}, as well as for generating all of the incongruent solutions from any one of them \cite{symmBPjbcb}; moreover, it turns out that BP is Fixed-Parameter Tractable (FPT) on the DMDGP \cite{bppolybook}.

\subsubsection{Machine Learning}
So far, we have only listed applications where $K$ is fixed. The focus of this survey, however, is a case where $K$ may vary: if we need to map graphs to vectors in view of preprocessing the input of a ML methodology, we may choose a dimension $K$ appropriate to the methodology and application at hand. See Sect.~\ref{s:anndg} for an example. 

\subsection{Complexity}

\subsubsection{Membership in {\bf NP}}
The DGP is clearly a decision problem, and one may ask whether it is in $\mathbf{NP}$. As stated above, with real number input in the edge weight function, it is clear that it is not, since the Turing computation model cannot be applied. We therefore consider its rational equivalent, where $d:E\to\mathbb{Q}_+$, and ask the same question. It turns out that, for $K>1$, we do not know whether the DGP is in $\mathbf{NP}$: the issue is that the solutions of sets of quadratic polynomials over $\mathbb{Q}$ may well be algebraic irrational. We therefore have the problem of establishing that a realization matrix $x$ with algebraic component verifies Eq.~\eqref{eq:dgp2} in polynomial time. While some compact representations of algebraic numbers exist \cite[\S 2.3]{undecminlp}, it is not known how to employ them in the polynomial time verification of Eq.~\eqref{eq:dgp2}. Negative results for the most basic representations of algebraic numbers were derived in \cite{dgpinnp}.

On the other hand, it is known that the DGP is in $\mathbf{NP}$ for $K=1$: as this case reduces to realizing graphs on a single real line, the fact that all of the given distances are in $\mathbb{Q}$ means that the distance between any two points on the line is rational: therefore, if one point is rational, then all the others can be obtained as sums and differences of this one point and a set of rational values, which implies that there is always a rational realization. Naturally, verifying whether a rational realization verifies Eq.~\eqref{eq:dgp2} can be carried out in polynomial time.

\subsubsection{{\bf NP}-hardness}
It was proved in \cite{saxe79} that the DGP is $\mathbf{NP}$-hard, even for $K=1$ (reduction from {\sc Partition} to the DGP on simple cycle graphs, see a detailed proof in \cite[\S 2.4.2]{dgbook}), and hence actually $\mathbf{NP}$-complete for $K=1$. In the same paper \cite{saxe79}, with more complicated gadgets it was also shown that the DGP is $\mathbf{NP}$-hard for each fixed $K$ and with edge weights restricted to taking values in $\{1,2\}$ (reduction from {\sc 3sat}).

A sketch of an adaptation of the reduction to cycle graphs is given in \cite{yemini} for DMDGP graphs, showing that they are an $\mathbf{NP}$-hard subclass of the DGP. A full proof following a similar idea can be found in \cite{dmdgp}. 

\section{Representing data by graphs}
\label{s:datgph}
It may be obvious to most readers that data can be naturally represented by graphs. This is immediately evident whenever data represent similarities or dissimilarities between entities in a vertex set $V$. In this section we make this intuition more explicit for a number of other relevant cases.

\subsection{Processes}
\label{s:gphproc}
The description of a process, be it chemical, electric/electronic, mechanical, computational, logical or otherwise, is practically always based on a directed graph, or {\it digraph}, $G=(N,A)$. The set of nodes $N$ represents the various stages of the process, while the arcs in $A$ represent transitions between stages.

Formalizations of this concept may possibly be first ascribed to the organization of knowledge proposed by Aristotle into genera and differences, commonly represented with a tree (a class of digraphs). While no graphical representation of this tree ever came to us from Aristotelian times, the commentator Porphyry of Tyre (3rd century AD) did refer to a representation which was actually drawn as a tree (at least since the 10th century \cite{verboon}). Many interesting images can be found in \url{last-tree.scottbot.net/illustrations/}, see e.g.~Fig.~\ref{fig:bacon}.
\begin{figure}[!ht]
  \begin{center}
    \includegraphics[width=6cm]{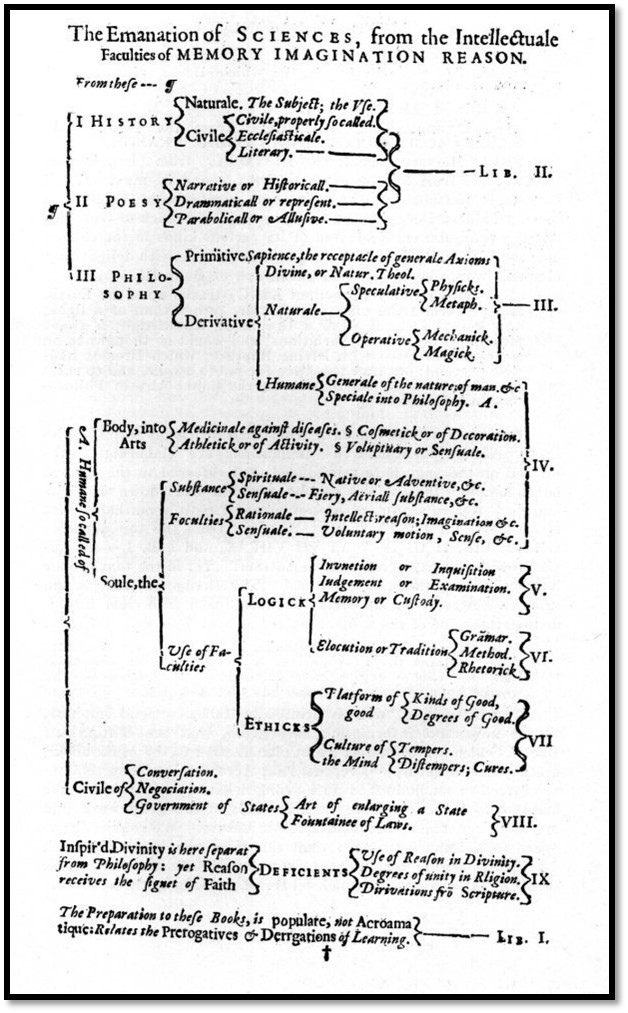}
  \end{center}
  \caption{A tree diagram from F.~Bacon's {\it Advancement of Learning}, Oxford 1640.}
  \label{fig:bacon}
\end{figure}

A general treatment of process diagrams in mechanical engineering is given in \cite{gilbreth}. Bipartite graphs with two node classes representing operations and materials have been used in process network synthesis in chemical engineering \cite{friedler}. Circuit diagrams are a necessary design tool for any electrical and electronic circuit \cite{seshu}. Software flowcharts (i.e.~graphical description of computer programs) have been used in the design of software so pervasively that one of the most important results in computer science, namely the B\"ohm-Jacopini's theorem on the expressiveness of universal computer languages, is based on a formalization of the concept of flowchart \cite{universallang}. The American National Standards Institute (ANSI) set standards for flowcharts and their symbols in the 1960s. The International Organization for Standardization (ISO) adopted the ANSI symbols in 1970 \cite{wikipedia_flowchart}. The {\it cyclomatic number} $|E|-|V|+1$ of a graph, namely the size of a cycle basis of the cycle space, was adopted as a measure of process graph complexity very early (see \cite{paton,deo82,brambilla,fcbmmor} and \cite[\S 2.3.4.1]{knuth1_3rd}).

An evalution of flowcharts to process design is the Unified Modelling Language (UML) \cite{omguml}, which was mainly conceived to aid the design of software-based systems, but was soon extended to much more general processes. With respect to flow\-charts, UML also models interactions between software systems and hardware systems, as well as with system users and stakeholders. When it is applied to software, UML is a semi-formal language, in the sense that it can automatically produce a set of header files with the description of classes and other objects, ready for code development in a variety of programming languages \cite{swarchex}.

\subsection{Text}
\label{s:gphtext}
One of the foremost issues in linguistics is the formalization of the rules of grammar in natural languages. On the one hand, text is scanned linearly, word by word. On the other hand, the sense of a sentence becomes apparent only when sentences are organized as trees \cite{chomsky}. This is immediately evident in the computer parsing of formal languages, with a ``lexer'' which carries out the linear scanning, and a ``parser'' which organizes the lexical tokens in a {\it parsing tree} \cite{lexyacc}. The situation is much more complicated for natural languages, where no rule of grammar is ever absolute, and any proposal for overarching principles has so many exceptions that it is hard to argue in their favor \cite{moro}.

The study of natural languages is usually split into syntax (how the sentence is organized), semantics (the sense conveyed by the sentence) and pragmatics (how the context when the sentence is uttered influences the meaning, and the impact that the uttered sentence has on the context itself) \cite{morris}. The current situation is that we have been able to formalize rules for natural language syntax (namely turning a linear text string into a parsing tree) fairly well, using probabilistic parsers \cite{manning} as well as supervised ML \cite{collobert}. We are still far from being able to successfully formalize semantics. Semiotics suggested many ways to assign semantics to sentences \cite{eco84}, but none of these is immediately and easily implementable as a computer program.

Two particularly promising suggestions are the organization of knowledge into an evolving encyclopedia, and the representation of the sense of words in a ``space'' with ``semantic axes'' (e.g.~``good/bad'', ``white/black'', ``left/right''\dots). The first suggestion yielded organized corpora such as WordNet \cite{wordnet}, which is a tree representation of words, synonyms and their semantical relations, not unrelated to a Porphyrian tree (Sect.~\ref{s:gphproc}). There is still a long way to go before the second is successfully implemented, but we see in the Google Word Vectors \cite{mikolov2013} the start of a promising path (although even easy semantical interpretations, such as analogies, are apparently not so well reflected in these word vectors, despite the publicity \cite{khalife19}).

For pragmatics, the situation is even more dire; some suggestions for representing knowledge and cognition w.r.t.~the state of the world are given in \cite{minsky_mind}. See \cite{wikipedia_pragmatics} for more information.

Insofar as graphs are concerned, syntax is organized into tree graphs, and semantics is often organized in corpora that are also trees, or directed acyclic graphs (DAGs), e.g.~WordNet and similar.

\subsubsection{Graph-of-words}
\label{s:graphofwords}
In Sect.~\ref{s:anndg} we consider a graph representation of sentences known as the {\it graph-of-words} \cite{vazirgiannis}. Given a sentence $s$ represented as a sequence of words $s=(s_1,\ldots,s_m)$, an $n$-gram is a subsequence of $n$ consecutive words of $s$. Each sentence obviously has at most $(m-n+1)$ $n$-grams. In a graph-of-words $G=(V,E)$ of order $n$, $V$ is the set of words in $s$; two words have an edge only if they appear in the same $n$-gram; the weight of the edge is equal to the number of $n$-grams in which the two words appear. This graph may also be enriched with semantic relations between the words, obtained e.g.~from WordNet.

\subsection{Databases}
The most common form of data collection is a database; among the existing database representations, one of the most popular is the tabular form used in spreadsheets and relational databases.

A {\it table} is a rectangular array $A$ with $n$ rows (the records) and $m$ columns (the features), which is (possibly only partially) filled with values. Specifically, each feature column must have values of the same type (when present). If $A_{rf}$ is filled with a value, we denote this $\mathsf{def}(r,f)$, for each record index $r$ and feature index $f$. We can represent this array via a bipartite graph $B=(R,F,E)$ where $R$ is the set of records, $F$ is the set of features, and there is an edge $\{r,f\}\in E$ if the $(r,f)$-th component $A_{rf}$ of $A$ is filled. A label function $\ell$ assigns the value $A_{rf}$ to the edge $\{r,f\}$. While this is an edge-labelled graph, the labels (i.e.~the contents of $A$) may not always be interpretable as edge weights --- so this representation is not yet what we are looking for. 

We now assume that there is a symmetric function $d_f:A_{\cdot,f}\times A_{\cdot,f}\to\mathbb{R}_+$ defined over elements of the column $A_{\cdot,f}$: since all elements in a column have the same type, such functions can always be defined in practice. We note that $d_f$ is undefined whenever one of the two arguments is not filled with a value. We can then define a composite function $d:R\times R\to\mathbb{R}_+$ as follows:
\begin{equation}
  \forall r\not=s\in R \quad d(r,s) =
  \left\{\begin{array}{l} \sum\limits_{f\in F\atop \mathsf{def}(r,f)\land\mathsf{def}(s,f)} d_f(A_{rf},A_{sf}) \\ \mbox{undefined if $\exists f\in F\;(\neg\mathsf{def}(r,f)\vee\neg\mathsf{def}(s,f))$.} \end{array}\right.\label{eq:datgph1}
\end{equation}
Next, we define a graph $G=(R,E')$ over the records $R$, where
\[E'=\{\,\{r,s\}\;|\;r\not=s\in R\land d(r,s)\mbox{ is defined}\},\]
weighted by the function $d:E'\to\mathbb{R}_+$ defined in Eq.~\eqref{eq:datgph1}. We call $G$ the {\it database distance graph}. Analysing this graph yields insights about record distributions, similarity and differences.


\subsection{Abductive inference}
\label{s:abduction}
According to \cite{eco83}, there are three main modes of rational thought, corresponding to three different permutations of the concepts ``hypothesis'' (call this H), ``prediction'' (call this P), ``observation'' (call this O). Each of the three permutations singles out a pair of concepts and a remaining concept. Specifically:
\begin{enumerate}
\item {\it deduction}: H $\land$ P $\to$ O;
\item {\it (scientific) induction}: O $\land$ P $\to$ H;
\item {\it abduction}: H $\land$ O $\to$ P.
\end{enumerate}
Take for example the most famous syllogism about Socrates being mortal:
\begin{itemize}
\item H: ``all humans are mortal'';
\item P: ``Socrates is human'';
\item O: ``Socrates is mortal''.
\end{itemize}
The syllogism is an example of deduction: we are given H and P, and deduce O. Note also that deduction is related to {\it modus ponens}: if we call $A$ the class of all humans and $B$ the class of all mortals, and let $s$ be the constant denoting Socrates, the syllogism can be restated as $(\forall x\; A(x)\subseteq B(x)\land A(s))\to B(s)$. Deduction infers truths (propositional logic) or provable sentences (first-order and higher-order logic), and is mostly used by logicians and mathematicians.

Scientific induction\footnote{Not to be confused with {\it mathematical induction}.} exploits observations and verifies predictions in order to derive a general hypothesis: if a large quantity of predictions is verified, a general hypothesis can be formulated. In other words, given O and P we infer H. Scientific induction can never provide proofs in sufficiently expressive logical universes, no matter the amount of observations and verified predictions. Any false prediction, however, disproves the hypothesis \cite{popper}. Scientific induction is about causality; it is mostly used by physicists and other natural scientists.

Abduction \cite{abduction} infers educates guesses about a likely state of a known universe from observed facts: given H and O, we infer P. According to \cite{mcculloch},
\begin{quote}
  {\small Deductions lead from rules and cases to facts --- the conclusions. Inductions lead toward truth, with less assurance, from cases and facts, toward rules as generalizations, valid for bound cases, not for accidents. Abductions, the apagoge of Aristotle, lead from rules and facts to the hypothesis that the fact is a case under the rule.\par}
\end{quote}
According to \cite{eco83} it can be traced back to Peirce \cite{peirce1878}, who cited Aristotle as a source. The author of \cite{proni} argues that the precise Aristotelian source cited by Peirce fails to make a valid reference to abduction; however, he also concedes that there are some forms of abduction foreshadowed by Aristotle in the texts where he defines definitions.

Let us see an example of abduction. Sherlock Holmes is called on a crime scene where Socrates lies dead on his bed. After much evidence is collected and a full-scale investigation is launched, Holmes ponders some possible hypotheses: for example, all rocks are dead. The prediction that is logically consistent with this hypothesis and the observation that Socrates is dead would be that Socrates is a rock. After some unsuccessful tests using Socrates' remains as a rock, Holmes eliminates this possibility. After a few more untenable suggestions by Dr.~Watson, Holmes considers the hypothesis that all humans are mortal. The logically consistent prediction is that Socrates is a man, which, in a dazzling display of investigating abilities, Holmes finds it to be exactly the case. Thus Holmes brilliantly solves the mystery, while Lestrade was just about ready to give up in despair. Abduction is about plausibility; it is the most common type of human inference.

Abduction is also the basis of learning: after witnessing a set of facts, and postulating hypotheses which link them together, we are able to make predictions about the future. Abductions also can, and in fact often turn out to, be wrong, e.g.:
\begin{itemize}
\item H: all beans in the bag are white;
\item O: there is a white bean next to the bag;
\item P: the bean was in the bag.
\end{itemize}
The white bean next to the bag, however, might have been placed there before the bag was even in sight. With this last example, we note that abductions are inferences often used in statistics. For an observation O, a set $\mathcal{H}$ of hypotheses and a set of possible predictions $\mathcal{P}$, we must evaluate
\begin{equation*}
  \forall \mbox{H}\in\mathcal{H}, \mbox{P}\in\mathcal{P} \quad p_{\mbox{\scriptsize HP}}=\mathbb{P}(\mbox{O}\;|\;\mbox{O,H abduce P}),
\end{equation*}
and then choose the pair $(\mbox{H,P})$ having largest probability $p_{\mbox{\scriptsize HP}}$ (see a simplified example in Fig.~\ref{f:abduction1}). 
\begin{figure}[!ht]
  \begin{center}
    \begin{tikzpicture}
      \node (observation) at (-2,0) {\color{magenta}\small white bean beside bag};
      \node (inf1) at (4,2) {\small{\color{blue}bag of white beans}$\to${\small\color{red}bean was in bag}};
      \draw [->] (observation.east) -- (inf1.west) node[midway,above,sloped]{\tiny 0.3};
      \node (inf2) at (4,1) {\small{\color{blue}white bean field closeby}$\to${\small\color{red}bean came from field}};
      \draw [->] (observation.east) -- (inf2.west) node[midway,above,sloped]{\tiny 0.25};
      \node (inf3) at (4,0) {\small{\color{blue}farmer market yesterday}$\to${\small\color{red}bean came from market}};
      \draw [->] (observation.east) -- (inf3.west) node[midway,above,sloped]{\tiny 0.1};
      \node (inf4) at (4,-1) {\small{\color{blue}kid was playing with beans}$\to${\small\color{red}kid lost a bean}};
      \draw [->] (observation.east) -- (inf4.west) node[midway,above,sloped]{\tiny 0.15};
      \node (inf4) at (4,-2) {\small{\color{blue}UFOs fueled with beans}$\to${\small\color{red}bean clearly a UFO sign}};
      \draw [->] (observation.east) -- (inf4.west) node[midway,above,sloped]{\tiny 0.2};      
    \end{tikzpicture}    
  \end{center}
  \caption{Evaluating probabilities in abduction. From left to right, observation {\color{magenta}O} abduces the inference {\color{blue}H}$\to${\color{red}P}.}
  \label{f:abduction1}
\end{figure}
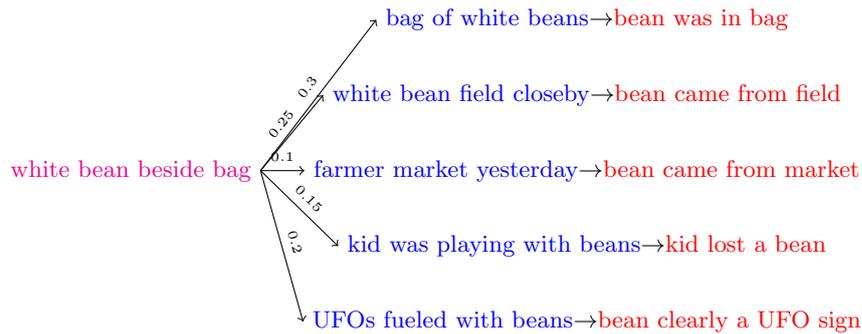

When more than one observation is collected, one can also compare distributions to make more plausible predictions, see Fig.~\ref{f:abduction2}. Abduction appears close to the kind of analysis often required by data scientists. 
\begin{figure}[!ht]
  \begin{center}
    \begin{tikzpicture}
      \node (observation1) at (-2,1) {\color{magenta}\small white bean beside bag};
      \node (inf1) at (4,2) {\small{\color{blue}bag of white beans}$\to${\small\color{red}bean was in bag}};
      \draw [->] (observation1.east) -- (inf1.west) node[midway,above,sloped]{\tiny 0.3};
      \node (inf2) at (4,1) {\small{\color{blue}white bean field closeby}$\to${\small\color{red}bean came from field}};
      \draw [->] (observation1.east) -- (inf2.west) node[midway,above,sloped]{\tiny 0.25};
      \node (inf3) at (4,0) {\small{\color{blue}farmer market yesterday}$\to${\small\color{red}bean came from market}};
      \draw [->] (observation1.east) -- (inf3.west) node[midway,above,sloped]{\tiny 0.1};
      \node (inf4) at (4,-1) {\small{\color{blue}kid was playing with beans}$\to${\small\color{red}kid lost a bean}};
      \draw [->] (observation1.east) -- (inf4.west) node[midway,above,sloped]{\tiny 0.15};
      \node (inf5) at (4,-2) {\small{\color{blue}UFOs fueled with beans}$\to${\small\color{red}bean clearly a UFO sign}};
      \draw [->] (observation1.east) -- (inf5.west) node[midway,above,sloped]{\tiny 0.2};      
      \node (observation2) at (-2,-1) {\color{magenta}\small red bean beside bag};
      \draw [->,darkgrey] (observation2.east) -- (inf1.west) node[above,sloped]{\tiny 0.01};
      \draw [->,darkgrey] (observation2.east) -- (inf2.west) node[above,sloped]{\tiny 0.01};
      \draw [->,darkgrey] (observation2.east) -- (inf3.west) node[above,sloped]{\tiny 0.49};
      \draw [->,darkgrey] (observation2.east) -- (inf4.west) node[above,sloped]{\tiny 0.29};
      \draw [->,darkgrey] (observation2.east) -- (inf5.west) node[above,sloped]{\tiny 0.2};      
    \end{tikzpicture}
  \end{center}
  \caption{Probability distributions over abduction inferences assigned to observations.}
  \label{f:abduction2}
\end{figure}
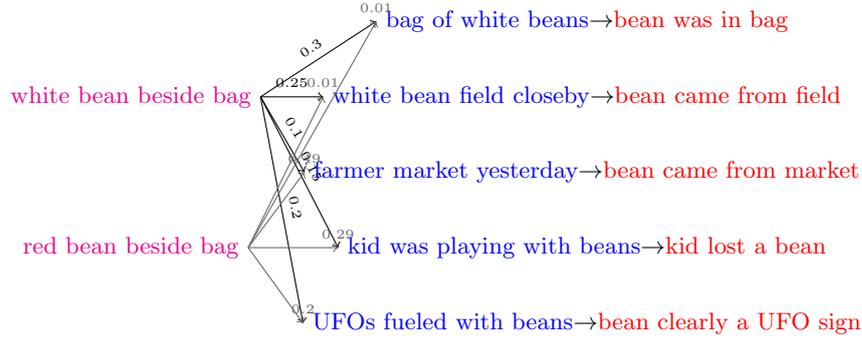

\subsubsection{The abduction graph}
We now propose a protocol for modelling good predictions from data, by means of an {\it abduction graph}. We consider:
\begin{itemize}
  \item a set $\mathcal{O}$ of observations O;
  \item a set $\mathcal{I}\subseteq\mathcal{H}\times\mathcal{P}$ of abductive premises, namely pairs $(\mbox{H},\mbox{P})$.
\end{itemize}
First, we note that different elements of $\mathcal{I}$ might be logically incompatible (e.g.~there may be contradictory sets of hypotheses or predictions). We must therefore extract a large set of logically compatible subsets of $\mathcal{I}$. Consider the relation $\sim$ on $\mathcal{I}$ with $h\sim k$ meaning that $h,k\in\mathcal{I}$ are logically compatible. This defines a graph $(\mathcal{I},\sim)$. We then find the largest (or at least large enough) clique $\bar{\mathcal{I}}$ in $(\mathcal{I},\sim)$.

Next, we define probability distributions $p^{\mbox{\scriptsize O}}$ on $\bar{\mathcal{I}}$ for each $\mbox{O}\in \mathcal{O}$. We let $E=\{\{\mbox{O},\mbox{O}'\}\;|\; \delta(p^{\mbox{\scriptsize O}},p^{\mbox{\scriptsize O}'})\le\delta_0\}$, where $\delta$ evaluates dissimilarities between probability distributions, e.g.~$\delta$ could be the Kullback-Leibler (KL) divergence \cite{kldiv}, and $\delta_0$ a given threshold. Thus $E$ defines a relation on $\mathcal{O}$ if $p^{\mbox{\scriptsize O}},p^{\mbox{\scriptsize O}'}$ are sufficiently similar. We can finally define the graph $\mathcal{F}=(\mathcal{O},E)$, with edges weighted by $\delta$. 

If we think of Sherlock Holmes again, the abduction graph encodes sets of clues compatible with the most likely consistent explanations. 

\section{Common data science tasks}
\label{s:clustering}
DS refers to practically every task or problem defined over large amounts of data. Even problems in $\mathbf{P}$, and sometimes even those for which there exist linear time algorithms, may take too long when confronted with huge-scale instances. We are not going to concern ourselves here with evaluation problems (such as computing means, variances, higher-order moments or other statistical measures), which are the realm of statistics, but rather with decision problems. In particular, it appears that a very common family of decision problems solved on large masses of data are those that help people make sense of the data themselves: in other words, classification and clustering.

There is no real functional distinction between the two, as both aim at partitioning the data into a relatively small number of subsets. However, ``classification'' usually refers to the problem of assigning class labels to data elements, while ``clustering'' indicates a classification based on the concept of similarity or distance, meaning that similar data elements should be in the same class. This difference is usually more evident in the algorithmic description: classification methods tend to exploit information inherent to elements, while clustering methods consider information relative to pairs of elements. In the rest of this paper, we shall adopt a functional view, and simply refer to ``clustering'' to indicate both classification and clustering.

Given a set $P$ of $n$ entities and some pairwise similarity function $\delta:P\times P\to\mathbb{R}_+$, clustering aims at finding a set of $k$ subsets $C_1,\ldots,C_k\subseteq P$ such that each cluster contains as many similar entities, and as few dissimilar entities, as possible. Cluster analysis --- as a field --- grew out of statistics in the course of the second half of the 20th century, encouraged by the advances in computing power. But some early forms of cluster analysis may also be attributed to earlier scientists (e.g.~Aristotle, Buffon, Cuvier, Linn\'e \cite{hansenjaumard}).

We note that ``clustering on graphs'' may refer to two separate tasks.
\begin{enumerate}[A.]
\item Cluster the vertices of a given graph.\label{taskA}
\item Cluster the graphs in a given set.\label{taskB}
\end{enumerate}
Both may arise depending on the application at hand. The proposed DG techniques for realizing graphs into vector spaces apply to both of these tasks (see Sect.~\ref{s:rlz2vec}).

As mentioned above, this paper focuses on transforming graphs into vectors so as to be able to use vector-based methods for classification and clustering. We shall first survey some of these methods. We shall then mention some methods for classifying/clustering graphs directly (i.e.~without needing to transform them into vectors first).

\subsection{Clustering on vectors}
Methods for classification and clustering on vectors are usually seen as part of ML. They are partitioned into unsupervised and supervised learning methods. The former are usually based on some similarity or dissimilarity measure defined over pairs of elements. The latter require a {\it training set}, which they exploit in order to find a set of optimal parameter values for a parametrized ``model'' of the data.

\subsubsection{The k-means algorithm}
\label{s:kmeans}
The k-means algorithm is a well-known heuristic for solving the following problem \cite{clustmp}. 
\begin{quote}
  {\sc Minimum Sum-of-Squares Clustering} (MSSC). Given an integer $k>0$ and a set $P\subset\mathbb{R}^m$ of $n$ vectors, find a set $\mathcal{C}=\{C_1,\ldots,C_k\}$ of subsets of $P$ such that the function
  \begin{equation}
    f(\mathcal{C}) = \sum\limits_{j\le k} \sum\limits_{x\in C_j} \|x-\centroid{C_j}\|_2^2 \label{mssc}
  \end{equation}
  is minimum, where
  \begin{equation}
    \centroid{C_j} = \frac{1}{|C_j|}\sum\limits_{x\in C_j} x.
  \end{equation}
\end{quote}
It is interesting to note that the MSSC problem can also be seen as a discrete analogue of the problem of partitioning a body into smaller bodies having minimum sum of moments of inertia \cite{steinhaus}. 

The k-means algorithm improves a given initial clustering $\mathcal{C}$ by means of the two following operations:
\begin{enumerate}
\item compute centroids $c_j=\centroid{C_j}$ for each $j\le k$;\label{op1}
\item for any pair of clusters $C_h,C_j\in\mathcal{C}$ and any point $x\in C_h$, if $x$ is closer to $c_j$ than to $c_h$, move $x$ from $C_h$ to $C_j$.\label{op2}
\end{enumerate}
These two operations are repeated until the clustering $\mathcal{C}$ no longer changes. Since the only decision operation (i.e.~operation \ref{op2}) is effective only if it decreases $f(\mathcal{C})$, it follows that k-means is a local descent algorithm. In particular, this very simple analysis offers no guarantee on the approximation of the objective function. For more information on the k-means algorithm, see \cite{kmeans2}.

k-means is an unsupervised learning technique \cite{jain}, insofar as it does not rest on a data model with parameters to be estimated prior to actually finding clusters. Moreover, the number ``k'' of clusters must be known {\it a priori}. 

\subsubsection{Artificial Neural Networks}
\label{s:ann}
An ANN is a parametrized model for representing an unknown function. Like all such models, it needs data in order to estimate suitable values for the parameters: this puts ANNs in the category of supervised ML. An ANN consists of two MP formulations defined over a graph and a training set.

An ANN is formally defined as a triplet $\mathcal{N}=(G,T,\phi)$, where:
\begin{itemize}
\item $G=(V,A)$ is a directed graph, with a node weight function $b:V\to\mathbb{R}$ (threshold at a node), and an edge weight function $w:A\to\mathbb{R}$ (weight on an arc); moreover, a subset $I\subset V$ of {\it input nodes} with $|I|=n$ and a subset $O\subset V$ of {\it output nodes} with $|O|=k$ are given in $G$;
\item $T=(X,Y)$ is the training set, where $X\subset\mathbb{R}^n$ (input set), $Y\subset\mathbb{R}^k$ (output set), and $|X|=|Y|$;
\item $\phi:\mathbb{R}\to\mathbb{R}$ is the {\it activation function} (many common activation functions map injectively into $[0,1]$).
\end{itemize}
The two MP formulations assigned to an ANN describe the {\it training problem} and the {\it evaluation problem}. In the training problem, appropriate values for $b,w$ are found using $T$. In the evaluation problem, a given input vector in $\mathbb{R}^n$ (usually not part of the input training set $X$) is mapped to an output vector in $\mathbb{R}^k$. The training problem decides values for the ANN parameters when seen as a model for an unknown function mapping the training input $X$ to the training output $Y$. After the model is trained, it can be evaluated on new (unseen) input.

In the following, we use standard notation on graphs. For a node $i\in V$ we let $N^-(i)=\{j\in V\;|\;(j,i)\in A\}$ be the {\it inward star} and $N^+(i)=\{j\in V\;|\;(i,j)\in A\}$ be the {\it outward star} of $i$. For undirected graphs $G=(V,E)$, we let $N(i)=\{j\in V\;|\;\{i,j\}\in E\}$ be the {\it star} of $i$. Moreover, for a tensor $s_{i_1,\ldots,i_r}$, where $i_j\in I_j$ for each $j\le r$, we denote a {\it slice} of $s$, defined by subsets $J_j\subseteq I_j$ for some $j\le r$, by $s[J_1]\cdots[J_r]$.

We discuss the evaluation phase first. Given values for $w,b$ and an input vector $x\in\mathbb{R}^n$, we decide a node weight function $u$ over $V$ as follows:
\begin{eqnarray}
  u_I &=& x \label{ann1} \\
  \forall j\in V\smallsetminus I \quad u_j &=& \phi\big(\!\!\sum_{i\in N^-(j)} w_{ij} u_i + b_j\big). \label{ann2}
\end{eqnarray}
We remark that Eq.~\eqref{ann2} is not an optimization but a decision problem. Nonetheless, it is a MP formulation (formally with zero objective function). After solving Eq.~\eqref{ann2}, one retrieves in particular $u[O]$, which correspond to an output vector in $u[0]=y\in\mathbb{R}^k$. When $G$ is acyclic, this decision problem reduces to a simple computation, which ``propagates'' the values of $u$ from the input nodes and forward through the network until they reach the output nodes. If $G$ is not acyclic, different solution methods must be used \cite{anderson,retineur,bengiobook}. 

The {\it training problem} is given in Eq.~\eqref{eq:trainprob1}. We let $N$ be the index set for the training pairs $(x,y)$ in $T$ (we recall that $|X|=|Y|$), and introduce a 2-dimensional tensor $v$ of decision variables indexed by $N$ and $V$.
\begin{equation}
  \left. \begin{array}{rrcl}
  \min\limits_{w,b,v} & \dist{v[N][O],Y} && \\
  \forall t\in N & v_t[I] &=& X  \\
  \forall t\in N, j\in V\smallsetminus I & v_{tj} &=& \phi_j\big(\sum\limits_{i\in N^-(j)} w_{ij} v_{ti} + b_j\big),
  \end{array}\right\}
  \label{eq:trainprob1}
\end{equation}
where $\dist{A,B}$ is a dissimilarity function taking dimensionally consistent tensor arguments $A,B$, which becomes closer to zero as $A$ and $B$ get closer. The solution of the training problem yields optimal values $w^\ast,b^\ast$ for the arc weights and node biases.


The training problem is in general a nonconvex optimization problem (because of the products between $w$ and $v$, and of the $\phi$ functions occurring in equations), which may have multiple global optima: finding them with state-of-the-art methods might require exponential time. For specific types of graphs and choices of objective function $\dist{\cdot,\cdot}$, the training problem may turn out to be convex. For example, if $G$ is a DAG, $V=I\dot{\cup}O$, the induced subgraphs $G[I]$ and $G[O]$ are empty (i.e.~they have no arcs), the activation functions are all sigmoids $\phi(z)=(1+\exp(-z))^{-1}$, and $\dist{\cdot,\cdot}$ is the negative logarithm of the likelihood functions
\[ \prod_{t\in N} \phi(\transpose{w}x^t+b_i)^{y_t}(1-\phi(\transpose{w}x^t+b_i))^{1-y_t} \]
summed over all output nodes $i\in O$, then it can be shown that the training problem is convex \cite{jordan95,schumacher}.

In contemporary treatments of ANNs, the underlying graph $G$ is almost always assumed to be a DAG. In modern Application Programming Interfaces (API), the acyclicity of $G$ is enforced by recursively replacing $v_{tj}$ with the corresponding expression in $\phi(\cdot)$.

Most algorithms usually solve Eq.~\eqref{eq:trainprob1} only locally and approximately. Usually, they employ a technique called Stochastic Gradient Descent (SGD) \cite{bottousgd}. This is a form of gradient descent where, at each iteration, the gradient of a multivariate function is estimated by partial gradients with respect to a randomly chosen subset of variables \cite[p.~100]{moitra}. 

The functional definition of an optimum for the training problem Eq.~\ref{eq:trainprob1} is poorly understood, as finding precise local (or global) optima is considered ``overfitting''. In other words, global or almost global optima of Eq.~\eqref{eq:trainprob1} lead to evaluations which are possibly perfect for pairs in the training set, but unsatisfactory for yet unseen input. Currently, finding ``good'' optima of ANN training problems is mostly based on experience, although a considerable effort is under way in order to reach a sound definition of optimum \cite{dauphin14,srann,vidalnn,lecun}. 

The main reason why ANNs are so popular today is that they have proven hugely successful at image recognition \cite{bengiobook}, and also extremely good at accomplishing other tasks, including natural language processing \cite{collobert}. Many efficient applications of ANNs to complex tasks involve interconnected networks of ANNs of many different types \cite{bengio}.

ANNs originated from an attempt to simulate neuronal activity in the brain: should the attempt prove successful, it would realize the old human dream of endowing a machine with human intelligence \cite{golem}. While ANNs today display higher precision than humans in some image recognition tasks, they may also be easily fooled by a few appropriately positioned pixels of different colors, which places the realization of ``human machine intelligence'' still rather far in the future --- or even unreachable, e.g.~if Penrose's hypothesis of quantum activity in the brain influencing intelligence at a macroscopic level holds \cite{penrose}. For more information about ANNs, see \cite{annhistory,bengiobook}. 

\subsection{Clustering on graphs}
While we argue in this paper that DG techniques allow the use of vector clustering methods to graph clustering, there also exist methods for clustering on graphs directly. We discuss two of them, both applicable to the task of clustering vertices of a given graph (Task~\ref{taskA} on p.~\pageref{taskA}).

\subsubsection{Spectral clustering}
\label{s:spclust}
Consider a connected graph $G=(V,E)$ with an edge weight function $w:E\to\mathbb{R}_+$. Let $A$ be the {\it adjacency matrix} of $G$, with $A_{ij}=w_{ij}$ for all $\{i,j\}\in E$, and $A_{ij}=0$ otherwise. Let $\Delta$ be the diagonal {\it weighted degree matrix} of $G$, with $\Delta_{ii}=\sum_{j\not=i} A_{ij}$ and $\Delta_{ij}=0$ for all $i\not=j$. The {\it Laplacian} of $G$ is defined as $L=\Delta-A$.

Spectral clustering aims at finding a minimum balanced cut $U\subset V$ in $G$ by looking at the spectrum of the Laplacian of $G$. For now, we give the word ``balanced'' only an informal meaning: it indicates the fact that we would like clusters to have approximately the same cardinality (we shall be more precise below). Removing the {\it cutset} $\delta(U)$ (i.e.~the set of edges between $U$ and $V\smallsetminus U$) from $G$ yields a two-way partitioning of $V$. If $|\delta(U)|$ is minimum over all possible cuts $U$, then the two sets $U,V\smallsetminus U$ should both intuitively induce subgraphs $G[U]$ and $G[V\smallsetminus U]$ having more edges than those in $\delta(U)$. In other words, the criterion we are interested in maximizes the intra-cluster edges of the subgraphs of $G$ induced by the cluster while minimizing the inter-cluster edges of the corresponding cutsets. 

We remark that each of the two partitions can be recursively partitioned again. A recursive clustering by two-way partitioning is a general methodology which is part of a family of {\it hierarchical} clustering methods \cite{schaefferclust}. So the scope of this section is not limited to generating two clusters only. 

For simplicity, we only discuss the case with unit edge weights, although the generalization to general weights is not problematic. Thus, $\Delta_{ii}$ is the degree of vertex $i\in V$. We model a balanced partition $\{B,C\}$ corresponding to a minimum cut by means of decision variables $x_i=1$ if $i\in B$ and $x_i=-1$ if $i\in C$, for each $i\le n$, with $n=|V|$. Then $f(x)=\frac{1}{4}\sum_{\{i,j\}\in E} (x_i-x_j)^2$ counts the number of intercluster edges between $B$ and $C$. We have:
\begin{eqnarray*}
 4f(x) &=&\sum_{\{i,j\}\in E} (x_i^2+x_j^2) -  2\sum_{\{i,j\}\in E} x_ix_j =  \sum_{\{i,j\}\in E} 2 -  \sum_{i,j\le n} x_i a_{ij} x_j = \\
  &=& 2|E| - \transpose{x}Ax = \sum_{i\le n} x_i d_i x_i - \transpose{x}Ax = \transpose{x}(\Delta - A)x = \transpose{x} Lx,
\end{eqnarray*}
whence $f(x)=\frac{1}{4}\transpose{x}Lx$. We can therefore obtain cuts with minimum $|\delta(B)|$ by minimizing $f(x)$.

We can now give a more precise meaning to the requirement that partitions are {\it balanced}: we require that $x$ must satisfy the constraint
\begin{equation}
  \sum_{i\le n} x_i=0. \label{balanced}
\end{equation}
Obviously, Eq.~\eqref{balanced} only ensures equal cardinality partitions on graphs having an even number of vertices. However, we relax the integrality constraints $x\in\{-1,1\}^n$ to $x\in[-1,1]^n$, so $\sum_{i\le n} x_i=0$ is applicable to any graph. With this relaxation, the values of $x$ might be fractional. We shall deal with this issue by rounding them to $\{-1,1\}$ after obtaining the solution. We also note that the constraint
\begin{equation}
  \transpose{x}x=\|x\|_2^2=n \label{sumn}
\end{equation}
holds for $x\in\{-1,1\}^n$, and so it provides a strengthening of the continuous relaxation to $x\in[-1,1]^n$. We therefore obtain a relaxed formulation of the minimum balanced two-way partitioning problem as follows:
\begin{equation}
  \left. \begin{array}{rrcl}
    \min\limits_{x\in[-1,1]^n} & \frac{1}{4} \transpose{x} L x & & \\
    \mbox{s.t.} & \transpose{\mathbf{1}}{x} &=& 0 \\
                & \|x\|_2^2 &=& n.
  \end{array} \right\} \label{balmincut}
\end{equation}
We remark that, by construction, $L$ is a {\it diagonally dominant} (dd) symmetric matrix with non-negative diagonal, namely it satisfies
\begin{equation}
  \forall i\le n \quad L_{ii} \ge \sum\limits_{j\not=i} |L_{ij}| \label{dddef}
\end{equation}
(in fact, $L$ satisfies Eq.~\eqref{dddef} at equality). Since all dd matrices are also psd \cite{wikipedia_dd}, $f(x)$ is a convex function. This means that Eq.~\eqref{balmincut} is a cQP, which can be solved at global optimality in polynomial time \cite{vavasis}.

By \cite{fiedler}, there is another polynomial time method for solving Eq.~\eqref{dddef}, which is generally more efficient than solving a cQP in polynomial time using a Nonlinear Programming (NLP) solver. This method concerns the second-smallest eigenvalue of $L$ (called {\it algebraic connectivity}) and its corresponding eigenvector. Let $\lambda_1\le\lambda_2\le\dots\le\lambda_n$ be the ordered eigenvalues of $L$ and $u_1,\ldots,u_n$ be the corresponding eigenvectors, normalized so that $\|u_i\|_2^2=n$ for all $i\le n$. It is known that $u_1={\bf 1}$, $\lambda_1=0$ and, if $G$ is connected, $\lambda_2>0$ \cite{merris,bollobas}. By the definition of eigenvalue and eigenvector, we have
\begin{equation}
  \forall i\le n\quad Lu_i=\lambda_iu_i\quad\Rightarrow\quad\transpose{u_i} Lu_i = \lambda_i\transpose{u_i}u_i = \lambda_i \|u_i\|_2^2=\lambda_i n. \label{defEE}
\end{equation}
Because of the orthogonality of the eigenvectors, if $i\ge 2$ we have $u_iu_1=0$, which implies $u_2{\bf 1}=0$ (i.e.~$u_2$ satisfies Eq.~\eqref{balanced}). We recall that eigenvectors are normalized so that $\|u_i\|_2^2=n$ for all $i\le n$ (in particular, $u_2$ satisfies Eq.~\eqref{sumn}). By Eq.~\eqref{defEE}, since $\lambda_1=0$, $\lambda_2$ yields the smallest nontrivial objective function value $\frac{n}{4}\lambda_2$ with solution $\bar{x}=u_2$, which is therefore a solution of Eq.~\eqref{balmincut}.

\begin{theorem}
  The eigenvector $u_2$ corresponding to the second smallest eigenvalue $\lambda_2$ of the graph Laplacian $L$ is an optimal solution to Eq.~\eqref{balmincut}. \label{thm:spectralclust}
\end{theorem}
\begin{proof}
  Since the eigenvectors $u_1,\ldots,u_n$ are an orthogonal basis of $\mathbb{R}^n$, we can express an optimal solution as $\bar{x}=\sum_i c_i u_i$. Thus,
  \begin{equation}
    \transpose{x}Lx = \sum\limits_{i,j} c_ic_j\transpose{u}_i Q u_j = \sum\limits_{i,j} c_ic_j\lambda_j\transpose{u}_i u_j = n\sum_{i>1} c_i^2\lambda_i. \label{xLx}
  \end{equation}
  The last equality in Eq.~\eqref{xLx} follows because $Lu_i=\lambda_iu_i$ for all $i\le n$, $\transpose{u}_iu_j=0$ for each $i\not=j$, and $\lambda_1=0$. Since $u_1=\mathbf{1}$ and by eigenvector orthogonality, letting $\transpose{\mathbf{1}}\bar{x}=0$ yields $c_1=0$. Lastly, requiring $\|\bar{x}\|_2=n$, again by eigenvector orthogonality, yields
  \begin{eqnarray}
    \big\|\sum_{i>1} c_i u_i \big\|_2^2 &=& \big\langle \sum_{i>1} c_i u_i, \sum_{j>1} c_j u_j\big\rangle = \sum_{i,j>1} c_i c_j \langle u_i,u_j\rangle \nonumber \\
    &=& \sum_{i>1} c_i^2 \|u_i\|_2^2 = n\sum_{i>1} c_i^2 = n. \label{xast1}
  \end{eqnarray}
  After replacing $c_i^2$ by $y_i$ in Eq.~\eqref{xLx}-\eqref{xast1}, we can reformulate Eq.~\eqref{balmincut} as
  \begin{equation*}
      n\,\min\big\{ \sum\limits_{i>1} \lambda_iy_i \;|\; \sum\limits_{i>1} y_i = 1\land y\ge 0\big\},
  \end{equation*}
  which is equivalent to finding the convex combination of $\lambda_2,\ldots,\lambda_n$ with smallest value. Since $\lambda_2\le\lambda_i$ for all $i>2$, the smallest value is achieved at $y_2=1$ and $y_i=0$ for all $i>2$. Hence $\bar{x}=u_2$ as claimed. 
\end{proof}

Normally, the components of $\bar{x}$ obtained this way are not in $\{-1,1\}$. We round $\bar{x}_i$ to its closest value in $\{-1,1\}$, breaking ties in such a way as to keep the bisection balanced. We then obtain a practically efficient approximation of the minimum balanced cut.

\subsubsection{Modularity clustering}
\label{s:modularity}
Modularity, first introduced in \cite{newmanPRE}, is a measure for evaluating the quality of a clustering of vertices in a graph $G=(V,E)$ with a weight function $w:E\to\mathbb{R}_+$ on the edges. We let $n=|V|$ and $m=|E|$. Given a vertex clustering $\mathcal{C}=(C_1,\ldots,C_k)$, where each $C_i\subseteq V$, $C_i\cap C_j=\varnothing$ for each $i\not=j$, and $\bigcup_i C_i=V$, the {\it modularity} of $\mathcal{C}$ is the proportion of edges in $E$ that fall within a cluster minus the expected proportion of the same quantity if edges were distributed at random while keeping the vertex degrees constant. This definition is not so easy to understand, so we shall assume for simplicity that $w_{uv}=1$ for all $\{u,v\}\in E$ and $w_{uv}=0$ otherwise. We give a more formal definition of modularity, and comment on its construction.

The ``fraction of the edges that fall within a cluster'' is
\begin{equation*}
  \frac{1}{m}\sum_{i\le k} \sum_{u,v\in C_k\atop \{u,v\}\in E} 1 = \frac{1}{2m} \sum_{i\le k\atop (u,v)\in (C_k)^2} w_{uv}
\end{equation*}
where $w_{uv}=w_{vu}$ turns out to be the $(u,v)$-th component of the $n\times n$ symmetric incidence matrix of the edge set $E$ in $V\times V$ --- thus we divide by $2m$ rather than $m$ in the right hand side (RHS) of the above equation. The ``same quantity if edges were distributed at random while keeping the vertex degrees constant'' is the probability that a pair of vertices $u,v$  belongs to the edge set of a random graph on $V$. If we were computing this probability over random graphs sampled uniformly over all graphs on $V$ with $m$ edges, this probability would be $1/m$; but since we only want to consider graphs with the same degree sequence as $G$, the probability is $\frac{|N(u)|\,|N(v)|}{2m}$ \cite{lehmann}. Here is an informal explanation: given vertices $u,v$, there are $k_u=|N(u)|$ ``half-edges'' out of $u$, and $k_v=|N(v)|$ out of $v$, which could come together to form an edge between $u$ and $v$ (over a total of $2m$ ``half-edges''). Thus we obtain a modularity
\begin{equation*}
  \mu(\mathcal{C})=\frac{1}{2m} \sum_{(u,v)\in C^2 \atop C\in\mathcal{C}} (w_{uv} - k_uk_v/(2m))
\end{equation*}
for the clustering $\mathcal{C}$.

We now introduce binary variables $x_{uv}$ which have value $1$ if $u,v\in V$ are in the same cluster, and $0$ otherwise. This allows us to rewrite the modularity as:
\begin{eqnarray}
  \mu(x) &=& \frac{1}{2m}\sum\limits_{u\not=v\in V} (w_{uv} - k_uk_v/(2m))x_{uv} \nonumber \\ &=& \frac{1}{m}\sum\limits_{u<v\in V} (w_{uv} - k_uk_v/(2m))x_{uv}.\label{eq:modularity1}
\end{eqnarray}
Following \cite{pre-10c}, we can reformulate the modularity maximization problem to a clique partitioning problem with the following formulation:
\begin{equation}
  \left.\begin{array}{rrcl}
    \max &\mu(x) && \\
    \forall 1\le i<j<k\le n & \quad  x_{ij} + x_{jk} - x_{ik} &\le& 1 \\
    \forall 1\le i<j<k\le n & \quad  x_{ij} - x_{jk} + x_{ik} &\le& 1 \\
    \forall 1\le i<j<k\le n & \quad  - x_{ij} + x_{jk} + x_{ik} &\le& 1 \\
    \forall 1\le i<j\le n & \quad  x_{ij} &\in& \{0,1\},
  \end{array}\right\}\label{eq:modularity2}
\end{equation}
which is a BLP formulation. The weighted variant of this problem yields a formulation like Eq.~\eqref{eq:modularity2} where $w$ are the edge weights and $k_u=\sum_{\{u,v\}\in E} w_{uv}$ for all $v\not=u$ in $V$. Another variant for graphs including loops and multiple edges is described in \cite{cafieri_loops}. We note that, by Eq.~\eqref{eq:modularity2}, maximizing modularity does not require the number of clusters to be known {\it a priori}.

There is a large literature about modularity maximization and its solution methods: for a survey, see \cite[\S VI]{fortunato2}. Solution methods based on MP are of particular interest to the topics of this survey. A BLP formulation similar to Eq.~\eqref{eq:modularity2} was proposed in \cite{wagner}. Another BLP formulation with different sets of decision variables (requiring the number of clusters to be known {\it a priori}) was proposed in \cite{papageorgiou}. Some column generation approaches, which scale better in size w.r.t.~previous formulations, were proposed in \cite{pre-10c}. Some MP based heuristics are discussed in \cite{pre-11,dam-12,dimacs-12}. 

\section{Robust solution methods for the DGP}
\label{s:dgpsol}
In this section we discuss some solution methods for the DGP which can be extended to deal with cases where distances are uncertain, noisy or wrong. Most of the methods we present are based on MP. We also discuss a different (non-MP based) class of methods in Sect.~\ref{s:fastdgp}, in view of their computational efficiency.

\subsection{Mathematical programming based methods}
\label{s:dgpmp}
DGP solution methods based on MP are robust to noisy or wrong data because MP allows for: (a) modification of the objective and constraints; (b) adjoining of side constraints. Moreover, although we do not review these here, there are MP-based methodologies for ensuring robustness of solutions \cite{robustopt}, probabilistic constraints \cite{probmp}, and scenario-based stochasticity \cite{stochmp}, which can be applied to the formulations in this section.

\subsubsection{Unconstrained quartic formulation}
\label{s:unconstrained}
A system of equations such as Eq.~\eqref{eq:dgp2} is itself a MP formulation with objective function identically equal to zero, and $X=\mathbb{R}^{nK}$. It therefore belongs to the QCP class. In practice, solvers for this class perform rather poorly when given Eq.~\eqref{eq:dgp2} as input \cite{lln1}. Much better performances can be obtained by solving the following unconstrained formulation:
\begin{equation}
  \min \sum\limits_{\{u,v\}\in E} \big(\|x_u-x_v\|_2^2 - d_{uv}^2\big)^2.\label{dgpuncon}
\end{equation}
We note that Eq.~\eqref{dgpuncon} consists in the minimization of a polynomial of degree four. It belongs to the class of nonconvex NLP formulations. In general, this is an {\bf NP}-hard class \cite{undecminlp}, which is not surprising, as it formulates the DGP which is itself an {\bf NP}-hard problem. Very good empirical results can be obtained on the DGP by solving Eq.~\eqref{dgpuncon} with a local NLP solver (such as e.g.~IPOPT \cite{ipopt} or SNOPT \cite{snopt7}) from a good starting point \cite{lln1}. This is the reason why Eq.~\eqref{dgpuncon} is very important: it can be used to ``refine'' solutions obtained with other methods, as it suffices to let such solutions be starting points given to a local solver acting on Eq.~\eqref{dgpuncon}.

Even if the distances $d_{uv}$ are noisy or wrong, optimizing Eq.~\eqref{dgpuncon} can yield good approximate realizations. If the uncertainty on the distance values is modelled using an interval $[d^L_{uv},d^U_{uv}]$ for each edge $\{u,v\}$, the following function \cite{mdgpsurvey} can be optimized instead of Eq.~\eqref{dgpuncon}:
\begin{equation}
  \min \sum\limits_{\{u,v\}\in E} \big(\max(0,(d^L_{uv})^2-\|x_u-x_v\|_2^2)+\max(0, \|x_u-x_v\|_2^2 - (d^U_{uv})^2)\big).\label{idgpuncon}
\end{equation}
The DGP variant where distances are intervals instead of values is known as the {\sc interval DGP} (iDGP) \cite{idgpsurvey,bpinterval}.

Note that Eq.~\eqref{idgpuncon} involves binary $\max$ functions with two arguments. Relatively few MP user interfaces/solvers would accept this function. To overcome this issue, we linearize (see Sect.~\ref{s:linearization}) the two $\max$ terms by two sets of added decision variables $y,z$, and obtain
\begin{equation}
  \left.\begin{array}{rrcl}
    \min & \sum\limits_{\{u,v\}\in E} (y_{uv}+z_{uv}) && \\
    \forall \{u,v\}\in E & \|x_u-x_v\|_2^2 &\ge& (d^L_{uv})^2 - y_{uv}  \\
    \forall \{u,v\}\in E & \|x_u-x_v\|_2^2 &\le& (d^U_{uv})^2 + z_{uv} \\
    & y,z &\ge& 0,
  \end{array}\right\} \label{idgpuncon2}
\end{equation}
which follows from Eq.~\eqref{idgpuncon} because of the objective function direction, and because $a\ge\max(b,c)$ is equivalent to $a\ge b\land a\ge c$. We note that Eq.~\eqref{idgpuncon2} is no longer an unconstrained quartic, however, but a QCP. It expresses a minimization of penalty variables to the quadratic inequality system
\begin{equation}
  \forall \{u,v\}\in E \quad (d_{uv}^L)^2 \le \|x_u-x_v\|_2^2 \le (d_{uv}^U)^2.\label{idgp}
\end{equation}
We also note that many local NLP solvers take very arbitrary functions in input (such as functions expressed by computer code), so the reformulation Eq.~\eqref{idgpuncon2} may be unnecessary when only locally optimal solutions of Eq.~\eqref{idgpuncon} are needed.

\subsubsection{Constrained quadratic formulations}
\label{s:qcqp}
We propose two formulations in this section. The first is derived directly from Eq.~\eqref{eq:dgp2}:
\begin{equation}
  \left.\begin{array}{rrcl}
    \min & \sum\limits_{\{u,v\}\in E} s_{uv}^2 && \\
    \forall \{u,v\}\in E & \|x_u-x_v\|_2^2 &=& d_{uv}^2 + s_{uv}.
    \end{array}\right\} \label{dgpqcqp}
\end{equation}
We note that Eq.~\eqref{dgpqcqp} is a QCQP formulation. Similarly to Eq.~\eqref{idgpuncon2} it uses additional variables to penalize feasibility errors w.r.t.~\eqref{eq:dgp2}. Differently from Eq.~\eqref{idgpuncon2}, however, it removes the need for two separate variables to model slack and surplus errors. Instead, $s_{uv}$ is unconstrained, and can therefore take any value. The objective, however, minimizes the sum of the squares of the components of $s$. In practice, Eq.~\eqref{dgpqcqp} performs much better than Eq.~\eqref{eq:dgp2}; on average, the performance is comparable to that of Eq.~\eqref{dgpuncon}. We remark that Eq.~\eqref{dgpqcqp} has a convex objective function but nonconvex constraints. 

The second formulation we propose is an exact reformulation of Eq.~\eqref{dgpuncon}. First, we replace the minimization of squared errors by absolute values, yielding
\begin{equation*}
  \min \sum\limits_{\{u,v\}\in E} \big|\|x_u-x_v\|_2^2 - d_{uv}^2\big|, 
\end{equation*}
which clearly has the same set of global optima as Eq.~\eqref{dgpuncon}. We then rewrite this similarly to Eq.~\eqref{idgpuncon2} as follows:
\begin{equation*}
  \left.\begin{array}{rrcl}
    \min & \sum\limits_{\{u,v\}\in E} (y_{uv}+z_{uv}) && \\
    \forall \{u,v\}\in E & \|x_u-x_v\|_2^2 &\ge& d_{uv}^2 - y_{uv}  \\
    \forall \{u,v\}\in E & \|x_u-x_v\|_2^2 &\le& d_{uv}^2 + z_{uv} \\
    & y,z &\ge& 0,
  \end{array}\right\} 
\end{equation*}
which, again, does not change the global optima. Next, we note that we can fix $z_{uv}=0$ without changing global optima, since they all have the property that $z_{uv}=0$. Now we replace $y_{uv}$ in the objective function by $d_{uv}^2-\|x_u-x_v\|_2^2$, which we can do without changing the optima since the first set of constraints reads $y_{uv}\ge d_{uv}^2-\|x_u-x_v\|_2^2$. We can discard the constant $d_{uv}^2$ from the objective, since adding constants to the objective does not change optima, and change $\min -f$ to $-\max f$, yielding:
\begin{equation}
  \left.\begin{array}{rrcl}
    \max & \sum\limits_{\{u,v\}\in E} \|x_u-x_v\|_2^2  && \\
    \forall \{u,v\}\in E & \|x_u-x_v\|_2^2 &\le& d_{uv}^2,
  \end{array}\right\} \label{pushpull}
\end{equation}
which is a QCQP known as the ``push-and-pull'' formulation of the DGP, since the constraints ensure that $x_u,x_v$ are pushed closer together, while the objective attempts to pull them apart \cite[\S 2.2.1]{mwu}.

Contrariwise to Eq.~\eqref{dgpqcqp}, Eq.~\eqref{pushpull} has a nonconvex (in fact, concave) objective function and convex constraints. Empirically, this often turns out to be somewhat easier than tackling the reverse situation. The theoretical justification is that finding a feasible solution in a nonconvex set is a hard task in general, whereas finding local optima of a nonconvex function in a convex set is tractable: the same cannot be said for global optima, but in practice one is often satisfied with ``good'' local optima.

\subsubsection{Semidefinite programming}
\label{s:sdp}
SDP is linear optimization over the cone of psd matrices, which is convex: if $A,B$ are two psd matrices, $C=\alpha A + (1-\alpha) B$ is psd for $\alpha\in[0,1]$. Suppose there is $x\in\mathbb{R}^n$ such that $\transpose{x}Cx<0$. Then $\alpha \transpose{x} A x + (1-\alpha) \transpose{x} B x <0$, so $0\le \alpha \transpose{x} A x < -(1-\alpha) \transpose{x} B x\le 0$, i.e.~$0<0$, which is a contradiction, hence $C$ is also psd, as claimed. Therefore, SDP is a subclass of cNLP.

The SDP formulation we propose is a relaxation of Eq.~\eqref{eq:dgp2}. First, we write $\|x_u-x_v\|_2^2 = \langle x_u,x_u\rangle + \langle x_v,x_v\rangle - 2\langle x_u,x_v\rangle$. Then we linearize all of the scalar products by means of additional variables $X_{uv}$:
\begin{eqnarray*}
    \forall \{u,v\}\in E \quad X_{uu}+X_{vv}-2X_{uv} &=& d_{uv}^2  \\
     X &=& x\transpose{x}.
\end{eqnarray*}
We note that $X=x\transpose{x}$ constitutes the whole set of defining constraints $X_{uv}=\langle x_u,x_v\rangle$ (for each $u,v\le n$) introduced by the linearization procedure (Sect.~\ref{s:linearization}).

The relaxation we envisage does not entirely drop the defining constraints, as in Sect.~\ref{s:linearization}. Instead, it relaxes them from $X-x\transpose{x}=0$ to $X-x\transpose{x}\succeq 0$. In other words, instead of requiring that all of the eigenvalues of the matrix $X-x\transpose{x}$ are zero, we simply require that they should be $\ge 0$. Moreover, since the original variables $x$ do not appear anywhere else, we can simply require $X\succeq 0$, obtaining:
\begin{equation}
  \left.\begin{array}{rrcl}
    \forall \{u,v\}\in E & X_{uu}+X_{vv}-2X_{uv} &=& d_{uv}^2 \\
    & X &\succeq & 0.\label{sdpfeas}
  \end{array}\right\}   
\end{equation}

The SDP relaxation in Eq.~\eqref{sdpfeas} has the property that it provides a solution $\bar{X}$, which is an $n\times n$ symmetric matrix. Spectral decomposition of $\bar{X}$ yields $P\Lambda\transpose{P}$, where $P$ is a matrix of eigenvectors and $\Lambda=\diag{\lambda}$ where $\lambda$ is a vector of eigenvalues of $\bar{X}$. Since $\bar{X}$ is psd, $\lambda\ge 0$, which means that $\sqrt{\Lambda}$ is a real matrix. Therefore, by setting $Y=P\sqrt{\Lambda}$ we have that
\[Y\transpose{Y}=(P\sqrt{\Lambda})\transpose{(P\sqrt{\Lambda})}= P\sqrt{\Lambda}\sqrt{\Lambda}\transpose{P}=P\Lambda\transpose{P}=\bar{X},\]
which implies that $\bar{X}$ is the Gram matrix of $Y$. Thus we can take $Y$ to be a realization satisfying Eq.~\eqref{eq:dgp2}. The only issue is that $Y$, as an $n\times n$ matrix, is a realization in $n$ dimensions rather than $K$. Naturally, $\rank{Y}=\rank{\bar{X}}$ need not be equal to $n$, but could be lower; in fact, in order to find a realization of the given graph, we would like to find a solution $\bar{X}$ with rank at most $K$. Imposing this constraint is equivalent to asking that $X=x\transpose{x}$ (which have been relaxed in Eq.~\eqref{sdpfeas}).

We note that Eq.~\eqref{sdpfeas} is a pure feasibility problem. Every SDP solver, however, also accepts an objective function as input. In absence of a ``natural'' objective in a pure feasibility problem, we can devise one to heuristically direct the search towards parts of the psd cone which we believe might contain ``good'' solutions. A popular choice is
\begin{eqnarray*}
  \min \trace{X} &=& \min \trace{P\Lambda \transpose{P}}
  = \min\trace{P\transpose{P}\Lambda}= \\
  &=& \min\trace{PP^{-1}\Lambda}
  = \min \lambda_1+\cdots+\lambda_n, 
\end{eqnarray*}
where \textsf{tr} is the trace, the first equality follows by spectral decomposition (with $P$ a matrix of eigenvectors and $\Lambda$ a diagonal matrix of eigenvalues of $X$), the second by commutativity of matrix products under the trace, the third by orthogonality of eigenvectors, and the last by definition of trace. This aims at minimizing the sum of the eigenvalues of $X$, hoping this will decrease the rank of $\bar{X}$. 

For the DGP applied to protein conformation (Sect.~\ref{s:science}), the objective function
\begin{equation*}
  \min \sum\limits_{\{u,v\}\in E} (X_{uu} + X_{vv} - 2X_{uv}) \label{pushpullobj}
\end{equation*}
was empirically found to be a good choice \cite[\S 2.1]{isco16}. More (unpublished) experimentation showed that the scalarization of the two objectives:
\begin{equation}
  \min \sum\limits_{\{u,v\}\in E} (X_{uu} + X_{vv} - 2X_{uv}) + \gamma\trace{X},\label{sdpobj}
\end{equation}
with $\gamma$ in the range $O(10^{-2})$-$O(10^{-3})$, is a good objective function for solving Eq.~\eqref{sdpfeas} when it is applied to protein conformation.

In the majority of cases, solving SDP relaxations does not yield solution matrices with rank $K$, even with objective functions such as Eq.~\eqref{sdpobj}. We discuss methods for constructing an approximate rank $K$ realization from $\bar{X}$ in Sect.~\ref{s:dimred}.

SDP is one of those problems which is not known to be in {\bf P} (nor {\bf NP}-complete) in the Turing machine model. It is, however, known that SDPs can be solved in polynomial time up to a desired error tolerance $\epsilon>0$, with the complexity depending on $\frac{1}{\epsilon}$ as well as the instance size. Currently, however, the main issue with SDP is technological: state-of-the art solvers do not scale all that well with size. One of the reasons is that $K$ is usually fixed (and small) with respect to $n$, so the while the original problem has $O(n)$ variables, the SDP relaxation has $O(n^2)$. Another reason is that the Interior Point Method (IPM), which often features as a ``state of the art'' SDP solver, has a relatively high computational complexity \cite{wrightIPM}: a ``big oh'' notation estimate of $O(\max(m,n)mn^{2.5})$ is given in Bubeck's blog at ORFE, Princeton.\footnote{\url{blogs.princeton.edu/imabandit/2013/02/19/orf523-ipms-for-lps-and-sdps/}}

\subsubsection{Diagonally dominant programming}
\label{s:ddp}
In order to address the size limitations of SDP, we employ some interesting linear approximations of the psd cone proposed in \cite{majumdar,ahmadimajumdar}. An $n\times n$ real symmetric matrix $X$ is diagonally dominant (dd) if
\begin{equation}
  \forall i\le n\quad \sum\limits_{j\not=i} |X_{ij}| \le X_{ii}. \label{eq:dd}
\end{equation}
As remarked in Sect.~\ref{s:spclust}, it is well known that every dd matrix is also psd, while the converse may not hold. Specifically, the set of dd matrices form a sub-cone of the cone of psd matrices \cite{barker2}.

The interest of dd matrices is that, by linearization of the absolute value terms, Eq.~\eqref{eq:dd} can be reformulated so it becomes linear: we introduce an added matrix $T$ of decision variables, then write:
\begin{eqnarray}
  \forall i\le n\quad \sum\limits_{j\not=i} T_{ij} &\le& X_{ii} \label{eq:ddp1} \\
  -T \le X &\le& T, \label{eq:ddp2}
\end{eqnarray}
which are linear constraints equivalent to Eq.~\eqref{eq:dd} \cite[Thm.~10]{ahmadimajumdar}. One can see this easily whenever $X\ge 0$ or $X\le 0$. Note that
\begin{eqnarray*}
  \forall i\le n\quad X_{ii}&\ge& \sum\limits_{j\not=i} T_{ij}\ge \sum\limits_{j\not=i} X_{ij}\\
  \forall i\le n\quad X_{ii}&\ge& \sum\limits_{j\not=i} T_{ij}\ge \sum\limits_{j\not=i} -X_{ij}
\end{eqnarray*}
follow directly from Eq.~\eqref{eq:ddp1}-\eqref{eq:ddp2}. Now one of the RHSs is equal to $\sum_{j\not=i} |X_{ij}|$, which implies Eq.~\eqref{eq:dd}. For the general case, the argument uses the extreme points of Eq.~\eqref{eq:ddp1}-\eqref{eq:ddp2} and elimination of $T$ by projection.

We can now approximate Eq.~\eqref{sdpfeas} by the pure feasibility LP:
\begin{equation}
  \left.\begin{array}{rrcl}
    \forall \{u,v\}\in E & X_{uu}+X_{vv}-2X_{uv} &=& d_{uv}^2 \\
    \forall i\le n & \sum\limits_{j\not=i} T_{ij} &\le& X_{ii} \\
     &-T \le X &\le& T,\label{lpfeas}
  \end{array}\right\}   
\end{equation}
which we call a {\it diagonally dominant program} (DDP). As in Eq.~\eqref{sdpfeas}, we do not explicitly give an objective function, since it depends on the application. Since the DDP in Eq.~\eqref{lpfeas} is an inner approximation of the corresponding SDP in Eq.~\eqref{sdpfeas}, the DDP feasible set is a subset of that of the SDP. This situation yields both an advantage and a disadvantage: any solution $\tilde{X}$ of the DDP is psd, and can be obtained at a smaller computational cost; however, the DDP might be infeasible even if the corresponding SDP is feasible (see Fig.~\ref{f:ddp1}, left).
\begin{figure}[!ht]
  \begin{center}
    \includegraphics[width=5cm]{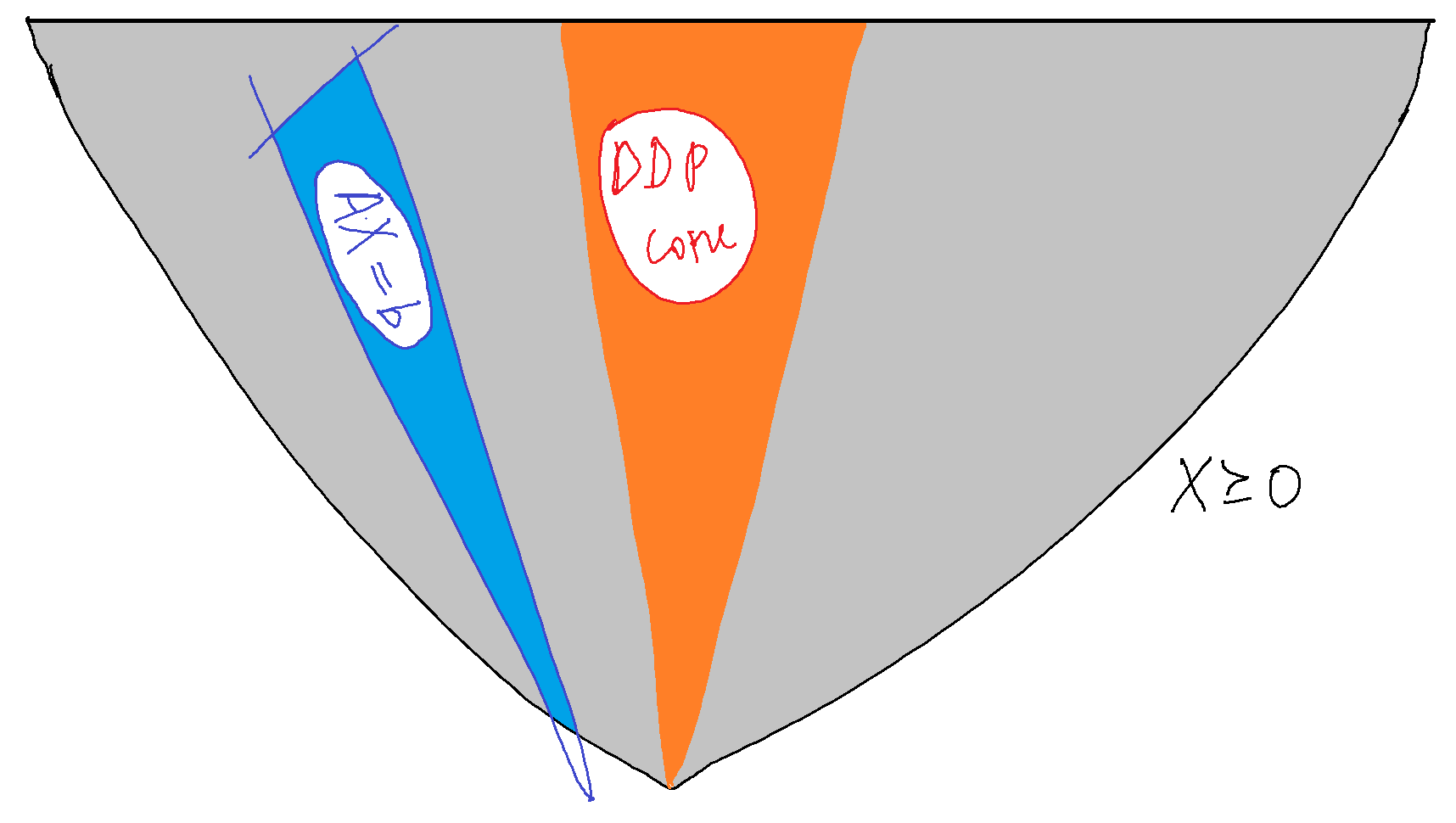} \hfill
    \includegraphics[width=5cm]{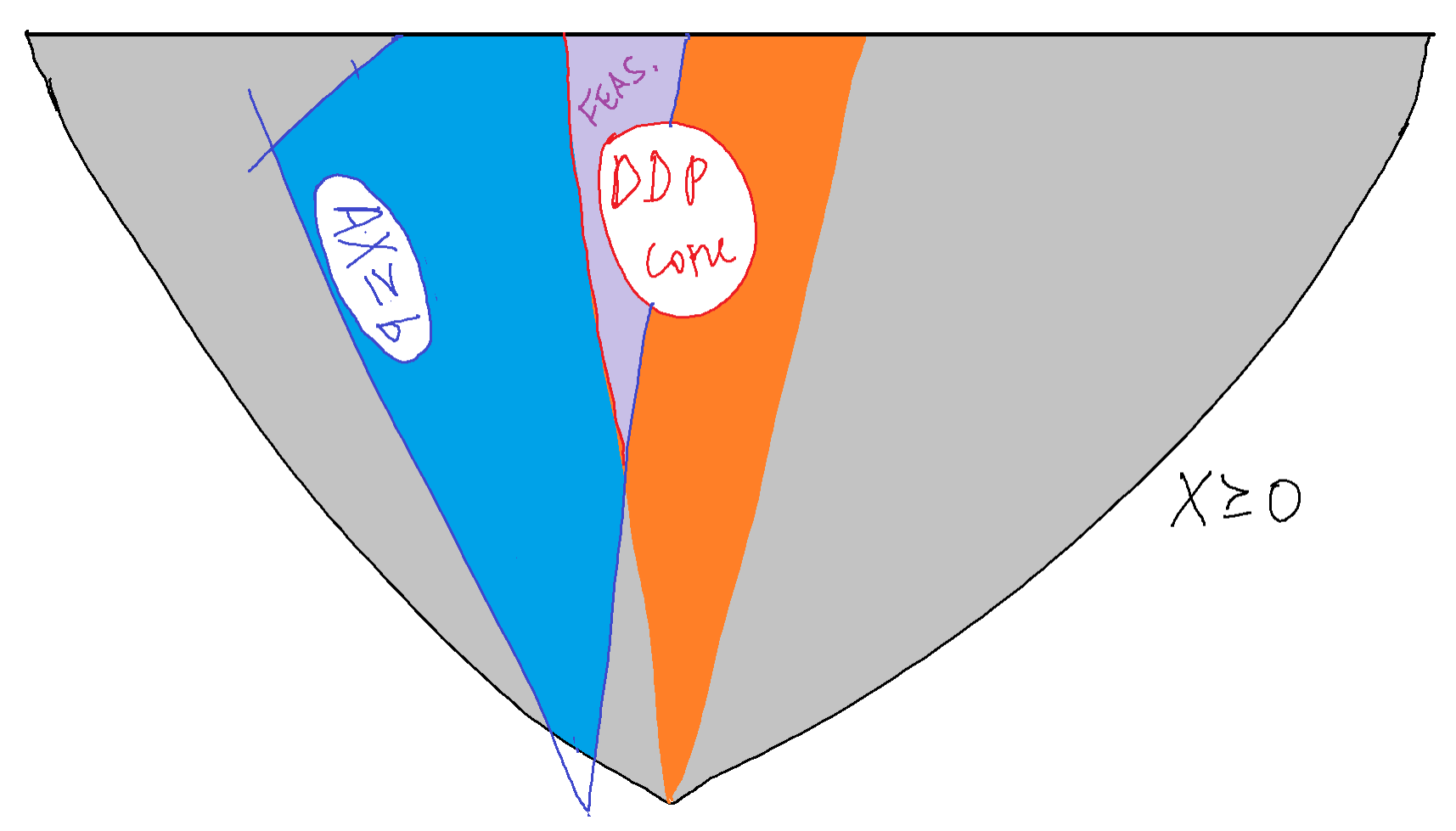}    
  \end{center}
  \caption{On the left, the DDP is infeasible even if the SDP is not; on the right, a relaxed set of constraints makes the DDP feasible.}
  \label{f:ddp1}
\end{figure}
In order to decrease the risk of infeasibility of Eq.~\eqref{lpfeas}, we relax the equation constraints to inequality, and impose an objective as in the push-and-pull formulation Eq.~\eqref{pushpull}:
\begin{equation}
  \left.\begin{array}{rrcl}
    \max & \sum\limits_{\{u,v\}\in E} (X_{uu}+X_{vv}-2X_{uv}) && \\
    \forall \{u,v\}\in E & X_{uu}+X_{vv}-2X_{uv} &\le & d_{uv}^2 \\
    \forall i\le n & \sum\limits_{j\not=i} T_{ij} &\le& X_{ii} \\
     &-T \le X &\le& T.\label{ddp}
  \end{array}\right\}   
\end{equation}
This makes the DDP feasible set larger, which means it is more likely to be feasible (see Fig.~\ref{f:ddp1}, right). Eq.~\eqref{ddp} was successfully tested on protein graphs in \cite{isco16}.

If $C$ is any cone in $\mathbb{R}^n$, the {\it dual cone} $C^\ast$ is defined as:
\[C^\ast = \{y\in\mathbb{R}^n \;|\; \forall x\in C\;\langle x,y\rangle \ge 0\}.\]
Note that the dual cone contains the set of vectors making a non-obtuse angle with all of the vectors in the original (primal) cone. We can exploit the dual dd cone in order to provide another DDP formulation for the DGP which turns out to be an outer approximation. Outer approximations have symmetric advantages and disadvantages w.r.t.~the inner ones: if the original SDP is feasible, than the outer DDP approximation is also feasible (but the DDP may be feasible even if the SDP is not); however, the solution $\tilde{X}$ we obtain from the DDP need not be a psd matrix. Some computational experience related to \cite{salgado3} showed that it often happens that more or less  half of the eigenvalues of $\tilde{X}$ are negative. 

We now turn to the actual DDP formulation related to the dual dd cone. A cone $C$ of $n\times n$ real symmetric matrices is {\it finitely generated} by a set $\mathcal{X}$ of matrices if:
\[\forall X\in C \; \exists\delta\in\mathbb{R}^{|\mathcal{X}|}_+ \quad X = \sum\limits_{x\in\mathcal{X}} \delta_x x\transpose{x}.\]
It turns out \cite{barker2} that the dd cone is finitely generated by
\[\mathcal{X}_{\mathsf{dd}} = \{e_i\;|\;i\le n\}\cup\{e_i\pm e_j\;|\;i<j\le n\},\]
where $e_1,\ldots,e_n$ is the standard orthogonal basis of $\mathbb{R}^n$. This is proved in \cite{barker2} by showing that the following rank-one matrices are extreme rays of the dd cone:
\begin{itemize}
  \item $E_{ii} = \diag{e_i}$, where $e_i=\transpose{(0,\ldots,0,1_i,0,\ldots,0)}$;
  \item $E_{ij}^+$ has a minor $\left(\begin{array}{cc} 1_{ii}&1_{ij}\\1_{ji}&1_{jj}\end{array}\right)$ and is zero elsewhere;
  \item $E_{ij}^-$ has a minor $\left(\begin{array}{rr} 1_{ii}&-1_{ij}\\-1_{ji}&1_{jj}\end{array}\right)$ and is zero elsewhere,
\end{itemize}
and, moreover, that the extreme rays are generated by the standard basis vectors as follows:
\begin{eqnarray*}
  \forall i\le n \quad E_{ii} &=& e_i\transpose{e}_i \\
  \forall i<j\le n \quad E_{ij}^+ &=& (e_i+e_j)\transpose{(e_i+e_j)} \\
  \forall i<j\le n \quad E_{ij}^- &=& (e_i-e_j)\transpose{(e_i-e_j)}.
\end{eqnarray*}
This observation allowed Ahmadi and his co-authors to write the DDP formulation Eq.~\eqref{ddp} in terms of the extreme rays $E_{ii},E_{ij}^{\pm}$ \cite{ahmadimajumdar}, and also to define a column generation algorithms over them \cite{ahmadi}.

If a matrix cone is finitely generated, the dual cone has the same property. Let $\mathbb{S}_n$ be the set of real symmetric $n\times n$ matrices; for $A,B\in\mathbb{S}_n$ we define an inner product $\langle A,B\rangle=A\bullet B\triangleq\trace{A\transpose{B}}$.
\begin{theorem}
  Assume $C$ is finitely generated by $\mathcal{X}$. Then $C^\ast$ is also finitely generated. Specifically, $C^\ast = \{Y\in\mathbb{S}_n\;|\;\forall x\in\mathcal{X}\; (Y\bullet x\transpose{x}\ge 0)\}$. \label{thm:dualcone}
\end{theorem}
\begin{proof}
  By assumption, $C = \{X\in\mathbb{S}_n\;|\;\exists \delta\in\mathbb{R}^{|\mathcal{X}|}_+ \; X=\sum_{x\in\mathcal{X}} \delta_x x\transpose{x}\}$. \\ [0.3em]
  ($\Rightarrow$) Let $Y\in\mathbb{S}_n$ be such that, for each $x\in\mathcal{X}$, we have $Y\bullet x\transpose{x}\ge 0$. We are going to show that $Y\in C^\ast$, which, by definition, consists of all matrices $Y$ such that for all $X\in C$, $Y\bullet X\ge 0$. Note that, for all $X\in C$, we have $X=\sum_{x\in\mathcal{X}} \delta_x x\transpose{x}$ (by finite generation). Hence $Y\bullet X = \sum_x \delta_x Y\bullet x\transpose{x}\ge 0$ (by definition of $Y$), whence $Y\in C^\ast$. \\ [0.3em]
  ($\Leftarrow$) Suppose $Z\in C^\ast\smallsetminus \{Y\;|\;\forall x\in\mathcal{X}\;(Y\bullet x\transpose{x}\ge 0)\}$. Then there is $\mathcal{X}'\subset\mathcal{X}$ such that for any $x\in\mathcal{X}'$ we have $Z\bullet x\transpose{x}<0$. Consider any $Y=\sum_{x\in\mathcal{X}'} \delta_x x\transpose{x}\in C$ with $\delta\ge 0$. Then $Z\bullet Y=\sum_{x\in\mathcal{X}'}\delta_{x} Z\bullet x\transpose{x}<0$, so $Z\not\in C^\ast$, which is a contradiction. Therefore $C^\ast=\{Y\;|\;\forall x\in\mathcal{X}\;(Y\bullet x\transpose{x}\ge 0)\}$ as claimed. 
\end{proof}

We are going to exploit Thm.~\ref{thm:dualcone} in order to derive an explicit formulation of the following DDP formulation based on the dual cone $C_{\mathsf{dd}}^\ast$ of the dd cone $C_{\mathsf{dd}}$ finitely generated by $\mathcal{X}_{\mathsf{dd}}$:
\begin{equation*}
  \left.\begin{array}{rrcl}
    \forall \{u,v\}\in E & X_{uu}+X_{vv}-2X_{uv} &=& d_{uv}^2 \\
    & X &\in& C^\ast_{\mathsf{dd}}. 
  \end{array}\right\}   
\end{equation*}
We remark that $X\bullet v\transpose{v}=\transpose{v}Xv$ for each $v\in\mathbb{R}^n$. By Thm.~\ref{thm:dualcone}, $X\in C^\ast_{\mathsf{dd}}$ can be restated as
$\forall v\in\mathcal{X}_{\mathsf{dd}}\; \transpose{v}Xv\ge 0$. We obtain the following LP formulation:
\begin{equation}
  \left.\begin{array}{rrcl}
    \max & \sum\limits_{\{u,v\}\in E} (X_{uu}+X_{vv}-2X_{uv}) && \\
    \forall \{u,v\}\in E & X_{uu}+X_{vv}-2X_{uv} &=& d_{uv}^2 \\
    \forall v\in\mathcal{X}_{\mathsf{dd}} & \transpose{v}Xv &\ge & 0. \label{dualddp}
  \end{array}\right\}   
\end{equation}
With respect to the primal DDP, the dual DDP formulation in Eq.~\eqref{dualddp} provides a very tight bound to the objective function value of the push-and-pull SDP formulation Eq.~\eqref{pushpull}. On the other hand, the solution $\bar{X}$ is usually far from being a psd matrix.

\subsection{Fast high-dimensional methods}
\label{s:fastdgp}
In Sect.~\ref{s:dgpmp} we surveyed methods based on MP, which are very flexible, insofar as they can accommodate side constraints and noisy data, but computationally demanding. In this section we discuss two very fast, yet robust, methods for embeddings graphs in Euclidean spaces. 

\subsubsection{Incidence vectors}
\label{s:incvec}
The simplest, and most naive methods for mapping graphs into vectors are given by exploiting various incidence information in the graph structure. By contrast, the resulting embeddings are unrelated to Eq.~\eqref{eq:dgp2}.

Given a simple graph $G=(V,E)$ with $|V|=n$, $|E|=m$ and edge weight function $w:E\to\mathbb{R}_+$, we present two approaches: one which outputs an $n\times n$ matrix, and one which outputs a single vector in $\mathbb{R}^K$ with $K=\frac{1}{2}n(n-1)$.
\begin{enumerate}
\item For each $u\in V$, let $x_u=(x_{uv}\;|\;v\in V)\in\mathbb{R}^n$ be the incidence vector of $N(u)$ on $V$, i.e.:
  \[\forall u\in V\quad x_{uv} = \left\{ \begin{array}{ll} w_{uv} & \mbox{ if } \{u,v\}\in E \\ 0 & \mbox{ otherwise.} \end{array}\right. \]
\item Let $K=\frac{1}{2}n(n-1)$, and $x^E=(x_e\;|\;e\in E)\in\mathbb{R}^K$ be the incidence vector of the edge set $E$ into the set $\{ \{i,j\}\;|\; i<j\le n\}$, i.e.:
  \[ x_e = \left\{\begin{array}{ll} w_e & \mbox{ if } e\in E \\ 0 & \mbox{ otherwise.} \end{array}\right. \]
\end{enumerate}
Both embeddings can be obtained in $O(n^2)$ time. Both embeddings are very high dimensional. So that they may be useful in practice, it is necessary to post-process them using dimensional reduction techniques (see Sect.~\ref{s:dimred}). 

\subsubsection{The universal isometric embedding}
\label{s:uie}
This method, also called {\it Fr\'echet embedding}, is remarkable in that it maps any finite metric space congruently into a set of vectors in the $\ell_\infty$ norm \cite[\S 6]{kuratowski}. No other norm allows exact congruent embeddings in vector spaces \cite{matousekmetric}. The Fr\'echet embedding provided the foundational idea for several other probabilistic approximate embeddings in various other norms and dimensions \cite{bourgain,linial}.

\begin{theorem}
  \label{thm:uie}
Given any finite metric space $(X,d)$, where $|X|=n$ and $d$ is a distance function defined on $X$, there exists an embedding $\rho:X\to\mathbb{R}^n$ such that $(\rho(X),\ell_\infty)$ is congruent to $(X,d)$. 
\end{theorem}
This theorem is surprising because of its generality in conjunction with the exactness of the result: it works on {\it any} (finite) metric space. The ``magic hat'' out of which we shall pull the vectors in $\rho(X)$ is simply the only piece of data we are given, namely the distance matrix of $X$. More precisely the $i$-th element of $X$ is mapped to the vector corresponding to the $i$-th column of the distance matrix. 
\begin{proof}
  Let $\mathsf{D}(X)$ be the distance matrix of $(X,d)$, namely $\mathsf{D}_{ij}(X)=(d(x_i,x_j))$ where $X=\{x_1,\ldots,x_n\}$. We denote $d(x_i,x_j)=d_{ij}$ for brevity. For any $j\le n$ we let $\rho(x_j)=\delta_j$, where $\delta_j$ is the $j$-th column of $\mathsf{D}(X)$. We have to show that $\|\rho(x_i)-\rho(x_j)\|_\infty=d_{ij}$ for each $i<j\le n$. By definition of the $\ell_\infty$ norm, for each $i<j\le n$ we have
  \[\|\rho(x_i)-\rho(x_j)\|_\infty=\|\delta_i-\delta_j\|_\infty=\max\limits_{k\le n}|\delta_{ik}-\delta_{jk}| = \max\limits_{k\le n}|d_{ik}-d_{jk}|.\quad(\ast)\]
By the triangular inequality on $(X,d)$, for $i<j\le n$ and $k\le n$ we have:
\begin{eqnarray*}
  && d_{ik}\le d_{ij}+d_{jk}\quad \land\quad d_{jk}\le d_{ij}+d_{ik} \\
  &\Rightarrow& d_{ik}- d_{jk}\le d_{ij}\quad \land\quad d_{jk}- d_{ik}\le d_{ij} \\
  &\Rightarrow& |d_{ik}-d_{jk}|\le d_{ij};
\end{eqnarray*}
since these inequalities are valid for each $k$, by ($\ast$) we have:
  \[\|\rho(x_i)-\rho(x_j)\|_\infty\le \max_k d_{ij} = d_{ij}, \quad (\dag) \]
  where the last equality follows because $d_{ij}$ does not depend on $k$. Now we note that the maximum of $|d_{ik}-d_{jk}|$ over $k$ must exceed the value of the same expression when either of the terms $d_{ik}$ or $d_{jk}$ is zero, i.e.~when $k\in\{i,j\}$, since, when $k=i$, then $|d_{ik}-d_{jk}|=|d_{ii}-d_{ji}|=d_{ij}$, and the same holds when $k=j$. Hence,
  \[\max\limits_{k\le n}|d_{ik}-d_{jk}|\ge d_{ij}.\quad(\ddag)\]
By ($\ast$), ($\dag$) and ($\ddag$), we finally have:
\[\forall i<j\le n\quad \|\rho(x_i)-\rho(x_j)\|_\infty=d_{ij}\]
as claimed. 
\end{proof}

We remark that Thm.~\ref{thm:uie} is only applicable when $\mathsf{D}(X)$ is a distance matrix, which corresponds to the case of a graph $G$ edge-weighed by $d$ being a complete graph. We address the more general case of any (connected) simple graph $G=(V,E)$, corresponding to a partially defined distance matrix, by completing the matrix using the shortest path metric (this distance matrix completion method was used for the Isomap heuristic, see \cite{tenenbaum_00,isomapdg} and Sect.~\ref{s:isomap}):
\begin{equation}
  \forall \{i,j\}\not\in E \quad d_{ij} = \mathsf{shortest\_path\_length}_G(i,j). \label{shpathmetric}
\end{equation}
In practice, we can compute the lengths of all shortest paths in $G$ by using the Floyd-Warshall algorithm, which runs in $O(n^3)$ time (but in practice it is very fast). 

This method yields a realization of $G$ in $\ell_\infty^n$, which is a high-dimensional embedding. It is necessary to post-process it using dimensional reduction techniques (see Sect.~\ref{s:dimred}). 

\subsubsection{Multidimensional scaling}
\label{s:mds}
The literature on Multidimensional Scaling (MDS) is extensive \cite{coxcox,borg_10}, and many variants exist. The basic version, called {\it classic MDS}, aims at finding an approximate realization of a partial distance matrix. In other words, it is a heuristic solution method for the
\begin{quote}
  {\sc Euclidean Distance Matrix Completion Problem} (EDMCP). Given a simple undirected graph $G=(V,E)$ with an edge weight function $w:E\to\mathbb{R}_+$, determine whether there exists an integer $K>0$ and a realization $x:V\to\mathbb{R}^K$ such that Eq.~\eqref{eq:dgp2} holds.
\end{quote}

The difference between EDMCP and DGP may appear diminutive, but it is in fact very important. In the DGP the integer $K$ is part of the input, whereas in the EDMCP it is part of the output. This has a large effect on worst-case complexity: while the DGP is {\bf NP}-hard even when only an $\varepsilon$-approximate realization is sought \cite[\S 5]{saxe79}, $\varepsilon$-approximate realizations of EDMCPs can be found in polynomial time by solving an SDP \cite{wolkowicz}. Consider the following matrix:
\begin{equation*}
  \Delta(E,d) = \left\{\begin{array}{ll}
  w_{ij}^2 & \mbox{ if } \{i,j\}\in E \\
  d_{ij} & \mbox{ otherwise,}
  \end{array}\right.
\end{equation*}
where $d=(d_{ij}\;|\;\{i,j\}\not\in E)$ is a vector of decision variables, and $J=I_n-\frac{1}{n}\mathbf{1}\transpose{\mathbf{1}}$, with $\mathbf{1}$ being the all-one vector. Then the following formulation is valid for the EDMCP:
\begin{equation}
  \left. \begin{array}{rll}
    \min\limits_{d,T,G} && \mathbf{1}\bullet T  \\
     -T &\le& G  + \frac{1}{2} J\, \Delta(E,d)\, J \le T \\
     G &\succeq& 0,
  \end{array}\right\} \label{eq:sdpedmcp}
\end{equation}
where $\mathbf{1}$ is the $n\times n$ all-one matrix.

\begin{theorem}
  The SDP in Eq.~\eqref{eq:sdpedmcp} correctly models the EDMCP.
  \label{thm:sdpedmcp}
\end{theorem}
By ``correctly models'' we mean that the solution of the EDMCP can be obtained in polynomial time from the solution of the SDP in Eq.~\eqref{eq:sdpedmcp}.
\begin{proof}
  First, we remark that, given a realization $x:V\to\mathbb{R}^n$, its Gram matrix is $G=x\transpose{x}$, and its squared Euclidean distance matrix (EDM) is
  \[D^2=(\|x_u-x_v\|_2^2\;|\;u\le n\land v\le n)\in\mathbb{R}^{n\times n}.\]
  Next, we recall that
  \begin{equation}
    G = I_n - \frac{1}{2} J D^2 J \label{eq:gramdist}
  \end{equation}
  by \cite{dattorro} (after \cite{schoenberg} --- see \cite[\S 7]{six} for a direct proof). Now we note that minimizing $\mathbf{1}\bullet T$ subject to $-T \le G + \frac{1}{2} J \Delta(E,d) J \le T$ is an exact reformulation of
  \[\min_{G,d} \|G - (-1/2) J \Delta(E,d) J)\|_1, \quad (\ast)\]
  since $\mathbf{1}\bullet T =\sum_{i,j} T_{ij}$, and $T$ is used to ``sandwich'' the argument of the $\ell_1$ norm in ($\ast$).
  
  We also recall another basic fact of linear algebra: a matrix is Gram if and only if it is psd: hence, requiring $G\succeq 0$ forces $G$ to be a Gram matrix. Consequently, if the optimal objective function value of Eq.~\eqref{eq:sdpedmcp} is zero with corresponding solution $d^\ast,T^\ast,G^\ast$, then $\trace{T^\ast}=0\Rightarrow T^\ast=0\Rightarrow (\ast)$. Moreover, $G^\ast$ is a Gram matrix, so $\Delta(E,d^\ast)$ is its corresponding EDM. Lastly, the realization $x^\ast$ corresponding to the Gram matrix $G^\ast$ can be obtained by spectral decomposition of $G^\ast=P\Lambda\transpose{P}$, which yields $x^\ast=P\sqrt{\Lambda}$: this implies that the EDMCP instance is YES. Otherwise $T^\ast\not=0$, which means that the EDMCP instance is NO (otherwise there would be a contradiction on $\trace{T^\ast}>0$ being optimal). 
\end{proof}
The practically useful corollary to Thm.~\eqref{thm:sdpedmcp} is that solving Eq.~\eqref{eq:sdpedmcp} provides an approximate solution $x^\ast$ even if $\Delta(E,d)$ cannot be completed to an EDM.

Classic MDS is an efficient heuristic method for finding an approximate realization of a partial distance matrix $\Delta(E,d)$. It works as follows:
\begin{enumerate}
\item complete $\Delta(E,d)$ to an approximate EDM $\tilde{D}^2$ using the shortest-path metric (Eq.~\eqref{shpathmetric}); \label{mds1}
\item let $\tilde{G}=I_n - \frac{1}{n} J\tilde{D}^2J$;\label{mds2}
\item let $P\tilde{\Lambda}\transpose{P}$ be the spectral decomposition of $\tilde{G}$;\label{mds3}
\item if $\tilde{\Lambda}\ge 0$ then, by Eq.~\eqref{eq:gramdist}, $\tilde{D}^2$ is a EDM, with corresponding (exact) realization $\tilde{x}=P\sqrt{\Lambda}$;\label{mds4}
\item otherwise, let $\Lambda^+=\diag{(\max(\lambda,0)\;|\;\lambda\in\Lambda)}$: then $\tilde{x}=P\sqrt{\Lambda^+}$ is an approximate realization of $\tilde{D}^2$.\label{mds5}
\end{enumerate}
Note that both Eq.\eqref{eq:sdpedmcp} and classic MDS determine $K$ as part of the output, i.e.~$K$ is the rank of the realization (respectively $x^\ast$ and $\tilde{x}$).

\section{Dimensional reduction techniques}
\label{s:dimred}
Dimensional reduction techniques reduce the dimensionality of a set of vectors according to different criteria, which may be heuristic, or give some (possibly probabilistic) guarantee of keeping some quantity approximately invariant. They are necessary in order to make many of the methods in Sect.~\ref{s:dgpsol} useful in practice.

\subsection{Principal component analysis}
\label{s:pca}
Principal Component Analysis (PCA) is one of the foremost dimensional reduction techniques. It is ascribed to Harold Hotelling\footnote{A young and unknown George Dantzig had just finished his presentation of LP to an audience of ``big shots'', including Koopmans and Von Neumann. Harold Hotelling raised his hand, and stated: ``but we all know that the world is nonlinear!'', thereby obliterating the simplex method as a mathematical curiosity. Luckily, Von Neumann answered on Dantzig’s behalf and in his defence \cite{dantzig_remin}.} \cite{hotelling}. 

Consider an $n\times m$ matrix $X$ consisting of $n$ data row vectors in $\mathbb{R}^m$, and let $K<m$ be a given integer. We want to find a change of coordinates for $X$ such that the first component has largest variance over the transformed vectors, the second component has second-largest variance, and so on, until the $K$-th component. The other components can be neglected, as the variance of the data in those directions is low.

The usual geometric interpretation of PCA is to take the smallest enclosing ellipsoid $\mathcal{E}$ for $X$: then the required coordinate change maps component $1$ to the line parallel to the largest radius of $\mathcal{E}$, component $2$ to the line parallel to the second-largest radius of $\mathcal{E}$, and so on until component $K$ (see Fig.~\ref{f:pca}).
\begin{figure}[!ht]
  \begin{center}
    \includegraphics[width=7cm]{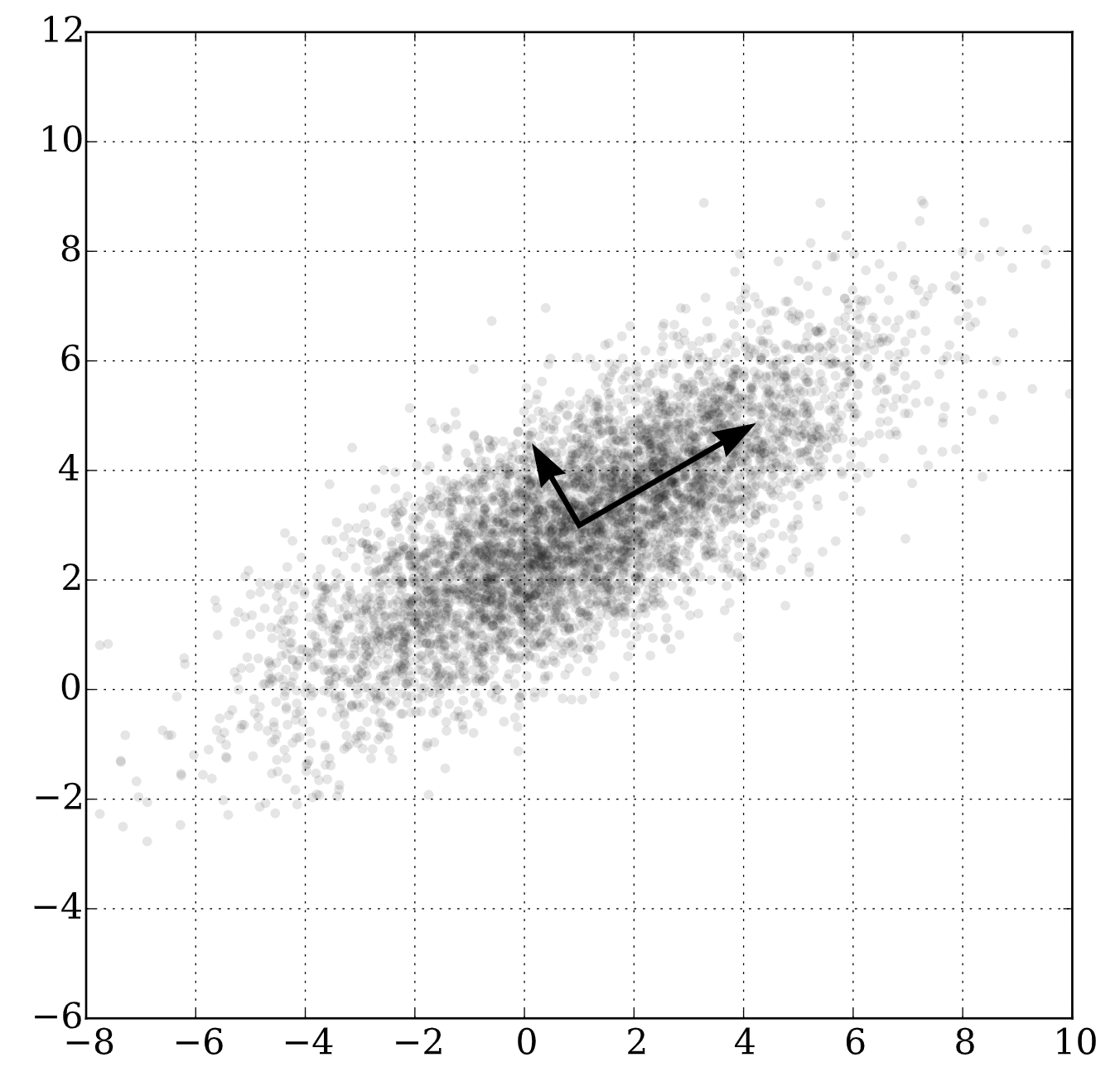}
  \end{center}
  \caption{Geometric interpretation of PCA (image from \cite{wikipedia_pca}).}
  \label{f:pca}
\end{figure}
The statistical interpretation of PCA looks for the change of coordinates which makes the data vectors be uncorrelated in their components. Fig.~\ref{f:pca} should give an intuitive idea about why this interpretation corresponds with the ellipsoid of the geometric interpretation. The cartesian coordinates in Fig.~\ref{f:pca} are certainly correlated, while the rotated coordinates look far less correlated. The zero correlation situation corresponds to a perfect ellipsoid. An ellipsoid is described by the equation $\sum_{j\le n} \big(\frac{x_j}{r_j}\big)^2 = 1$, which has no mixed terms $x_ix_j$ contributing to correlation. Both interpretations are well (and formally) argued in \cite[\S 2.1]{gpca}. 

The interpretation we give here is motivated by DG, and related to MDS (Sect.~\ref{s:mds}). PCA can be seen as a modification of MDS which only takes into account the $K$ (nonnegative) principal components. Instead of $\Lambda^+$ (step \ref{mds5} of the MDS algorithm), PCA uses a different diagonal matrix $\Lambda^{\mathsf{pca}}$: the $i$-th diagonal component is
\begin{equation}
  \Lambda^{\mathsf{pca}}_{ii}=\left\{\begin{array}{ll} \max(\Lambda_{ii},0) & \mbox{ if } i\le K \\ 0 & \mbox{ otherwise,}\end{array}\right.
  \label{eq:pca}
\end{equation}
where $P\Lambda\transpose{P}$ is the spectral decomposition of $\tilde{G}$. In this interpretation, when given a partial distance matrix and the integer $K$ as input, PCA can be used as an approximate solution method for the DGP. 

On the other hand, the PCA algorithm is most usually considered as a method for dimensionality reduction, so it has a data matrix $X$ and an integer $K$ as input. It is as follows:
\begin{enumerate}
\item let $\tilde{G} = X\transpose{X}$ be the $n\times n$ Gram matrix of the data matrix $X$;
\item let $P\tilde{\Lambda}\transpose{P}$ be the spectral decomposition of $\tilde{G}$;
\item return $\tilde{x}=P\sqrt{\Lambda^{\mathsf{pca}}}$.
\end{enumerate}
Then $\tilde{x}$ is an $n\times K$ matrix, where $K<n$. The $i$-th row vector in $\tilde{x}$ is a dimensionally reduced representation of the $i$-th row vector in $X$.


There is an extensive literature on PCA, ranging over many research papers, dedicated monographs and textbooks \cite{wikipedia_pca,jolliffe_10,gpca}. Among the variants and extensions, see \cite{cca,saerens_pca,bach_pca,allen_pca,dey_pca}.

\subsubsection{Isomap}
\label{s:isomap}
One of the most interesting applications of PCA is possibly the Isomap algorithm \cite{tenenbaum_00}, already mentioned above in Sect.~\ref{s:uie}, which is able to use PCA in order to perform a nonlinear dimensional reduction from the original dimension $m$ to a given target dimension $K$, as follows.
\begin{enumerate}
\item Form a connected graph $H=(V,E)$ with the column indices $1,\ldots,n$ of $X$ as vertex set $V$: determine a threshold value $\tau$ such that, for each column vector $x_i$ in $X$ (for $i\le n$), and for each $x_j$ in $X$ such that $\|x_i-x_j\|_2\le\tau$, the edge $\{i,j\}$ is in the edge set $E$; the graph $H$ should be as sparse as possible but also connected.
\item Complete $H$ using the shortest path metric (Eq.~\eqref{shpathmetric}).
\item Use PCA in the MDS interpretation mentioned above: interpret the completion of $(V,E)$ as a metric space, construct its (approximate) EDM $\tilde{D}$, compute the corresponding (approximate) Gram matrix $\tilde{G}$, compute the spectral decomposition of $\tilde{G}$, replace its diagonal eigenvalue matrix $\Lambda$ as in Eq.~\eqref{eq:pca}, and return the corresponding $K$-dimensional vectors.
\end{enumerate}
Intuitively, Isomap works well because in many practical situations where a set $X$ of points in $\mathbb{R}^m$ are close to a (lower) $K$-dimensional manifold, the shortest path metric is likely to be a better estimation of the Euclidean distance in $\mathbb{R}^K$ than the Euclidean distance in $\mathbb{R}^m$, see \cite[Fig.~3]{tenenbaum_00}.

\subsection{Barvinok's naive algorithm}
\label{s:barvinok}
By Eq.~\eqref{sdpfeas}, we can solve an SDP relaxation of the DGP and obtain an $n\times n$ psd matrix solution $\bar{X}$ which, in general, will not have rank $K$ (i.e., it will not yield an $n\times K$ realization matrix, but rather an $n\times n$ one). In this section we shall derive a dimensionality reduction algorithm to obtain an approximation of $\bar{X}$ which has the correct rank $K$.

\subsubsection{Quadratic Programming feasibility}
\label{s:qpsys}
Barvinok's naive algorithm \cite[\S 5.3]{barvinok2} is a probabilistic algorithm which can find an approximate vector solution $x'\in\mathbb{R}^n$ to a system of quadratic equations
\begin{equation}
  \forall i\le m \quad \transpose{x} Q^i x = a_i, \label{eq:qpsys}
\end{equation}
where the $Q^i$ are $n\times n$ symmetric matrices, $a\in\mathbb{R}^m$, $x\in\mathbb{R}^n$, and $m$ is polynomial in $n$. The analysis of this algorithm provides a probabilistic bound on the maximum distance that $x'$ can have from the set of solutions of Eq.~\eqref{eq:qpsys}. Thereafter, one can run a local NLP solver with $x'$ as a starting point, and obtain a hopefully good (approximate) solution to Eq.~\eqref{eq:qpsys}. We note that this algorithm is still not immediately applicable to the our setting where $K$ could be different from $1$: we shall address this issue in Sect.~\ref{s:bvknK}.

Barvinok's naive algorithm solves an SDP relaxation of Eq.~\eqref{eq:qpsys}, and then retrieves a certain randomized vector from the solution:
\begin{enumerate}
\item form the SDP relaxation
  \begin{equation}
    \forall i\le m\;(Q^i\bullet X = a_i)\land X\succeq 0\label{eq:qpsdp}
  \end{equation}
  of Eq.~\eqref{eq:qpsys} and solve it to obtain $\bar{X}\in\mathbb{R}^{n\times n}$;
\item let $T=\sqrt{\bar{X}}$, which is a real matrix since $\bar{X}\succeq 0$ ($T$ can be obtained by spectral decomposition, i.e.~$\bar{X}=P\Lambda\transpose{P}$ and $T=P\sqrt{\Lambda}$);
\item let $y$ be a vector sampled from the multivariate normal distribution $\mathsf{N}^n(0,1)$;\label{bvkstep3}
\item compute and return $x'=Ty$.
\end{enumerate}
The analysis provided in \cite{barvinok2} shows that $\exists c>0$ and an integer $n_0\in\mathbb{N}$ such that $\forall n\ge n_0$
\begin{equation}
  \mathsf{P}\left(\forall i\le m\quad \dist{x',\mathcal{X}_{i}}\le c\,\sqrt{\|\bar{X}\|_2\ln n}\right)\ge 0.9. \label{eq:bvk1}
\end{equation}
In Eq.~\eqref{eq:bvk1}, $\mathsf{P}(\cdot)$ denotes the probability of an event,
\[\dist{b,B}=\inf_{\beta\in B} \|b-\beta\|_2\]
is the Euclidean distance between the point $b$ and the set $B$, and $c$ is a constant that only depends on $\log_n m$. We note that the term $\sqrt{\|\bar{X}\|_2}$ in Eq.~\eqref{eq:bvk1} arises from $T$ being a factor of $\bar{X}$. We note also that $0.9$ follows from assigning some arbitrary value to some parameter --- i.e.~$0.9$ can be increased as long as the problem size is large enough.

For cases of Eq.~\eqref{eq:qpsys} where one of the quadratic equations is $\|x\|_2^2=1$ (namely, the solutions of Eq.~\eqref{eq:qpsys} must belong to the unit sphere), it is noted in \cite[Eg.~5.5]{barvinok2} that, if $\bar{X}$ is ``sufficiently generic'', then it can be argued that $\|\bar{X}\|_2=O(1/n)$, which implies that the bounding function $c\sqrt{\bar{X}_2\ln n}\to 0$ as $n\to\infty$. This, in turn, means that $x'$ converges towards a feasible solution of the original problem in the limit. 

\subsubsection{Concentration of measure}
\label{s:concmeas}
The term $\ln n$ in Eq.~\eqref{eq:bvk1} arises from a phenomenon of high-dimensional geometry called ``concentration of measure''. 

We recall that a function $f:\mathcal{X}\to\mathbb{R}$ is {\it Lipschitz} if there is a constant $M>0$ s.t.~for any $x,y\in \mathcal{X}$ we have $|f(x)-f(y)|<M\|x-y\|_2$. A measure space $(\mathcal{X},\mu)$ has the {\it concentration of measure} property if for any Lipschitz function $f$, there are constants $C,c>0$ such that:
\begin{equation}
  \forall \varepsilon>0 \quad \mathsf{P}(|f(x)-\mathsf{E}_\mu(f)|>\varepsilon\;|\;x\in \mathcal{X})\le C\, e^{-c\varepsilon^2} \label{eq:concmeas}
\end{equation}
where $\mathsf{E}_\mu(f)=\int_{\mathcal{X}} f(x)d\mu$. In other words, $\mathcal{X}$ has measure concentration if for any Lipschitz function $f$, its discrepancy from its mean value is small with arbitrarily high probability. It turns out that the Euclidean space $\mathbb{R}^n$ with the Gaussian density measure $\phi(x)=(2\pi)^{n/2} e^{-\|x\|_2^2/2}$ has measure concentration \cite[\S 5.3]{barvinokcc}.

Measure concentration is interesting in view of applications since, given any large enough closed subset $A$ of $\mathcal{X}$, its $\varepsilon$-neighbourhood
\begin{equation}
  A(\varepsilon)=\{x\in \mathcal{X} \;|\; \dist{x,A}\le\varepsilon\} \label{eq:bvk2}
\end{equation}
contains almost the whole measure of $\mathcal{X}$. More precisely, if $(\mathcal{X},\mu)$ has measure concentration and $A\subset \mathcal{X}$ is closed, for any $p\in(0,1)$ there is a $\varepsilon_0(p)>0$ such that \cite[Prop.~2]{barvinok_orl}:
\begin{equation}
  \forall \varepsilon\ge \varepsilon_0(p)\quad \mu(A(\varepsilon)) > 1-p.\label{eq:bvk3}
\end{equation}
Eq.~\eqref{eq:bvk3} is useful for applications because it defines a way to analyse probabilistic algorithms. For a random point sampled in $(\mathcal{X},\mu)$ that happens to be in $A$ on average, Eq.~\eqref{eq:bvk3} ensures that it is unlikely that it should be far from $A$. This can be used to bound errors, as Barvinok did with his naive algorithm. Concentration of measure is fundamental in data science, insofar as it may provide algorithmic analyses to the effect that some approximation errors decrease in function of the increasing instance size. 

\subsubsection{Analysis of Barvinok's algorithm}
We sketch the main lines of the analysis of Barvinok's algorithm (see \cite[Thm.~5.4]{barvinok} or \cite[\S 3.2]{barvinok_orl} for a more detailed proof). We let $\mathcal{X}=\mathbb{R}^n$ and $\mu(x)=\phi(x)$ be the Gaussian density measure. It is easy to show that
\[\mathsf{E}_\mu(\transpose{x}Q^ix\;|\;x\in \mathcal{X})=\trace{Q^i}\]
for each $i\le m$, and, from this, that given the factorization $\bar{X}=T\transpose{T}$, that
\[\mathsf{E}_\mu(\transpose{x}\transpose{T} Q^i\, Tx\;|\;x\in \mathcal{X})=\trace{\transpose{T}Q^i\, T}=\trace{Q^i\bar{X}}=Q^i\bullet \bar{X}=a_i.\]
This shows that, for any $y\sim\mathsf{N}^n(0,1)$, the average of $\transpose{y}\transpose{T} Q^i\, Ty$ is $a_i$.

The analysis then goes on to show that, for some $y\sim\mathsf{N}^n(0,1)$, it is unlikely that $\transpose{y}\transpose{T} Q^i\, Ty$ should be far from $a_i$. It achieves this result by defining the sets $A_i^+ = \{x\in\mathbb{R}^n \;|\; \transpose{x}Q^i x \ge a_i\}$, $A_i^- = \{x\in\mathbb{R}^n \;|\; \transpose{x}Q^i x \le a_i\}$, and their respective neighbourhoods $A_i^+(\varepsilon)$, $A_i^-(\varepsilon)$. Using a technical lemma \cite[Lemma 4]{barvinok_orl} it is possible to apply Eq.~\eqref{eq:bvk3} to $A_i^+(\varepsilon)$ and $A_i^-(\varepsilon)$ to argue for concentration of measure. Applying the union bound it can be shown that their intersection $A_i(\varepsilon)$ is the neighbourhood of $A_i = \{x\in\mathbb{R}^n \;|\; \transpose{x}Q^ix = a_i\}$. Another application of the union bound to all the sets $A_i(\varepsilon)$ yields the result \cite[Thm.~5]{barvinok_orl}.

We note that concentration of measure proofs often have this structure: (a) prove that a certain event holds on average; (b) prove that the discrepancy from average gets smaller and/or more unlikely with increasing size. Usually proving (a) is easier than proving (b). 

\subsubsection{Applicability to the DGP}
\label{s:bvknK}
The issue with trying to apply Barvinok's naive algorithm to the DGP is that we should always assume $K=1$ by Eq.~\eqref{eq:qpsys}. To circumvent this issue, we might represent an $n\times K$ realization matrix as a vector in $\mathbb{R}^{nK}$ by stacking its columns (or concatenating its rows). This, on the other hand, would require solving SDPs with $nK\times nK$ matrices, which is prohibitive because of size.

Luckily, Barvinok's naive algorithm can be very easily extended to arbitrary values of $K$. We replace Step \ref{bvkstep3} by:
\begin{itemize}
  \item[3{\it b}.] let $y$ be an $n\times K$ matrix sampled from $\mathsf{N}^{n\times K}(0,1)$.\label{bvkstep3b}
\end{itemize}
The corresponding analysis needs some technical changes \cite{barvinok_orl}, but the overall structure is the same as the case $K=1$. The obtained bound replaces $\sqrt{\ln n}$ in Eq.~\eqref{eq:bvk1} with $\sqrt{\ln nK}$.

In the DGP case, the special structure of the matrices $Q^i$ (for $i$ ranging over the edge set $E$) makes it possible to remove the factor $K$, so we retrieve the exact bound of Eq.~\eqref{eq:bvk1}. As noted in Sect.~\ref{s:qpsys}, if the DGP instance is on a sphere \cite{dgpsphere}, this means that $x'=Ty$ converges to an exact realization with probability 1 in the limit of $n\to\infty$. Similar bounds to Eq.~\eqref{eq:bvk1} were also derived for the iDGP case \cite{barvinok_orl}.

Barvinok also described concentration of measure based techniques for finding low-ranking solutions solutions of the SDP in Eq.~\eqref{eq:qpsdp} (see \cite{barvinok} and \cite[\S 6.2]{barvinokcc}), but these do not allow the user to specify an arbitrary rank $K$, so they only apply to the EDMCP.

\subsection{Random projections}
\label{s:rp}
{\it Random projections} are another dimensionality reduction technique exploiting high-dimensional geometry properties and, in particular, the concentration of measure phenomenon (Sect.~\ref{s:concmeas}). They are more general than Barvinok's naive algorithm (Sect.~\ref{s:barvinok}) in that they apply to sets of vectors in some high-dimensional Euclidean space $\mathbb{R}^n$ (with $n\gg 1$). These sets are usually finite and growing polynomially with instance sizes \cite{vempala}, but they may also be infinite \cite{woodruff}, in which case the technical name used is {\it subspace embeddings}.


\subsubsection{The Johnson-Lindenstrauss Lemma}
\label{s:jll}
The foremost result in RPs is the celebrated Johnson-Lindenstrauss Lemma (JLL) \cite{jllemma}. For a set of vectors $\mathcal{X}\subset\mathbb{R}^n$ with $|\mathcal{X}|=\ell$, and an $\varepsilon\in(0,1)$ there is a $k=O(\frac{1}{\varepsilon^2}\ln\ell)$ and a mapping $f:\mathcal{X}\to\mathbb{R}^k$ such that:
\begin{equation}
  \forall x,y\in\mathcal{X} \quad (1-\varepsilon)\|x-y\|_2\le\|f(x)-f(y)\|_2\le (1+\varepsilon)\|x-y\|_2. \label{jll}
\end{equation}
The proof of this result \cite[Lemma 1]{jllemma} is probabilistic, and show that an $f$ satisfying Eq.~\eqref{jll} exists with some nonzero probability.

Later and more modern proofs (e.g.~\cite{dasgupta}) clearly point out that $f$ can be a linear operator represented by a $k\times n$ matrix $T$, each component of which can be sampled from a {\it subgaussian distribution}. This term refers to a random variable $\mathfrak{V}$ for which there are constants $C,c$ s.t.~for each $t>0$ we have
\[\mathsf{P}(|\mathfrak{V}|>t)\le C\,e^{-ct^2}.\]
In particular, the Gaussian distribution is also subgaussian. Then the probability that a randomly sampled $T$ satisfies Eq.~\eqref{jll} can be shown to exceed $1/\ell$. The union bound then provides an estimate on the number of samplings of $T$ necessary to guarantee Eq.~\eqref{jll} with a desired probability.

Some remarks are in order.
\begin{enumerate}
\item Computationally Eq.~\eqref{jll} is applied to some given data as follows: given a set $\mathcal{X}$ of $\ell$ vectors in $\mathbb{R}^n$ and some error tolerance $\varepsilon\in(0,1)$, find an appropriate $k=O(\frac{1}{\varepsilon^2}\ln\ell)$, construct the $k\times n$ RP $T$ by sampling each of its components from $\mathsf{N}(0,\frac{1}{\sqrt{k}})$, then define the set $T\mathcal{X} = \{Tx\;|\;x\in X\}$. By the JLL, $T\mathcal{X}$ is {\it approximately congruent} to $\mathcal{X}$ in the sense of Eq.~\eqref{jll}; however, $T\mathcal{X}\subset\mathbb{R}^k$ whereas $\mathcal{X}\subset\mathbb{R}^n$, and, typically, $k\ll n$.
\item The computation of an appropriate $k$ would appear to require an estimation of the constant in the expression $O(\frac{1}{\varepsilon^2}\ln\ell)$. Values computed theoretically are often so large as to make the technique useless in practice. As far as we know, this constant has only been done empirically in some cases \cite{venkatasubramanian}, ending up with an estimation of the constant at $1.8$ (which is the value we employed in most of our experiments). 
\item The term $\frac{1}{\sqrt{k}}$ is the standard deviation of the normal distribution from which the components of $T$ must be sampled. It corresponds to a scaling of the vectors in $T\mathcal{X}$ induced by the loss in dimensions (see Thm.~\ref{thm:rpnorm}).
\item In the expression $O(\frac{1}{\varepsilon^2}\ln\ell)$, the logarithmic term is the one that counts for analysis purposes, but in practice $\varepsilon^{-2}$ can be large. Our advice is to take $\varepsilon\in(0.1,0.2)$ and then fine-tune $\varepsilon$ according to results.
\item Surprisingly, the target dimension $k$ is independent of the original dimension $n$.
\item Even if the data in $\mathcal{X}$ is sparse, $T\mathcal{X}$ ends up being dense. Different classes of sparse RPs have been investigated \cite{achlioptas,kane} in order to tackle this issue. A simple algorithm \cite[\S 5.1]{rpqptr} consists in initializing $T$ as the $k\times n$ zero matrix, and then only fill components using samples from $\mathsf{N}(0,\frac{1}{\sqrt{kp}})$ with some given probability $p$. The value of $p$ corresponds to the density of $T$. In general, and empirically, it appears that the larger $n$ and $\ell$ are, the sparser $T$ can be.
\item Obviously, a Euclidean space of dimension $k$ can embed at most $k$ orthogonal vectors. An easy, but surprising corollary of the JLL is that as many as $O(2^k)$ approximately orthogonal vectors can fit in $\mathbb{R}^k$. This follows by \cite[Prop.~1]{jllmor} applied to the standard basis $S=\{e_1,\ldots,e_n\}$ of $\mathbb{R}^n$: we obtain $\forall i<j\le n\;(-\varepsilon\le \langle Te_i,Te_j\rangle-e_ie_j\le\varepsilon)$, which implies $|\langle Te_i,Te_j\rangle|\le\varepsilon$ with $TS\subset\mathbb{R}^k$ and $k=O(\ln n)$. Therefore $TS$ is a set of $O(2^k)$ almost orthogonal vectors in $\mathbb{R}^k$, as claimed. 
\item Typical applications of RPs arise in clustering databases of large files (e.g.~e-mails, images, songs, videos), performing basic tasks in ML (e.g.~k-means \cite{boutsidis2010}, k-nearest neighbors (k-NN) \cite{indyk}, robust learning \cite{vempala06} and more \cite{indyk4}), and approximating large MP formulations (e.g.~LP, QP, see Sect.~\ref{s:jllmp}).
\item The JLL seems to suggest that most of the information encoded by the congruence of a set of vectors can be maintained up to an $\varepsilon$ tolerance in much smaller dimensional spaces. This is not true for sets of vectors in low dimensions. For example, with $n\in\{2,3\}$ a few attempts immediately show that RPs yield sets of projected vectors which are necessarily incongruent with the original vectors. 
\end{enumerate}

In this paper, we do not give a complete proof of the JLL, since many different ones have already been provided in research articles \cite{jllemma,dasgupta,indyk3,chazelle,kane,matousek-jll,micali} and textbooks \cite{vempala,matousekmetric,mathplusplus,vershynin}. We only prove the first part of the proof, namely the easy result that RPs preserve norms on average. This provides an explanation for the variance $1/k$ of the distribution from which the components of $T$ are sampled.
\begin{theorem}
  Let $T$ be a $k\times n$ RP sampled from $\mathsf{N}(0,\frac{1}{\sqrt{k}})$, and $u\in\mathbb{R}^n$; then $\mathsf{E}(\|Tu\|_2^2)=\|u\|_2^2$.
  \label{thm:rpnorm}
\end{theorem}
\begin{proof}
  We prove the claim for $\|u\|_2=1$; the result will follow by scaling. For each $i\le k$ we define $v_i=\sum_{j\le n}T_{ij} u_j$. Then $\mathsf{E}(v_i)=\mathsf{E}\big(\sum_{j\le m} T_{ij}u_j\big)=\sum_{j\le m}\mathsf{E}(T_{ij})u_j=0$. Moreover,
  \[\mathsf{Var}(v_i)=\sum\limits_{j\le m}\mathsf{Var}(T_{ij}u_j)= \sum\limits_{j\le m}\mathsf{Var}(T_{ij})u_j^2=\sum\limits_{j\le m} \frac{u_j^2}{{k}} = \frac{1}{{k}}\|u\|^2=\frac{1}{{k}}.\]
  Now, $\frac{1}{{k}}=\mathsf{Var}(v_i)=\mathsf{E}(v_i^2-(\mathsf{E}(v_i))^2) = \mathsf{E}(v_i^2-0) = \mathsf{E}(v_i^2)$. Hence
  \[\mathsf{E}(\|Tu\|^2)=\mathsf{E}(\|v\|^2)=\mathsf{E}\big(\sum\limits_{i\le{k}}v_i^2\big)=\sum\limits_{i\le{k}}\mathsf{E}(v_i^2) =\sum\limits_{i\le{k}} \frac{1}{{k}} = 1,\]
  as claimed. 
\end{proof}

\subsubsection{Approximating the identity}
\label{s:rpidapprox}
If $T$ is a $k\times n$ RP where $k=O(\varepsilon^{-2}\ln n)$, both $T\transpose{T}$ and $\transpose{T}T$ have some relation with the identity matrices $I_k$ and $I_n$. This is a lesser known phenomenon, so it is worth discussing it here in some detail.

We look at $T\transpose{T}$ first. By \cite[Cor.~7]{zhang2013} for any $\epsilon\in(0,\frac{1}{2})$ we have
\[\|\frac{1}{n}\,T\,\transpose{T} - I_k\|_2\le\varepsilon\]
with probability at least $1-\delta$ as long as $n\ge \frac{(k+1)\ln (2k/\delta)}{\mathscr{C}\varepsilon^2}$, where $\mathcal{C}\ge\frac{1}{4}$ is a constant.

In Table \ref{t:PPTI} we give values of $\|s\,T\transpose{T} - I_d\|_2$ for $s\in\{1/n,1/d,1\}$, $n\in\{1000,2000,\ldots,10000\}$ and $d=\lceil \ln(n) / \epsilon^2\rceil$ where $\epsilon=0.15$.
\begin{table}[!ht]
  \begin{center}
    \begin{tabular}{|r|rrrrrrrrrr|} \hline
      & \multicolumn{10}{c|}{$n$} \\ 
      $s$ & 1e3 & 2e3 & 3e3 & 4e3 & 5e3 & e3 & 7e3 & 8e3 & 9e3 & 1e4 \\ \hline
      $1/n$ & 9.72 & 7.53 & 6.55 & 5.85 & 5.36 & 5.01 & 4.71 & 4.44 & 4.26 & 4.09 \\
      $1/d$ & 5e1 & 1e2 & 1.5e2 & 2e2 & 2.5e2 & 3e2 & 3.5e2 & 3.9e2 & 4.4e2 & 4.8e2 \\
      $1$ & 2e5 & 4e5 & 6e5 & 8e5 & 1e6 & 1.2e6 & 1.4e6 & 1.6e6 & 1.8e6 & 2e6 \\ \hline
    \end{tabular}
  \end{center}
  \caption{Values of $\|sT\transpose{T}- I_d\|$ in function of $s,n$.\label{t:PPTI}}
\end{table}
  It is clear that the error decreases as the size increases only in the case $s=\frac{1}{n}$. This seems to indicate that the scaling is a key parameter in approximating the identity.

Let us now consider the product $\transpose{T}T$. It turns out that, for each fixed vector $x$ not depending on $T$, the matrix $\transpose{T} T$ behaves like the identity w.r.t.~$x$.
\begin{theorem}
  Given any fixed $x\in\mathbb{R}^n$, $\epsilon\in(0,1)$ and a RP $P\in\mathbb{R}^{d\times n}$, there is a universal constant $\mathcal{C}$ such that
  \begin{equation}
    -\mathbf{1}\varepsilon \le \transpose{T}Tx - x \le \mathbf{1}\varepsilon.\label{eq:PTPI}
  \end{equation}
  with probability at least $1-4e^{\mathcal{C}\epsilon^2d}$.
  \label{thm:PPI}
\end{theorem}
\begin{proof}
By definition, for each $i\le n$ we have $x_i = \langle e_i,x\rangle$, where $e_i$ is the $i$-th unit coordinate vector. By elementary linear algebra we have $\langle e_i, \transpose{T} Tx \rangle = \langle Te_i, Tx\rangle$. By \cite[Lemma 3.1]{rpqptr}, for $i\le n$ we have
\[ \langle e_i, x\rangle-\epsilon\|x\|_2\le \langle T e_i, Tx\rangle\le \langle e_i,x\rangle + \epsilon\|x\|\]
with arbitrarily high probability, which implies the result. 
\end{proof}

One might be tempted to infer from Thm.~\ref{thm:PPI} that $\transpose{T}T$ ``behaves like the identity matrix'' (independently of $x$). This is generally false: Thm.~\ref{thm:PPI} only holds for a given (fixed) $x$.

In fact, since $T$ is a $k\times n$ matrix with $k<n$, $\transpose{T} T$ is a square symmetric psd $n\times n$ matrix with rank $k$, hence $n-k$ of its eigenvalues are zero --- and the nonzero eigenvalues need not have value one. On the other hand, $\transpose{T}T$ looks very much like a slightly perturbed identity, on average, as shown in Table \ref{PPItab1}.
\begin{table}[!ht]
  \begin{center}
    {\footnotesize
    \begin{tabular}{|r||r|r|} \hline
      $n$ & diagonal & off-diag \\ \hline
       500 & 1.00085 & 0.00014 \\
      1000 & 1.00069 & 0.00008 \\
      1500 & 0.99991 &-0.00006 \\
      2000 & 1.00194 & 0.00005 \\
      2500 & 0.99920 &-0.00004 \\
      3000 & 0.99986 &-0.00000 \\
      3500 & 1.00044 & 0.00000 \\
      4000 & 0.99693 & 0.00000 \\ \hline
    \end{tabular}
    }
  \end{center}
  \caption{Average values of diagonal and off-diagonal components of $\transpose{T}T$ in function of $n$, where $T$ is a $k\times n$ RP with $k=O(\varepsilon^{-2}\ln n)$ and $\epsilon=0.15$.\label{PPItab1}}
\end{table}

\subsubsection{Using RPs in MP}
\label{s:jllmp}
Random projections have mostly been applied to probabilistic approximation algorithms. By randomly projecting their (vector) input, one can execute algorithms with lower-dimensional vector more efficiently. The approximation guarantee is usually derived from the JLL or similar results.

A line of research about applying RPs to MP fromulations has been started in \cite{jll-dam,jllmor,ipco19,rpqptr}. Whichever algorithm one may choose in order to solve the MP, the RP properties guarantee an approximation on optimality and/or feasibility. Thus, this approach leads to stronger/more robust results with respect to applying RPs to algorithmic input.

Linear and integer feasibility problems (i.e.~LP and MILP formulations without objective function) are investigated in \cite{jll-dam} from a purely theoretical points of view. The effect of RPs on LPs (with nonzero objective) are investigated in \cite{jllmor}, both theoretically and computationally. Specifically, the randomly projected LP formulation is shown to have bounded feasibility error and an approximation guarantee on optimality. The computational results suggest that the range of practical application of this technique starts with relatively small LPs (thousands of variables/constraints). In both \cite{jll-dam,jllmor} we start from a (MI)LP in standard form
\[\mathcal{P}\equiv \min\{\transpose{c}x\;|\;Ax=b\land x\ge 0\land x\in X\}\]
(where $X=\mathbb{R}^n$ or $\mathbb{Z}^n$ respectively), and obtain a randomly projected formulation under the RP $T\sim\mathsf{N}^{n\times k}(0,\frac{1}{\sqrt{k}})$ with the form 
\[T\mathcal{P}\equiv \min\{\transpose{c}x\;|\;TAx=Tb\land x\ge 0\land x\in X\},\]
i.e.~$T$ reduces the number of constraints in $\mathcal{P}$ to $O(\ln n)$, which can therefore be solved more efficiently.

The RP technique in \cite{ipco19,rpqptr} is different, insofar as it targets the number of variables. In \cite{rpqptr} we consider a QP of the form:
\[\mathcal{Q}\equiv\max\{\transpose{x}Qx +\transpose{c}x \;|\;Ax\le b\},\]
where $Q$ is $n\times n$, $c\in\mathbb{R}^n$, $A$ is $m\times n$, and $b\in\mathbb{R}^m$, $x\in\mathbb{R}^n$. This is projected as follows:
\[T\mathcal{Q}\equiv\max\{\transpose{u}\bar{Q}x +\transpose{\bar{c}}u \;|\;\bar{A}u\le b\},\]
where $\bar{Q}=TQ\transpose{T}$ is $k\times k$, $\bar{A}=A\transpose{T}$ is $m\times k$, $\bar{c}=Tc$ is in $\mathbb{R}^k$, and $u\in\mathbb{R}^k$. 
In \cite{ipco19} we consider a QCQP $\mathcal{Q}'$ like $\mathcal{Q}$ but subject to a ball constraint $\|x\|_2\le 1$. In the projected problem $T\mathcal{Q}'$, this is replaced by a ball constraint $\|u\|_2\le 1$. Both \cite{rpqptr,ipco19} are both theoretical and computational. In both cases, the number of variables of the projected problem is $O(\ln n)$.

In applying RPs to MPs, one solves the smaller projected problems in order to obtain an answer concerning the corresponding original problems. In most cases one has to devise a way to retrieve a solution for the original problem using the solution of the projected problem. This may be easy or difficult depending on the structure of the formulation and the nature of the RP. 

\section{Distance instability}
\label{s:distres}
Most of the models and methods in this survey are based on the concept of distance: usually Euclidean, occasionally with other norms. The k-means algorithm (Sect.~\ref{s:kmeans}) is heavily based on Euclidean distances in Step \ref{op2} (p.~\pageref{op2}), where the reassignment of a point to a cluster is carried out based on proximity: in particular, one way to implement Step \ref{op2} is to solve a 1-nearest neighbor problem. The training of an ANN (Sect.~\ref{s:ann}) repeatedly solves a minimum distance subproblem in Eq.~\eqref{eq:trainprob1}. In spectral clustering (Sect.~\ref{s:spclust}) we have a Euclidean norm constraint in Eq.~\eqref{sumn}. All DGP solution methods (Sect.~\ref{s:dgpsol}), with the exception of incidence vectors (Sect.~\ref{s:incvec}), are concerned with distances by definition. PCA (Sect.~\ref{s:pca}), in its interpretation of a modified MDS, can be seen as another solution method for the DGP. Barvinok's naive algorithm (Sect.~\ref{s:barvinok}) is a dimensional reduction method for SDPs the analysis of which is based on a distance bound; moreover, it was successfully applied to the DGP \cite{barvinok_orl}. The RP-based methods discussed in Sect.~\ref{s:rp} have all been derived from the JLL (Sect.~\ref{s:jll}), which is a statement about the Euclidean distance. We also note that the focus of this survey is on typical DS problems, which are usually high-dimensional.

It is therefore absolutely essential that all of these methods should be able to take robust decisions based on comparing distance values computed on pairs of high-dimensional vectors. It turns out, however, that smallest and largest distances $D_{\mathsf{min}},D_{\mathsf{max}}$ of a random point $Z\in\mathbb{R}^n$ to a set of random points $X_1,\ldots,X_\ell\subset\mathbb{R}^n$ are almost equal (and hence, difficult to compare) as $n\to\infty$ under some reasonable conditions. This holds for any distribution used to sample $Z,X_i$. This result, first presented in \cite{beyer} and subsequently discussed in a number of papers \cite{aggarwal1,aggarwal2,distconc,durrant,radovanovic,mansouri,flexer}, appears to jeopardize all of the material presented in this survey, and much more beyond. The phenomenon leading to the result is known as {\it distance instability} and {\it concentration of distances}. 

\subsection{Statement of the result}
Let us look at the exact statement of the distance instability result.

First, we note that the points $Z,X_1,\ldots,X_\ell$ are not given points in $\mathbb{R}^n$ but rather multivariate random variables with $n$ components, so distance instability is a purely statistical statement rather than a geometric one. We consider
\begin{eqnarray*}
  Z &=& (Z_1,\ldots,Z_n) \\
  \forall i\le\ell \quad X_i &=& (X_{i1},\ldots,X_{in}),
\end{eqnarray*}
where $Z_1,\ldots,Z_n$ are random variables with distribution $\mathcal{D}_1$; $X_{11},\ldots,X_{\ell n}$ are random variables with distribution $\mathcal{D}_2$; and all of these random variables are independently distributed.

Secondly, $D_{\mathsf{min}},D_{\mathsf{max}}$ are functions of random variables:
\begin{eqnarray}
  D_{\mathsf{min}} &=& \min \{\dist{Z,X_i} \;|\;i\le\ell\} \label{eq:Dmin}\\
  D_{\mathsf{max}} &=& \max \{\dist{Z,X_i} \;|\;i\le\ell\}, \label{eq:Dmax}
\end{eqnarray}
and are therefore random variables themselves. In the above, $\mathsf{dist}$ denotes a function mapping pairs of points in $\mathbb{R}^n$ to a non-negative real number, which makes distance instability a very general phenomenon. Specifically, $\mathsf{dist}$ need not be a distance at all.

Third, we now label every symbol with an index $m$, which will be used to compute limits for $m\to\infty$: $Z^m$, $X^m$, $\mathcal{D}_1^m$, $\mathcal{D}_2^m$, $D_{\mathsf{min}}^m$, $D_{\mathsf{max}}^m$, $\mathsf{dist}^m$. We shall see that the proof of the distance instability result is wholly syntactical: its steps are very simple and follow from basic statistical results. In particular, we can see $m$ as an abstract parameter under which we shall take limits, and the proof will hold. Since the proof holds independently of the value of $n$, it also holds if we assume that $m=n$, i.e.~if we give $m$ the interpretation of dimensionality of the Euclidean space embedding the points. While this assumption is not necessary for the proof to hold, it may simplify its understanding: $m=n$ makes the proof somewhat less general, but it gives the above indexing a more concrete meaning. Specifically, $Z,X,\mathcal{D},D,\mathsf{dist}$ are points, distributions, extreme distance values and a distance function in dimension $m$, and the limit $m\to\infty$ is a limit taken on increasing dimension.

Fourth, the ``reasonable conditions'' referred to above for the distance instability result to hold are that there is a constant $p>0$ such that
\begin{equation}
  \exists i\le\ell \quad \lim\limits_{m\to\infty} \mathsf{Var}\left(\frac{(\dist{Z^m,X^m_i})^p}{\mathsf{E}((\dist{Z^m,X^m_i})^p)}\right) = 0.\label{eq:distrescond}
\end{equation}
A few remarks on Eq.~\eqref{eq:distrescond} are in order.
\begin{enumerate}[(a)]
\item The existential quantifier simply encodes the fact that the $X_i$ are all identically distributed, so a statement involving variance and expectation of quantities depending on the $X_i$ random variables holds for all $i\le\ell$ if it holds for just one $X_i$.
\item The constant $p$ simply gives more generality to the result, but plays no role whatsoever in the proof; it can be used in order to simplify computations when $\mathsf{dist}$ is an $\ell_p$ norm.
\item The fraction term in Eq.~\eqref{eq:distrescond} measures a spread relative to an expectation. Requiring that the limit of this relative spread goes to zero for increasing dimensions looks like an asymptotic concentration requirement (hence the alternative name ``distance concentration'' for the distance instability phenomenon). Considering the effect of concentration of measure phenomena in high dimensions (Sect.~\ref{s:concmeas}), distance instability might now appear somewhat less surprising. 
\end{enumerate}

With these premises, we can state the distance instability result.
\begin{theorem}
  If $D_{\mathsf{min}}^m$ and $D_{\mathsf{max}}^m$ are as in Eq.~\eqref{eq:Dmin}-\eqref{eq:Dmax} and satisfy Eq.~\eqref{eq:distrescond}, then, for any $\varepsilon>0$, we have
  \begin{equation}
    \lim\limits_{m\to\infty} \mathsf{P}\left(D^m_{\mathsf{max}}\le(1+\varepsilon)D^m_{\mathsf{min}}\right) = 1.\label{eq:distres}
  \end{equation}
  \label{thm:distres}
\end{theorem}
Thm.~\ref{thm:distres} basically states that closest and farthest neighbors of $Z$ are indistinguishable up to an $\varepsilon$. If the closest and farthest are indistinguishable, trying to discriminate between the closest and the second closest neighbors of a given point might well be hopeless due to floating point errors (note that this discrimination occurs at each iteration of the well known k-means algorithm). This is why distance instability is sometimes cited as a reason for convergence issues in k-means \cite{gayraud}.

\subsection{Related results}
\label{s:distreslitrev}
In \cite{beyer}, several scenarios are analyzed to see where distance instability occurs --- even if some of the requirement of distance instability are relaxed \cite[\S 3.5]{beyer} --- and where it does not \cite[\S 4]{beyer}. Among the cases where distance instability does not apply, we find the case where the data points $X$ are well separated and the case where the dimensionality is implicitly low. Among the cases where it does apply, we find k-NN: in their experiments, the authors of \cite{beyer} find that k-NN becomes unstable already in the range $n\in\{10,20\}$ dimensions. Obviously, the instability of k-NN propagates to any algorithm using k-NN, such as k-means.

Among later studies, \cite{aggarwal1} proposes an alternative definition of $\mathsf{dist}$ where high-dimensional points are projected into lower dimensional spaces. In \cite{aggarwal1}, the authors study the impact of distance instability on different $\ell_p$ norms, and concludes that smallest values of $p$ lead to more stable norms; in particular, quasinorms with $0<p<1$ are considered. Some counterexamples are given against a generalization of this claim for quasinorms in \cite{distconc}. In \cite{durrant}, the converse of Thm.~\ref{thm:distres} is proved, namely that Eq.~\eqref{eq:distrescond} follows from Eq.~\eqref{eq:distres}: from this fact, the authors find practically relevant cases where Eq.~\eqref{eq:distrescond} is not verified, and propose them as ``good'' examples of where k-means can help. In \cite{mansouri}, the authors propose multiplicative functions $\mathsf{dist}$ and show that they are robust w.r.t.~distance instability. In \cite{radovanovic}, distance instability is related to ``hubness'', i.e.~the number of times a point appears among the $k$ nearest neighbors of other points. In \cite{flexer}, an empirical study is provided which shows how to show an appropriate $\ell_p$ norm that should avoid distance instability w.r.t.~hubness. 

\subsection{The proof}
The proof of the instability theorem can be found in \cite{beyer}. We repeat it here to demonstrate the fact that it is ``syntactical'': every step follows from the previous ones by simple logical inference. There is no appeal to any results other than convergence in probability, Slutsky's theorem, and a simple corollary as shown below. The proof does not pass from object language to meta-language, nor does it require exotic interpretations of symbols in complicated contexts. Although one may find this result surprising, there appears to be no reason to doubt it, and no complication in the proof warranting sophisticated interpretations. The only point worth re-stating is that this is a result about probability distributions, not about actual instances of real data. 

\begin{lemma}
  Let $\{B^m\;|\;m\in\mathbb{N}\}$ be a sequence of of random variables with finite variance. Assume that $\lim_{m\to\infty}\mathsf{E}(B^m)=b$ and that $\lim_{m\to\infty}\mathsf{Var}(B^m)=0$. Then
  \begin{equation}
    \forall\varepsilon>0\;\lim\limits_{m\to\infty}\mathsf{P}(\|B^m-b\|\le\varepsilon)=1.\label{eq:convprob}
  \end{equation}
  \label{lem:convprob}
\end{lemma}
A random variable sequence satisfying Eq.~\eqref{eq:convprob} is said to {\it converge in probability} to $b$. This is denoted $B^m\to_{\mathsf{P}} b$.

\begin{lemma}[Slutsky's theorem \cite{wikipedia_slutsky}]
  Let $\{B^m\;|\;m\in\mathbb{N}\}$ be a sequence of random variables, and $g:\mathbb{R}\to\mathbb{R}$ be a continuous function. If $B^m\to_{\mathsf{P}} b$ and $g(b)$ exists, then $g(B^m)\to_{\mathsf{P}} g(b)$.
  \label{lem:slutsky}
\end{lemma}

\begin{corollary}
  If $\{A^m\;|\;m\in\mathbb{N}\}$ and $\{B^m\;|\;m\in\mathbb{N}\}$ are sequences of random variables such that $A^m\to_{\mathsf{P}}a$ and $B^m\to_{\mathsf{P}}b\not=0$, then $\frac{A^m}{B^m}\to_{\mathsf{P}} \frac{a}{b}$.
  \label{cor:slutsky}
\end{corollary}

\noindent {\it Proof of Thm.~\ref{thm:distres}.} Let $\mu_m=\mathsf{E}((d^m(Z^m,X^m_i))^p)$. We note that $\mu_m$ is independent of $i$ since all $X^m_i$ are identically distributed.

We claim $V_m=\frac{(d^m(Z^m,X^m_i))^p}{\mu_m}\to_{\mathsf{P}} 1$:
\begin{itemize}
\item we have $\mathsf{E}(V_m)=1$ since it is a random variable over its mean: hence, trivially, $\lim_m\mathsf{E}(V_m)=1$;
\item by the hypothesis of the theorem (Eq.~\eqref{eq:distrescond}), $\lim_m\mbox{\sf Var}(V_m)=0$;
\item by Lemma \ref{lem:convprob}, $V_m\to_{\mathsf{P}}1$, which establishes the claim.
\end{itemize}
Now, let $\mathbf{V}^m=(V_m\;|\;i\le\ell)$. By the claim above, we have $\mathbf{V}^m\to_{\mathsf{P}}\mathbf{1}$. Now by Lemma \ref{lem:slutsky} we obtain $\min(\mathbf{V}^m)\to_{\mathsf{P}}\min(\mathbf{1})=1$ and, similarly, $\max(\mathbf{V}^m)\to_{\mathsf{P}}1$. By Cor.~\ref{cor:slutsky}, $\frac{\max(\mathbf{V}^m)}{\min(\mathbf{V}^m)}\to_{\mathsf{P}} 1$. Therefore,
\[\frac{D_{\mathsf{max}}^m}{D_{\mathsf{min}}^m}=\frac{\mu_m\max(\mathbf{V}^m)}{\mu_m\min(\mathbf{V}^m)}\to_{\mathsf{P}} 1.\]
By definition of convergence in probability, we have
\[\forall\varepsilon>0\quad\lim_{m\to\infty} \mathsf{P}(|D_{\mathsf{max}}^m/D_{\mathsf{min}}^m-1|\le\varepsilon)=1.\]
Moreover, since $\mathsf{P}(D_{\mathsf{max}}^m\ge D_{\mathsf{min}}^m)=1$, we have
\[\mathsf{P}(D_{\mathsf{max}}^m\le(1+\varepsilon)D_{\mathsf{min}}^m)= \mathsf{P}(D_{\mathsf{max}}^m/D_{\mathsf{min}}^m-1\le\varepsilon)=\mathsf{P}(|D_{\mathsf{max}}^m/D_{\mathsf{min}}^m-1|\le\varepsilon)=1.\]
The result follows by taking the limit as $m\to\infty$. 

\subsection{In practice}

In Fig.~\ref{fig:distres}, we show how $\varepsilon$ (Eq.~\eqref{eq:distres}) varies with increasing dimension $n$ (recall we assume $m=n$) between $1$ and $10000$.
\begin{figure}[!ht]
  \begin{center}
    \includegraphics[width=8cm]{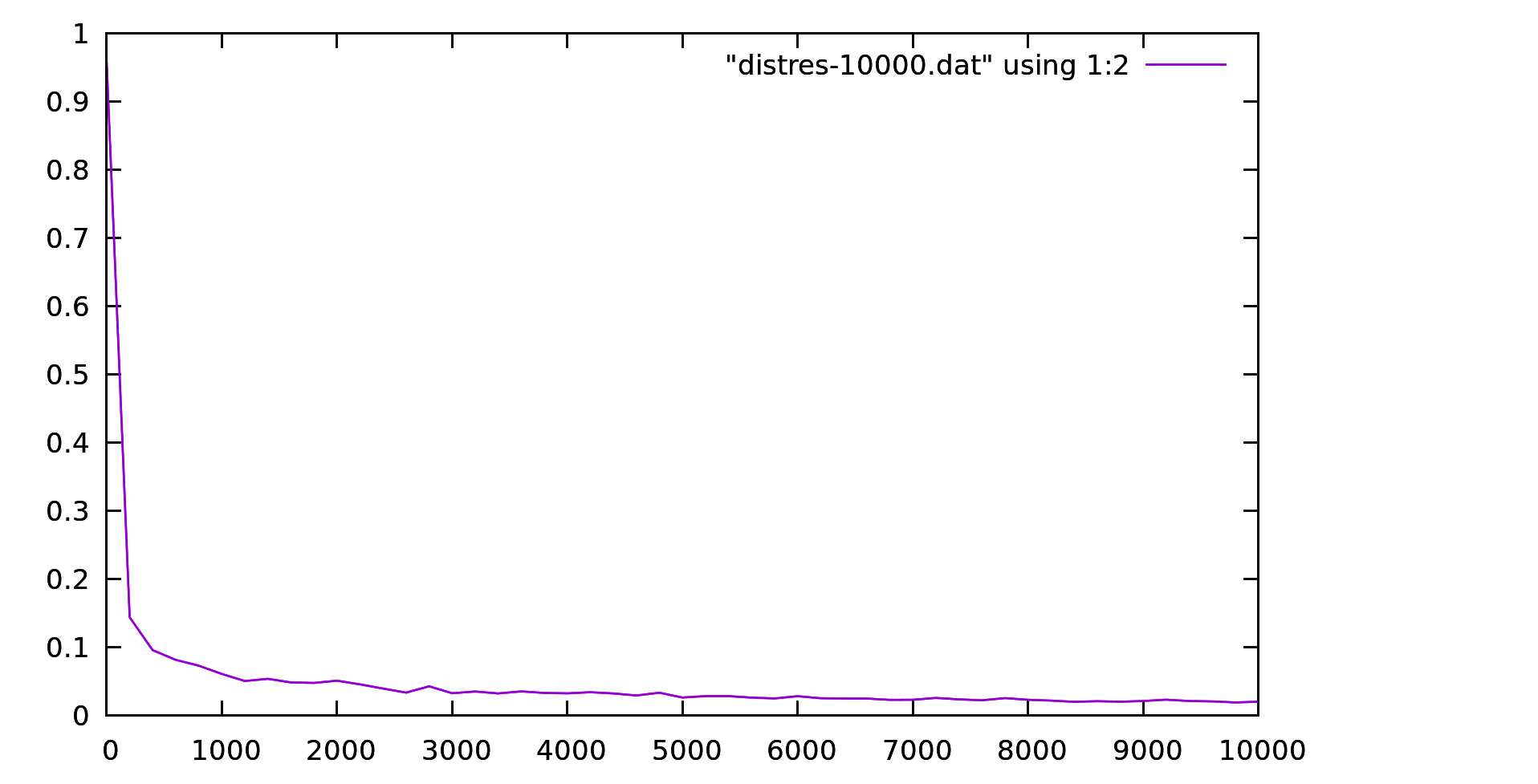}
    \includegraphics[width=8cm]{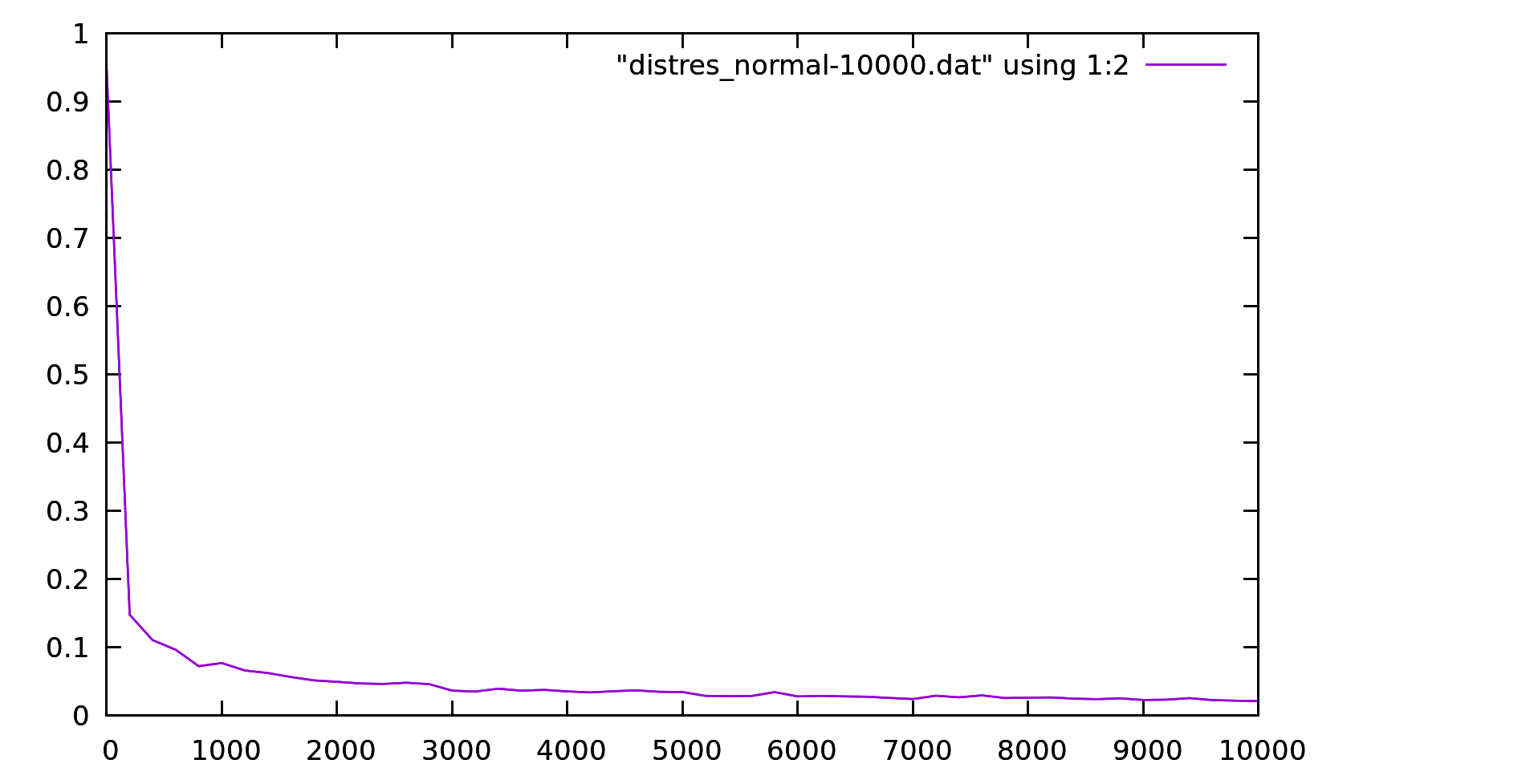}
    \includegraphics[width=8cm]{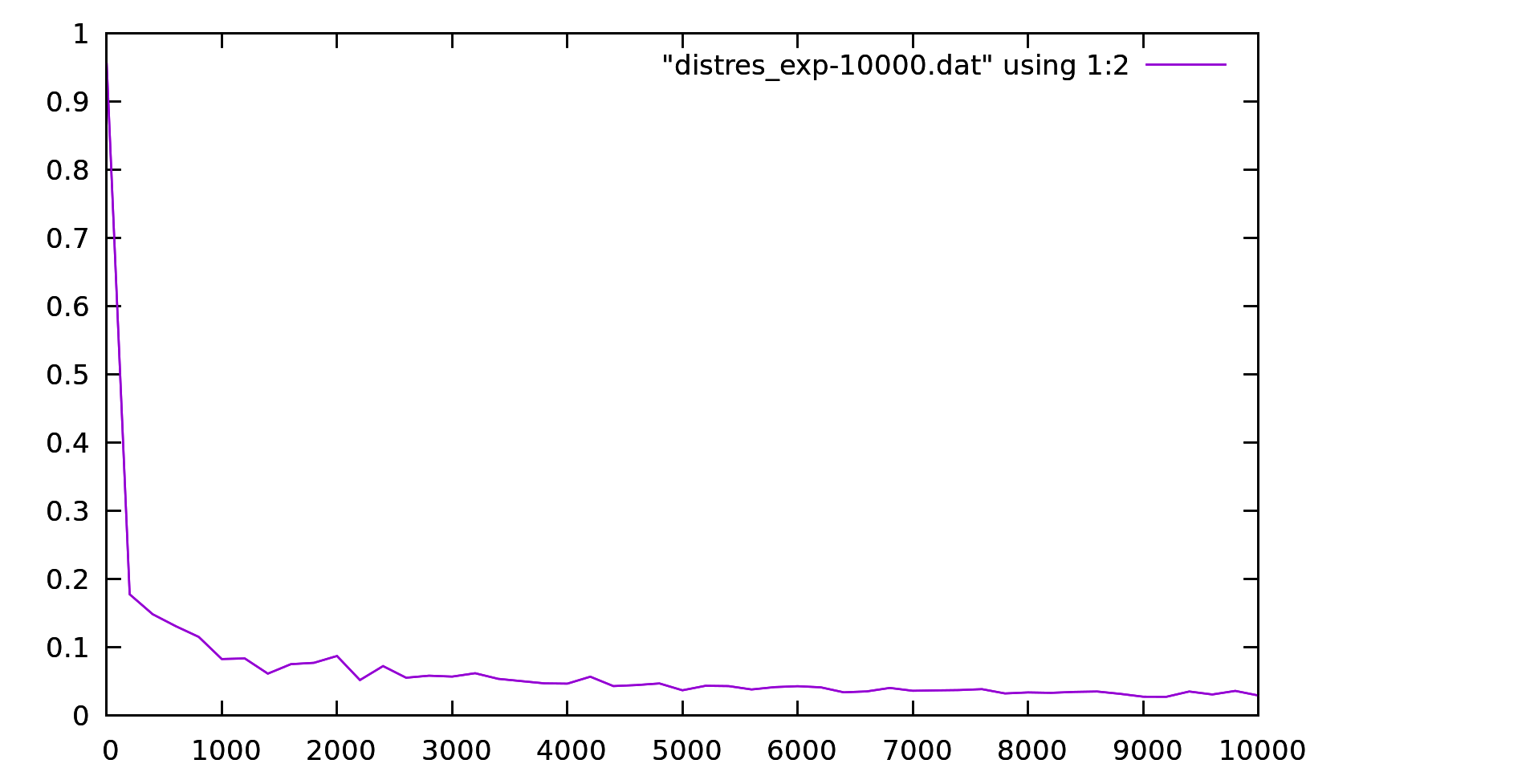}
  \end{center}
  \caption{Plots of $\varepsilon$ versus $n$ for the uniform distribution on $[0,1]$ (left), $\mathsf{N}(0,1)$ (center), the exponential distribution with parameter $1$ (right).}
  \label{fig:distres}
\end{figure}
It is clear that $\varepsilon$ decreases very rapidly towards zero, and then reaches its asymptotic value more slowly. On the other hand, $\varepsilon$ is the distortion between minimum and maximum distance values; most algorithms need to discriminate between smallest and second smallest distance values.

Most of the papers listed in Sect.~\ref{s:distreslitrev} include empirical tests which illustrate the impact and limits of the distance instability phenomenon.

\section{An application to neural networks}
\label{s:anndg}
In this last section we finally show how several concepts explained in this survey can be used conjunctively. We shall consider a natural language processing task (Sect.~\ref{s:datgph}) where we cluster some sentences (Sect.~\ref{s:clustering}) using an ANN (Sect.~\ref{s:ann}) with different training sets $T=(X,Y)$. We compare ANN performances depending on the training set used.

The input set $X$ is a vector representation of the input sentences. The output set $Y$ is a vectorial representation of cluster labels: we experiment with (a) clusterings obtained by running k-means (Sect.~\ref{s:kmeans}) on the input sets, and (b) a clustering found by a modularity maximization heuristic (Sect.~\ref{s:modularity}). All of these clusterings are considered ``ground truth'' sets $Y$ we would like our ANN to learn to associate to various types of input vector sets $X$ representing the sentences. The sentences to be clustered are first transformed into graphs (Sect.~\ref{s:gphtext}), and then into vectors (Sect.~\ref{s:dgpsol}), which then undergo dimensionality reduction (Sect.~\ref{s:dimred}).

Our goal is to compare the results obtained by the same ANN with different vector representations for the same text: most notably, the comparison will confront how well or poorly input vector sets can predict the ground truth outputs. We will focus specifically on a comparison of the well-known incidence vectors (Sect.~\ref{s:incvec}) embeddings w.r.t.~the newly proposed DGP methods we surveyed in Sect.~\ref{s:dgpsol}.

In our implementations, all our code was developed using Python 3 \cite{python3}.

\subsection{Performance measure}
\label{s:performance}


We are going to measure the performance quality of the error of an ANN, which is based on a comparison of its output with the ground truth that the ANN is supposed to learn. Using the notation of Sect.~\ref{s:ann}, if the ANN output for a given input $x\in\mathbb{R}^n$ consists of a vector $y\in\mathbb{R}^k$, and if the ground truth corresponding to $x$ is $z\in\mathbb{R}^k$, then we define the error as the {\it loss} function:
\begin{eqnarray}
  \mathsf{loss}(y,z) &=& \|y-z\|_2. \label{loss1} 
\end{eqnarray}
An ANN $\mathcal{N}=(G,T,\phi)$ is usually evaluated over many (input,output) pairs. Let $\hat{X}\subset\mathbb{R}^n$ and $\hat{Y}\subset\mathbb{R}^k$ be, respectively, a set of input vectors and the corresponding set of output vectors evaluated by the trained ANN. Let $\hat{Z}$ be a set of ground truth vectors corresponding to $\hat{X}$, and assume $|\hat{X}|=|\hat{Y}|=|\hat{Z}|=q$. The cumulative loss measure evaluated on the {\it test set} $(\hat{X},\hat{Z})$ is then
\begin{eqnarray}
  \mathsf{loss}(\mathcal{N}) &=& \frac{1}{q} \sum\limits_{y\in\hat{Y}\atop z\in\hat{Z}}\mathsf{loss}(y,z). \label{loss} 
\end{eqnarray}

\subsection{A Natural Language Processing task}
\label{s:natlangproc}
Clustering of sentences in a text is a common task in Natural Language Processing. We considered ``On the duty of civil disobedience'' by H.D.~Tho\-reau \cite{civildis,wikipedia_civdis}. This text is stored in an ASCII file {\tt walden.txt} which can be obtained from {\tt archive.org}. The file is 661146 bytes long, organized in 10108 lines and 116608 words. The text was parsed into sentences using basic methods from NLTK \cite{nltk} under Python 3. Common words, stopwords, punctuation and unusual characters were removed. After ``cleaning'', the text was reduced to 4083 sentences over a set of 11431 ``significant'' words (see Sect.~\ref{s:sentences}).

As mentioned above, we want to train our ANN to learn different types of clusterings:
\begin{itemize}
\item (\textbf{\textit{k-means}}) obtained by running the the k-means unsupervised clustering algorithm (Sect.~\ref{s:kmeans}) over the different vector representations of the sentences in the text;
\item (\textbf{\textit{sentence graph}}) obtained by running a modularity clustering heuristic (Sect.~\ref{s:modularity}) on a graph representation of the sentences in the document (see Sect.~\ref{s:sentencegraph}).
\end{itemize}
These clusterings are used as ground truths, and provide the output part of the training sets to be used by the ANN, as well as of the test sets for measuring purposes (Sect.~\ref{s:performance}). See Sect.~\ref{s:outputset} for more information on the construction of these clusterings.

\subsubsection{Selecting the sentences}
\label{s:sentences}
We constructed two sets of sentences. 
\begin{itemize}
\item {\bf The large sentence set.} Each sentence in {\tt walden.txt} was mapped to an incidence vector of $3$-grams in $\{0,1\}^{48087}$, i.e.~a dictionary of 48087 $3$-grams over the text. In other words, 48087 $3$-grams were found in the text, then each sentence was mapped to a vector having $1$ at component $i$ iff the $i$-th $3$-gram was present in the sentence. Since some sentences had fewer than 3 significant words, only 3940 sentences remained in the sentence set $S$, which was therefore represented as a $3940\times 48087$ matrix $\bar{S}$ with components in $\{0,1\}$.
\item {\bf The small sentence set.} It turns out that most of the $3$-grams in the set $S$ only appear a single time. We selected a subset $S'\subset S$ of sentences having $3$-grams appearing in at least two sentences. It turns out that $|S'|=245$, and the total number of $3$-grams appearing more than once is 160. $S'$ is therefore naturally represented as a $245\times 160$ matrix $\bar{S}'$ with components in $\{0,1\}$.
\end{itemize}
We constructed training sets (Sect.~\ref{s:trainingset}) for each of these two sets.

\subsubsection{Construction of a sentence graph}
\label{s:sentencegraph}
In this section we describe the method used to construct a sentence graph $G^{\mathsf{s}}=(S,E)$ from the text, which is used to produce a ground truth for the (\textbf{\textit{sentence graph}}) type. $G^{\mathsf{s}}$ is then clustered using the greedy modularity clustering heuristic in the Python library {\tt networkX} \cite{networkX}. 

Each sentence in the text is encoded into a weighted graph-of-word (see Sect.~\ref{s:graphofwords}) over $3$-grams, with edges $\{u,v\}$ weighted by the number $c_{uv}$ of $3$-grams where the two words $u,v$ appear. The union of the graph-of-words for the sentences (contracting repeated words to a single vertex) yields a weighted graph-of-word $G^{\mathsf{w}}$ for the whole text. 


The graph $G^{\mathsf{w}}=(W,F)$ is then ``projected'' onto the set $S$ of sentences as follows. We define the logical proposition $P(u,v,s,t)$ to mean $(u\in s\land v\in t)\vee (v\in s\land u\in t)$ for words $u,v$ and sentences $s,t$. The edge set $E$ of $G^{\mathsf{s}}$ is then defined by the following implication:
\[\forall \{u,v\}\in F,\; s,t\in S \qquad P(u,v,s,t)\to \{s,t\}\in E.\]
In other words, $s,t$ form an edge in $E$ if two words $u,v$ in $s,t$ (respectively) or $t,s$ form an edge in $F$. For each edge $\{s,t\}\in E$, the weight $w_{st}$ is given by:
\[ w_{st} = \sum\limits_{\{u,v\}\in F\atop P(u,v,s,t)} c_{uv}, \]
with edge weights meaning similarity. 

\subsection{The ANN}
\label{s:annimpl}
We consider a very simple ANN $\mathcal{N}=(G,T,\phi)$. In the terminology of Sect.~\ref{s:ann}, the underlying digraph $G=(V,A)$ is tripartite with $V=V_1\dot{\cup}V_2\dot{\cup}V_3$. The ``input layer'' $V_1$ has $n$ nodes, where $n$ is the dimensionality of the input vector set $X$. The ``output layer'' $V_3$ has a single node. The ``hidden layer'' $V_2$ has a constant number of nodes ($20$ in our experiments). The training set $T$ is discussed in Sect.~\ref{s:trainingset}. We adopt the piecewise-linear mapping known as {\it rectified linear unit} (ReLu) \cite{wikipedia_relu} for the activation functions $\phi$ in $V_2$, and a traditional sigmoid function for the single node in $V_3$. Both types of activation functions map into $[0,1]$.

We implemented $\mathcal{N}$ using the Python library {\tt keras} \cite{keras}, which is a high-level API running over TensorFlow \cite{tensorflow}. The default configuration was chosen for all layers. We used the {\sc Adam} solver \cite{adam} in order to train the network. Each training set was split in three parts: 35\% of the vectors were used for training, 35\% for {\it validation} (a training phase used for deciding values of any model parameter aside from $v,b,w$, if any exist, and/or for deciding when to stop the training phase), and 30\% for testing. The performance of the ANN is measured using the loss function in Eq.~\eqref{loss}. 

\subsection{Training sets}
\label{s:trainingset}
Our goal is to compare training sets $T=(X,Y)$ where the vectors in $X$ were constructed in a variety of ways, and the vectors in $Y$ were obtained by running k-means (Sect.~\ref{s:kmeans}) on the vectors in $X$. 

In particular, we consider input sets $X(\sigma,\mu,\rho)$ where:
\begin{itemize}
\item $\sigma\in\Sigma=\{S',S\}$ is the sentence set: $\sigma=S'$ corresponds to the small set with 245 sentences, $\sigma=S$ corresponds to the large set with 3940 sentences;
\item $\mu\in M=\{\mathsf{inc},\mathsf{uie},\mathsf{qrt},\mathsf{sdp}\}$ is the method used to map sentences to vectors: $\mathsf{inc}$ are the incidence vectors (Sect.~\ref{s:incvec}), $\mathsf{uie}$ is the universal isometric embedding (Sect.~\ref{s:uie}), $\mathsf{qrt}$ is the unconstrained quartic (Sect.~\ref{s:unconstrained}), $\mathsf{sdp}$ is the SDP (Sect.~\ref{s:sdp});
\item $\rho\in R=\{\mathsf{pca},\mathsf{rp}\}$ is the dimensional reduction method used: $\mathsf{pca}$ is PCA (Sect.~\ref{s:pca}), $\mathsf{rp}$ are RPs (Sect.~\ref{s:rp}).
\end{itemize}
The methods in $M$ were all implemented using Python 3 with some well known external libraries (e.g.~{\tt numpy}, {\tt scipy}). Specifically, $\mathsf{qrt}$ was implemented using the {\sc Ipopt} \cite{ipopt} NLP solver, and $\mathsf{sdp}$ was implemented using the SCS \cite{scs} SDP solver. As for the dimensional reduction methods in $R$, the PCA implementation of choice was the probabilistic PCA algorithm implemented in the Python library {\tt scikit-learn} \cite{scikitlearn}. The RPs we chose were the simplest: each component of the RP matrices was sampled from an appropriately scaled zero-mean Gaussian distribution (Thm.~\ref{thm:rpnorm}).

\subsubsection{The output set}
\label{s:outputset}
The output set $Y$ should naturally contain discrete values, namely the labels of the $h$ clusters $\{1,2,\ldots,h\}$ in the ground truth clusterings. We map these values to scalars in $[0,1]$ (or, according to Sect.~\ref{s:ann}, to $k$-dimensional vectors with $k=1$) as follows. We divide the range $[0,1]$ into $h-1$ equal sub-intervals of length $1/(h-1)$, and hence $h$ discrete values in $[0,1]$. Then we assign labels to sub-intervals endpoints: label $j$ is mapped to $(j-1)/(h-1)$ (for $1\le j\le h$).

As mentioned above, we consider two types of output sets:
\begin{itemize}
  \item (\textbf{\textit{k-means}}) for each input set $X(\sigma,\mu,\rho)$ we obtained an output set $Y(\sigma,\mu,\rho)$ using k-means (Sect.~\ref{s:kmeans}) implementation in {\tt scikit-learn} \cite{scikitlearn} on the vectors in $X$, for each sentence set $\sigma\in\Sigma$, method $\mu\in M$, and dimensional reduction method $\rho\in R$;
  \item (\textbf{\textit{sentence graph}}) for each sentence set $\sigma\in\Sigma$ we constructed a sentence graph as detailed in Sect.~\ref{s:sentencegraph}.
\end{itemize}

\subsubsection{Realizations to vectors}
\label{s:rlz2vec}
The $\mathsf{inc}$ method (Sect.~\ref{s:incvec}) is the only one (in our benchmark) that can natively map sentences of various lengths into vectors all having the same number of components.

For all other methods in $M\smallsetminus\{\mathsf{inc}\}$, we loop over sentences (in small/large sets $S',S$). For each sentence we construct its graph-of-words (Sect.~\ref{s:graphofwords}). We then realize it in some arbitrary dimensional Euclidean space $\mathbb{R}^K$ (specifically, we chose $K=10$) using \textsf{uie}, \textsf{qrt}, \textsf{sdp}. At this point, we are confronted with the following difficulty: a realization of a graph $G$ with $p$ vertices in $\mathbb{R}^K$ is a $p\times K$ matrix, and we have as many graphs $G$ as we have sentences, with $p$ varying over the number of unique words in the sentences (i.e.~the cardinalities of the vertex sets of the graphs-of-words).

In order to reduce all of these differently-sized realizations to vectors having the same dimension, we follow the following procedure. Given realizations $\{x^i\in\mathbb{R}^{p_i\times K}\;|\;i\in \sigma\}$, where $\sigma$ is the set of sentences (for $\sigma\in\Sigma$) and $x^i$ realizes the graph-of-word of sentence $i\in \sigma$,
\begin{enumerate}
\item we stack the columns of $x^i$ so as to obtain a single vector $\hat{x}^i\in\mathbb{R}^{p_iK}$ for each $i\in \sigma$;
\item we let $\hat{n}=\max_i p_iK$ be the maximum dimensionality of the stacked realizations;
\item we pad every realization vector $\hat{x}^i$ shorter than $\hat{n}$ with zeros to achieve dimension $\hat{n}$ for stacked realization vectors;
\item we form the $s\times\hat{n}$ matrix $\hat{X}$ having $\hat{x}^i$ as its $i$-th row (for $i\in\sigma$ and with $s=|\sigma|$);
\item we reduce the dimensionality of $\hat{X}$ to an $s\times n$ matrix $X$ with \textsf{pca} or \textsf{rp}.
\end{enumerate}

\subsection{Computational comparison}
We discuss the details of our training sets, a validation test, and the comparison tests.

\subsubsection{Training set statistics}

In Table \ref{t:stats} we report the dimensionalities of the vectors in the input parts $X(\sigma,\mu,\rho)$ of the training sets, as well as the number of clusters in the output sets $Y(\sigma,\mu,\rho)$ of the (\textbf{\textit{k-means}}) class.
\begin{table}[!ht]
  \begin{center}
    \begin{tabular}{|l|rrrr|rrrr|} \hline
      \multicolumn{9}{|c|}{\it Dimensionality of input vectors} \\ \hline
      \hfill$\mu$  & \multicolumn{4}{c|}{$|\sigma|=245$} & \multicolumn{4}{c|}{$|\sigma|=3940$} \\
      $\rho$ & \textsf{inc} & \textsf{uie} & \textsf{qrt} & \textsf{sdp} & \textsf{inc} & \textsf{uie} & \textsf{qrt} & \textsf{sdp} \\ \hline
      \textsf{pca} & 3 & 159 & 244 & 200 & 3 & 10 & 400 & 400  \\
      \textsf{rp} & 100 & 248 & 248 & 248 & 373 & 373 & 373 & 373 \\ 
      {\it original} & {\it 160} & {\it 1140} & {\it 1140} & {\it 1140} & {\it 48087} & {\it 1460} & {\it 1460} & {\it 1460} \\ \hline \hline
      \multicolumn{9}{|c|}{\it Number of clusters to learn} \\ \hline
      \textsf{pca} & 4 & 3 & 11 & 6 & 3 & 8 & 9 & 14 \\
      \textsf{rp} & 4 & 3 & 7 & 5 & 3 & 9 & 16 & 14 \\ \hline
    \end{tabular}
  \end{center}
  \caption{Training set statistics for $X(\sigma,\mu,\rho)$ and corresponding output sets in the (\textbf{\textit{k-means}}) class.}
  \label{t:stats}
\end{table}
We recall that the number of clusters was found with k-means in the {\tt scikit-learn} implementation. The choice of `k' corresponds to the smallest number of clusters giving a nontrivial clustering (with ``trivial'' meaning having a cluster of zero cardinality, or too close to zero relative to the set size, only possibly allowing some outlier clusters with a single element). Some more remarks follow.
\begin{itemize}
  \item For $\rho=\mathsf{pca}$ we employed the smallest dimension such that the residual variance in the neglected components was almost zero; this ranges from $3$ to $244$ in Table \ref{t:stats}. For the two cases where the dimensionality reduction was set to $400$ ($\mathsf{qrt}$ and $\mathsf{sdp}$ in the large sentence set $S$), the residual variance was nonzero.
  \item It is interesting that for $\mu=\mathsf{uie}$ we have higher projected dimensionality ($248$) in the small set $S'$ than in the large set $S$ (10): this depends on the fact that the large set has more easily distinguishable clusters (8 found by k-means) than the small set (only 3 found by k-means). The dimension of $X(\mathsf{inc},\mathsf{pca},S)$ is smaller (3) than that of $X(\mathsf{uie},\mathsf{pca},S)$ (10) even though the original number of dimensions of the former (48087) vastly exceeds that of the latter (1460) for the same reason.
  \item The training sets $X(\sigma,\mathsf{inc},\mathsf{pca})$ are the smallest-dimensional ones (for $\sigma\in\{S',S\}$): they are also ``degenerate'', in the sense that the vectors in a given clusters are all equal; the co-occurrence patterns of the incidence vectors conveyed relatively little information to this vectorial sentence representation.
  \item The RP-based dimensionality reduction method yields the same dimensionality ($373$) of $X(\mu,\mathsf{rp},S)$ for $\mu\in M$. This occurs because the target dimensionality in RP depends on the number of vectors, which is the same for all methods (3940), rather than on the number of dimensions (see Sect.~\ref{s:rp}).
\end{itemize}

There is one output set in the (\textbf{\textit{sentence graph}}) class for each $\sigma\in\Sigma$. For $\sigma=S'$ we have $|V|=245$, $|E|=28519$, and $230$ clusters, with the first $5$ clusters having $6,5,4,3,2$ elements, and the rest having a single element. For $\sigma=S$ we have $|V|=3940$, $|E|=7173633$, and 3402 clusters, with the first $10$ clusters having $161,115,62,38,34,29,19,16,14,11$ elements, and the rest having fewer than $10$ elements.

\subsubsection{Comparison tests}
We first report the comparative results of the ANN on
\[T=(X(\sigma,\mu_1,\rho_1),Y(\sigma,\mu_2,\rho_2))\]
for $\sigma\in\Sigma$, $\mu_1,\mu_2\in M$, $\rho_1,\rho_2\in R$. The sums in the rightmost columns of Table~\ref{t:comparison} are only carried out on terms obtained with an input vector generation method $\mu_1$ different from the method $\mu_2$ used to obtain the ground truth clustering via k-means (since we want to compare methods). The results corresponding to cases where $\mu_1=\mu_2$ are emphasized in italics in the table. The best performance sums are emphasized in boldface, and the worst are shown in grey.
\begin{table}[!ht]
  \begin{center}
    {\footnotesize
    \begin{tabular}{|c|r|rrrrrrrr|r|} \hline
       & \multicolumn{10}{c|}{Training set outputs} \\ \cline{2-11} 
      \multirow{20}{*}{\rotatebox[origin=c]{90}{\hspace*{-2cm}Training set inputs}} & $\mu$ & \textsf{inc} & \textsf{inc} & \textsf{uie} & \textsf{uie} & \textsf{qrt} & \textsf{qrt} &  \textsf{sdp} & \textsf{sdp} & sum \\ 
      & $\rho$ & \textsf{pca} & \textsf{rp} & \textsf{pca} & \textsf{rp} & \textsf{pca} & \textsf{rp} & \textsf{pca} & \textsf{rp} & $\mu'\not=\mu$ \\ \cline{2-11} \cline{2-11} 
      & $|\sigma|$ & \multicolumn{9}{c|}{\it $245$} \\ \cline{2-11} 
      & \begin{minipage}{0.5cm}\textsf{inc}\\[-0.3em] \textsf{pca}\end{minipage} & {\it 0.061} & {\it 0.042} & 0.059 & 0.013 & 0.094 & 0.108 & 0.064 & 0.025 & {\bf 0.363} \\ \cline{2-11} 
      & \begin{minipage}{0.5cm}\textsf{inc}\\[-0.3em] \textsf{rp}\end{minipage} & {\it 0.005} & {\it 0.010} & 0.055 & 0.015 & 0.104 & 0.109 & 0.065 & 0.025 & 0.373 \\ \cline{2-11} 
      & \begin{minipage}{0.5cm}\textsf{uie}\\[-0.3em] \textsf{pca}\end{minipage} & 0.271 & 0.052 & {\it 0.070} & {\it 0.169} & 0.233 & 0.201 & 0.127 & 0.111 & {\color{darkgrey}0.995} \\ \cline{2-11} 
      & \begin{minipage}{0.5cm}\textsf{uie}\\[-0.3em] \textsf{rp}\end{minipage} & 0.093 & 0.026 & {\it 0.094} & {\it 0.076} & 0.191 & 0.236 & 0.079 & 0.117 & {\color{darkgrey}0.976} \\ \cline{2-11} 
      & \begin{minipage}{0.5cm}\textsf{qrt}\\[-0.3em] \textsf{pca}\end{minipage} & 0.082 & 0.067 & 0.105 & 0.047 & {\it 0.084} & {\it 0.133} & 0.071 & 0.087 & 0.459 \\ \cline{2-11} 
      & \begin{minipage}{0.5cm}\textsf{qrt}\\[-0.3em] \textsf{rp}\end{minipage} & 0.057 & 0.068 & 0.059 & 0.053 & {\it 0.162} & {\it 0.073} & 0.095 & 0.055 & 0.387 \\ \cline{2-11} 
      & \begin{minipage}{0.5cm}\textsf{sdp}\\[-0.3em] \textsf{pca}\end{minipage} & 0.106 & 0.063 & 0.067 & 0.022 & 0.106 & 0.135 & {\it 0.058} & {\it 0.034} & 0.499 \\ \cline{2-11} 
      & \begin{minipage}{0.5cm}\textsf{sdp}\\[-0.3em] \textsf{rp}\end{minipage} & 0.095 & 0.065 & 0.093 & 0.021 & 0.103 & 0.139 & {\it 0.074} & {\it 0.018} & 0.516 \\ \cline{2-11} \cline{2-11} 
      & $|\sigma|$ & \multicolumn{9}{c|}{\it $3940$} \\ \cline{2-11} 
      & \begin{minipage}{0.5cm}\textsf{inc}\\[-0.3em] \textsf{pca}\end{minipage} & {\it 0.052} & {\it 0.013} & 0.068 & 0.027 & 0.106 & 0.164 & 0.079 & 0.161 & {\color{darkgrey}0.605} \\ \cline{2-11} 
      & \begin{minipage}{0.5cm}\textsf{inc}\\[-0.3em] \textsf{rp}\end{minipage} & {\it 0.001} & {\it 0.000} & 0.067 & 0.028 & 0.106 & 0.167 & 0.080 & 0.159 & {\color{darkgrey}0.607} \\ \cline{2-11} 
      & \begin{minipage}{0.5cm}\textsf{uie}\\[-0.3em] \textsf{pca}\end{minipage} & 0.063 & 0.022 & {\it 0.020} & {\it 0.016} & 0.124 & 0.201 & 0.070 & 0.127 & {\color{darkgrey}0.607} \\ \cline{2-11} 
      & \begin{minipage}{0.5cm}\textsf{uie}\\[-0.3em] \textsf{rp}\end{minipage} & 0.061 & 0.023 & {\it 0.024} & {\it 0.023} & 0.131 & 0.190 & 0.072 & 0.126 & {\color{darkgrey}0.603} \\ \cline{2-11} 
      & \begin{minipage}{0.5cm}\textsf{qrt}\\[-0.3em] \textsf{pca}\end{minipage} & 0.063 & 0.022 & 0.36 & 0.023 & {\it 0.038} & {\it 0.218} & 0.079 & 0.159 & {\bf 0.382} \\ \cline{2-11} 
      & \begin{minipage}{0.5cm}\textsf{qrt}\\[-0.3em] \textsf{rp}\end{minipage} & 0.062 & 0.024 & 0.047 & 0.025 & {\it 0.120} & {\it 0.035} & 0.076 & 0.164 & 0.398 \\ \cline{2-11} 
      & \begin{minipage}{0.5cm}\textsf{sdp}\\[-0.3em] \textsf{pca}\end{minipage} & 0.063 & 0.021 & 0.023 & 0.024 & 0.126 & 0.195 & {\it 0.033} & {\it 0.149} & 0.452 \\ \cline{2-11} 
      & \begin{minipage}{0.5cm}\textsf{sdp}\\[-0.3em] \textsf{rp}\end{minipage} & 0.059 & 0.021 & 0.025 & 0.024 & 0.121 & 0.176 & {\it 0.083} & {\it 0.037} & 0.426 \\ \hline
    \end{tabular}
    }
  \end{center}
  \caption{Comparison tests on output sets of (\textbf{\textit{k-means}}) class.}
  \label{t:comparison}
\end{table}

According to Table \ref{t:comparison}, for the small sentence set the best method is $\mathsf{inc}$, but $\mathsf{qrt}$ and $\mathsf{sdp}$ are not far behind; the only really imprecise method is $\mathsf{uie}$. For the large sentence set the best method is $\mathsf{qrt}$, with $\mathsf{sdp}$ not far behind; both $\mathsf{inc}$, $\mathsf{uie}$ are imprecise.

\begin{table}[!ht]
  \begin{center}
    {\footnotesize
    \begin{tabular}{|c|r|rrrrrrrr|} \hline
       & \multicolumn{9}{c|}{Training set outputs} \\ \cline{2-10} 
      \multirow{6}{*}{\rotatebox[origin=c]{90}{Training inputs}} & $\mu$ & \textsf{inc} & \textsf{inc} & \textsf{uie} & \textsf{uie} & \textsf{qrt} & \textsf{qrt} &  \textsf{sdp} & \textsf{sdp} \\ 
      & $\rho$ & \textsf{pca} & \textsf{rp} & \textsf{pca} & \textsf{rp} & \textsf{pca} & \textsf{rp} & \textsf{pca} & \textsf{rp} \\ \cline{2-10} \cline{2-10} 
      & $|\sigma|$ & \multicolumn{8}{c|}{\it $245$} \\ \cline{2-10} 
      & & {\bf 0.107} & 0.108 & {\color{darkgrey}0.196} & {\color{darkgrey}0.184} & 0.129 & {\color{darkgrey}0.151} & 0.109 & 0.122  \\ \cline{2-10} \cline{2-10} 
      & $|\sigma|$ & \multicolumn{8}{c|}{\it $3940$} \\ \cline{2-10} 
      & & {\bf 0.097} & 0.098 & {\color{darkgrey}0.124} & {\color{darkgrey}0.119} & {\color{darkgrey}0.136} &  {\color{darkgrey}0.113} &  {\color{darkgrey}0.114} & 0.106 \\ \hline
    \end{tabular}
    }
  \end{center}
  \caption{Comparison tests on output sets of (\textbf{\textit{sentence graph}}) class.}
  \label{t:comparison2}
\end{table}
In Table \ref{t:comparison2}, which has a similar format as Table \ref{t:comparison}, we report results on training sets
\[\bar{T}=(X(\sigma,\mu,\rho),\bar{Y}(\sigma))\]
for $\sigma\in\Sigma$, $\mu\in M$, $\rho\in R$, where $\bar{Y}(\sigma)$ are output sets of the (\textbf{\textit{sentence graph}}) class. For the small set, \textsf{inc} is the best method (independently of $\rho$), with $(\mu=\mathsf{sdp},\rho=\mathsf{pca})$ following very closely, and, in general, \textsf{sdp} and \textsf{qrt} still being acceptable; \textsf{uie} is the most imprecise method. For the large set \textsf{inc} is againt the best method, with $(\mu=\mathsf{sdp},\rho=\mathsf{rp})$ following closely. While the other method do not excel, the performance difference between all methods is less remarkable than with the small set.

\section{Conclusion}
\label{s:concl}
We have surveyed some of the concepts and methodologies of distance geometry which are used in data science. More specifically, we have looked at algorithms (mostly based on mathematical programming) for representing graphs as vectors as a pre-processing step to performing some machine learning task requiring vectorial input.

We started with brief introductions to mathematical programming and distance geometry. We then showed some ways to represent data by graphs, and introduced clustering on vectors and graphs. Following, we surveyed robust algorithms for realizing weighted graphs in Euclidean spaces, where the robustness is with respect to errors or noise in the input data. It turns out that most of these algorithms are based on mathematical programming. Since some of these algorithms output high-dimensional vectors and/or high-rank matrices, we also surveyed some dimensional reduction techniques. We also discussed a result about the instability of distances with respect to randomly generated points.

The guiding idea in this survey is that, when one is confronted with clustering on graphs, then distance geometry allows the use of many supervised and unsupervised clustering techniques based on vectors. To demonstrate the applicability of this idea, we showed that vectorial representations of graphs obtained using distance geometry offer competitive performances when training an artificial neural network. While we do not think that our limited empirical analysis allows any definite conclusion, we hope that it will entice more research in this area.

\section*{Acknowledgements}
I am grateful to J.J.~Salazar, the Editor-in-Chief of TOP, for inviting me to write this survey. This work would not have been possible without the numerous co-authors with whom I pursued my investigations in distance geometry, among which I will single out the longest-standing: C.~Lavor, N.~Maculan, A.~Mucherino. I have first heard of concentration of measure as I passed by D.~Malioutov's office at the T.J.~Watson IBM Research laboratory: the door was open, the Johnson-Lindenstrauss lemma was mentioned, and I could not refrain from interrupting the conversation and asking for clarification, as I thought it must be a mistake; incredibly, it was not, and I am grateful to Dmitry Malioutov for hosting the conversation I eavesdropped on. I very thankful to the co-authors who helped me investigate random projections, in particular P.L.~Poirion and K.~Vu, without whom none of our papers would have been possible. I learned about the existence of the distance instability result thanks to N.~Gayraud, who suggested it to me as I expressed puzzlement at the poor quality of k-means clusterings during my talk. I am very grateful to S.~Khalife and M.~Escobar for reading the manuscript and making insightful comments.

\bibliographystyle{plain}
\bibliography{dgds-arXiv}   

\end{document}